\documentclass{article}

\usepackage{microtype}
\usepackage{graphicx}
\usepackage{subfigure}
\usepackage{booktabs} 

\usepackage{hyperref}

\usepackage[accepted]{icml2023}

\usepackage{amsmath,amsthm,amssymb}
\usepackage{thmtools}
\usepackage{graphicx}
\usepackage{wrapfig}
\usepackage[all]{xy}
\usepackage{cancel}
\usepackage[figuresright]{rotating}
\usepackage{bbm}
\usepackage{subfigure}
\usepackage{makecell}

\usepackage{amsmath,amsfonts,bm}

\def\eqref#1{equation~\ref{#1}}

\def\1{\bm{1}}

\DeclareMathAlphabet{\mathsfit}{\encodingdefault}{\sfdefault}{m}{sl}
\SetMathAlphabet{\mathsfit}{bold}{\encodingdefault}{\sfdefault}{bx}{n}

\def\gA{{\mathcal{A}}}
\def\gB{{\mathcal{B}}}
\def\gC{{\mathcal{C}}}
\def\gD{{\mathcal{D}}}

\def\gJ{{\mathcal{J}}}
\def\gK{{\mathcal{K}}}

\def\gR{{\mathcal{R}}}
\def\gS{{\mathcal{S}}}

\def\sA{{\mathbb{A}}}

\def\sN{{\mathbb{N}}}

\def\sP{{\mathbb{P}}}
\def\sQ{{\mathbb{Q}}}
\def\sR{{\mathbb{R}}}

\def\sZ{{\mathbb{Z}}}

\newcommand{\E}{\mathbb{E}}

\newcommand*{\imgintext}[1]{%
	\raisebox{-.3\baselineskip}{%
		\includegraphics[
		height=\baselineskip,
		width=\baselineskip,
		keepaspectratio,
		]{#1}%
	}%
}

\usepackage[capitalize,noabbrev]{cleveref}

\theoremstyle{plain}

\newtheorem{lemma}{Lemma}

\newtheorem{assumption}{Assumption}

\usepackage[textsize=tiny]{todonotes}

\icmltitlerunning{Offline Meta Reinforcement Learning with In-Distribution Online Adaptation}

\begin{document}

\twocolumn[
\icmltitle{Offline Meta Reinforcement Learning with In-Distribution Online Adaptation}



\icmlsetsymbol{equal}{*}

\begin{icmlauthorlist}
\icmlauthor{Jianhao Wang}{equal,IIIS}
\icmlauthor{Jin Zhang}{equal,IIIS}
\icmlauthor{Haozhe Jiang}{IIIS}
\icmlauthor{Junyu Zhang}{huazhong}
\icmlauthor{Liwei Wang}{Peking}
\icmlauthor{Chongjie Zhang}{IIIS}
\end{icmlauthorlist}

\icmlaffiliation{IIIS}{Institute for Interdisciplinary Information Sciences, Tsinghua University}
\icmlaffiliation{huazhong}{Huazhong University of Science and Technology}
\icmlaffiliation{Peking}{Institute of Artificial Intelligence, Peking University}

\icmlcorrespondingauthor{Chongjie Zhang}{zhangchongjie@gmail.com}

\icmlkeywords{Machine Learning, ICML}

\vskip 0.3in
]



\printAffiliationsAndNotice{\icmlEqualContribution} 

\begin{abstract}
	Recent offline meta-reinforcement learning (meta-RL) methods typically utilize task-dependent behavior policies (e.g., training RL agents on each individual task) to collect a multi-task dataset. However, these methods always require extra information for fast adaptation, such as offline context for testing tasks. To address this problem, we first formally characterize a unique challenge in offline meta-RL: transition-reward distribution shift between offline datasets and online adaptation. Our theory finds that out-of-distribution adaptation episodes may lead to unreliable policy evaluation and that online adaptation with in-distribution episodes can ensure adaptation performance guarantee. Based on these theoretical insights, we propose a novel adaptation framework, called In-Distribution online Adaptation with uncertainty Quantification (IDAQ), which generates in-distribution context using a given uncertainty quantification and performs effective task belief inference to address new tasks. We find a return-based uncertainty quantification for IDAQ that performs effectively. Experiments show that IDAQ achieves state-of-the-art performance on the Meta-World ML1 benchmark compared to baselines with/without offline adaptation. 
	
	
	
\end{abstract}

\section{Introduction}

Human intelligence is capable of learning a wide variety of skills from past experiences and can adapt to new environments by transferring skills with limited experience. Current reinforcement learning (RL) has surpassed human-level performance \citep{mnih2015human,silver2017mastering,hafner2019dream}. However, in many real-world applications, RL encounters two major challenges: multi-task efficiency and costly online interactions. In multi-task settings, such as robotic manipulation or locomotion \citep{yu2020meta}, agents are expected to solve new tasks in few-shot adaptation using previously learned knowledge. Moreover, collecting sufficient exploratory interactions is usually expensive or dangerous in robotics \citep{rafailov2021offline}, autonomous driving \citep{yu2018bdd100k}, and healthcare \citep{gottesman2019guidelines}. One popular paradigm for breaking this practical barrier is \textit{offline meta reinforcement learning} \citep[offline meta-RL;][]{li2020focal,mitchell2021offline}, which trains a meta-RL agent with pre-collected offline multi-task datasets and enables fast policy adaptation to unseen tasks.

Recent offline meta-RL methods have been proposed to utilize a multi-task dataset collected by task-dependent behavior policies \citep{li2020focal,dorfman2021offline}. They show promise by solving new tasks with few-shot adaptation. However, existing offline meta-RL approaches require additional information or assumptions for fast adaptation. For example, FOCAL \citep{li2020focal} and MACAW \citep{mitchell2021offline} use offline contexts for meta-testing tasks. BOReL \citep{dorfman2021offline} and SMAC \citep{pong2022offline} employ few-shot online adaptation, in which the former assumes known reward functions, and the latter assumes unsupervised online samples (without rewards) are available in offline meta-training. Therefore, achieving effective online fast adaptation without extra information remains an open problem for offline meta-RL.

\begin{figure*}
	\centering
	\includegraphics[width=\linewidth]{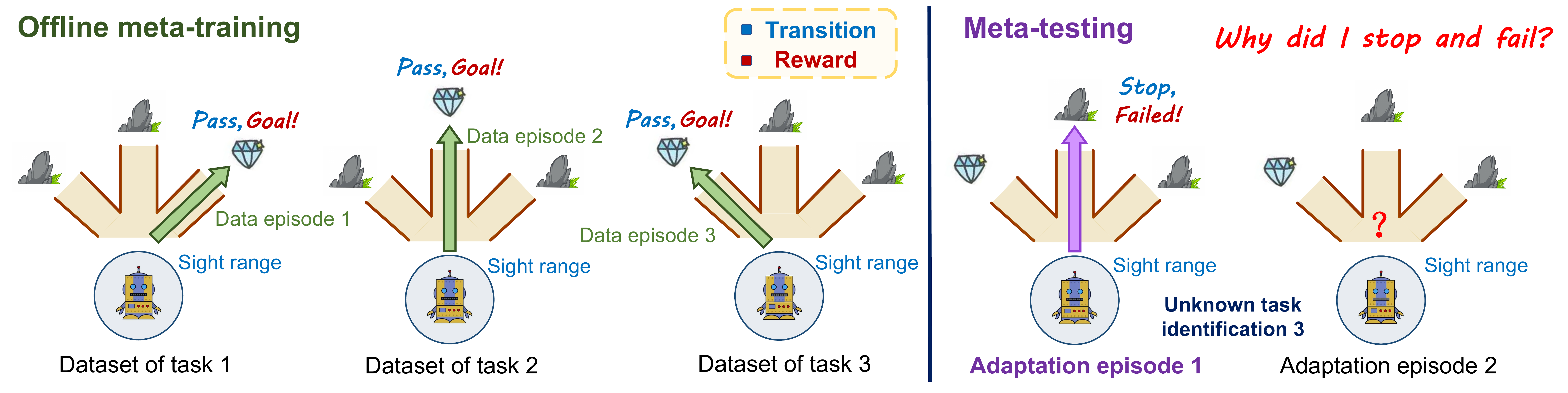}
	\vspace{-0.1in}
	\caption{Illustration of transition-reward distribution shift between offline training and online adaptation.}
	\label{fig:motivation}
	\vspace{-0.15in}
\end{figure*}

To approach meta-testing relying on online experience in offline meta-RL, we first characterize a unique conundrum: \textbf{\textit{transition-reward distribution shift}} between offline datasets and online adaptation, complementary to \textit{state-action distribution shift} in offline RL \citep{levine2020offline}. As illustrated in Figure \ref{fig:motivation}, we propose a motivating example to visualize the transition-reward distribution shift. In this example, the robot aims to choose the correct path to reach the \textit{diamond} (\imgintext{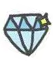}) in three tasks. During task-dependent data collection, the offline multi-task dataset consists of all \textit{successful episodes} (\imgintext{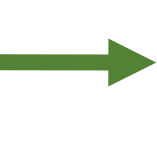}) through expert behavior policies. After offline meta-training on the given dataset, the robot needs to fast adapt to a (unknown) meta-testing task, i.e., task 3 shown in Figure \ref{fig:motivation}. In the first adaptation episode, the robot does not know the identification of meta-testing task. It may try the middle path, stop in front of the stone, and fail. The reward and transition of this failed \textit{adaptation episode} (\imgintext{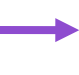}) is out-of-distribution from the offline dataset because the trajectories of the given dataset are successful. This out-of-distribution context will \textit{confuse} the agent in inferring task belief since it is not encountered during offline meta-training. To formalize this phenomenon, we build a theory from the perspective of Bayesian RL \citep[BRL;][]{duff2002optimal,zintgraf2019varibad}, which maintains a task belief given the context history and learns a meta-policy on the belief states. Our theory finds that (i) the transition-reward distribution shift exists and may lead to unreliable policy evaluation, (ii) filtering out out-of-distribution episodes in online adaptation can ensure the performance guarantee, and (iii) meta-policies with Thompson sampling \citep{strens2000bayesian} can generate in-distribution episodes. 

The transition-reward distribution shift induces the \textbf{\textit{inconsistency dilemma}} of experience between offline meta-training and online meta-testing. We can choose either to trust the \textit{offline dataset} (\imgintext{figures/x2.png}) or to trust \textit{new experience} (\imgintext{figures/x1.png}) and continue online exploration. The latter may not be able to collect sufficient data in few-shot adaptation to learn a good policy only on online data. Therefore, we adopt the former strategy and, inspired by our theory, propose a novel context-based online adaptation framework, called \textit{\textbf{I}n-\textbf{D}istribution online \textbf{A}daptation with uncertainty \textbf{Q}uantification} (IDAQ). To align online experience with the offline dataset, IDAQ distinguishes in-distribution context using a given uncertainty quantification, performs task belief updating, and samples ``task hypotheses'' to solve new tasks. We investigate three uncertainty quantifications to measure the confidence that adaptation episodes are in-distribution, and find that IDAQ with a greedy return-based quantification can perform effectively in complex domains. To serve intuitions in Figure \ref{fig:motivation}, IDAQ will continue to sample other ``task hypotheses'' (i.e., try other paths) during meta-testing and infer the unknown task~3 using in-distribution adaptation episode (left).


Our main contribution is to formalize a specific challenge (i.e., transition-reward distribution shift), reveal theoretical insights for offline meta-RL with online adaptation, and furthermore propose a novel in-distribution online adaptation framework with theoretical motivation. To our best knowledge, our method is the first to conduct in-distribution online fast adaptation in offline meta-RL. We extensively evaluate the performance of IDAQ in didactic problems proposed by prior work \citep{rakelly2019efficient,zhang2021metacure} and Meta-World ML1 benchmark with 50 tasks \citep{yu2020meta}. Empirical results show that IDAQ significantly outperforms baselines with fast online adaptation, and achieves better or comparable performance than offline adaptation baselines with expert context.

\vspace{-0.1in}
\section{Notations and Preliminaries}\label{sec:background}

We defer the detailed background to Appendix~\ref{appendix:theory-bg}.

\vspace{-0.1in}
\subsection{Standard Meta-RL} \label{sec:meta-rl}
\vspace{-0.05in}
The standard meta-RL \citep{finn2017model,rakelly2019efficient} deals with a distribution $p(\kappa)$ over \textit{Markov Decision Processes} (MDPs), where each task $\kappa_i\sim p(\kappa)$ presents a finite-horizon MDP \citep{zintgraf2019varibad}, which is defined by a tuple $\left(\gS, \gA, \gR, H, P^{\kappa_i}, R^{\kappa_i}\right)$, including state space $\gS$, action space $\gA$, reward space $\gR$, planning horizon $H$, transition function $P^{\kappa_i}(s'|s,a)$, and reward function $R^{\kappa_i}(r|s,a)$. In this paper, we assume dynamics function $P$ and reward function $R$ may vary across tasks and share a common structure.  The meta-RL algorithms repeatedly sample batches of tasks to train a meta-policy. In meta-testing, the agent aims to rapidly adapt a good policy for new tasks drawn from $p(\kappa)$ within $N$ adaptation episodes.

From a perspective of Bayesian RL \citep[BRL;][]{ghavamzadeh2015bayesian}, recent meta-RL methods \citep{zintgraf2019varibad} utilize a \textit{Bayes-adaptive MDP} \citep[BAMDP;][]{duff2002optimal} to formalize standard meta-RL. A BAMDP is defined as a tuple $M^{+}=\left(\gS^{+},\gA,\gR,H^+,P_0^{+},P^{+},R^{+}\right)$, where $\gS^{+}=\gS\times\gB$ is hyper-state space, $\gB$ is task belief space, a task belief $b$ is the posterior over MDPs given the previous experience, $H^+=N\times H$ is planning horizon, $P_0^+\left(s_0^+\right)$ is initial hyper-state distribution, $P^+\left(s_{t+1}^+\left|s_t^+,a_t,r_t\right.\right)$ is transition function, and $R^+\left(r_t\left|s_t^+,a_t\right.\right)$ is reward function. The objective of meta-RL agents is to find a meta-policy $\pi^+\left(a_t\left|s^+_t\right.\right)$  to maximize the \textit{online policy evaluation} $\gJ_{M^+}\left(\pi^+\right)$.  


\vspace{-0.1in}
\subsection{Offline Meta-RL} \label{sec:offline-meta-rl}
\vspace{-0.05in}
In the offline meta-RL setting \citep{li2020focal}, a meta-learner only has access to an offline multi-task dataset $\gD^+$ and is not allowed to interact with the environment during meta-training. Recent offline meta-RL methods \citep{dorfman2021offline} always utilize task-dependent behavior policies $p(\mu|\kappa)$, which represents the random variable of the behavior policy $\mu(a|s)$ conditioned on the random variable of the task $\kappa$. For brevity, we overload $[\mu]=p(\mu|\kappa)$. Similar to offline RL \citep{yin2021towards}, we assume that $\gD^+$ consists of multiple i.i.d. trajectories that are collected by executing task-dependent policies $[\mu]$ in $M^+$. Denote the reward and transition distribution of the task-dependent offline data collection \citep{jin2021pessimism} by $\sP_{M^+,[\mu]}\left(r_t,s_{t+1}\left|s^+_t,a_t\right.\right)$.

During meta-training, offline RL \citep{liu2020provably,chen2019information} approximates \textit{offline policy evaluation} for a batch-constrained policy $\pi^+$ by sampling from an offline dataset $\gD^+$, which is denoted by $\gJ_{\gD^+}\left(\pi^+\right)$ and called 
\textit{Approximate Dynamic Programming} \citep[ADP;][]{bertsekas1995neuro}. Note that a batch-constrained policy $\pi^+$ only selects actions within the dataset $\gD^+$ to avoid \textit{extrapolation error} \citep{fujimoto2019off}. In meta-testing, RL agents perform online fast adaptation using a meta-trained policy $\pi^+$ in a new task $\kappa_{test}\sim p(\kappa)$. The reward and transition distribution of online data collection in $M^+$ \citep{zintgraf2019varibad} is denoted by $\sP_{M^+,\pi^+}\left(r_t,s_{t+1}\left|s^+_t,a_t\right.\right)$.

\vspace{-0.1in}
\subsection{Offline Meta-Training with Task Embedding} \label{sec:task-embedding}
\vspace{-0.05in}
In this paper, we follow the algorithmic framework of \textit{Task Embeddings for Actor-Critic RL} \citep[PEARL;][]{rakelly2019efficient}. The task identification $\kappa$ is modeled by a latent task embedding $z$, called ``task hypothesis''. The offline meta-training learns a context encoder $q(z|\bm{c})$, a policy $\pi(a|s, z)$, and a value function $Q(s, a, z)$ from a given dataset, where $\bm{c}$ is the context information including states, actions, rewards, and next states. The encoder $q(z|\bm{c})$ infers a task belief about the latent task variable $z$ based on the received context. Denote the prior distribution with $\bm{c}=\emptyset$ by $q(z)$. To distinguish different task identifications from an offline dataset, recent offline meta-RL \citep{li2020focal,yuan2022robust} apply the contrastive loss on the representation of latent task embedding $z$. The policy $\pi$ and value function $Q$ are trained with RL losses on the given $z$. 


\vspace{-0.1in}
\section{Theory: Transition-Reward Distribution Shift in Offline Meta-RL}\label{sec:theory}

Recently, offline meta-RL \citep{dorfman2021offline} faces a new challenge: \textit{transition-reward distribution shift} between offline datasets and online adaptation. We first formalize this data distribution mismatch from the perspective of Bayesian RL \citep[BRL;][]{zintgraf2019varibad} and prove its existence. Our theory shows that the transition-reward distribution shift may lead to unreliable policy evaluation and that in-distribution online adaptation can provide consistent performance guarantee. In addition, we prove that meta-policies with Thompson sampling \citep{strens2000bayesian} can generate in-distribution online adaptation episodes.

\vspace{-0.1in}
\subsection{Transition-Reward Distribution Shift}\label{sec:DSRT}

We define the distributional shift as follows.

\begin{restatable}[Transition-Reward Distribution Shift]{definition}{DSRT}\label{def:DSRT}
	In a BAMDP $M^+$, for each task-dependent behavior policy $[\mu]$ and batch-constrained meta-policy $\pi^+$, the transition-reward distribution shift is defined by that there exists a pair of $\left(s_t^+, a_t\right)$ with executing $\pi^+$ in $M^+$, s.t.,
	\begin{align}\label{eq:DSRT}
	\sP_{M^+,[\mu]}\left(r_t,s_{t+1}\left|s_t^+,a_t\right.\right)\neq\sP_{M^+,\pi^+}\left(r_t,s_{t+1}\left|s_t^+,a_t\right.\right),
	\end{align}
	
	\vspace{-0.2in}
	where $\sP_{M^+,[\mu]},\sP_{M^+,\pi^+}$ are the reward and transition distribution of offline data collection by $[\mu]$ and online data collection by $\pi^+$, respectively, whose formal definition are deferred to Appendix~\ref{appendix:off-meta-rl}.
\end{restatable}

This definition utilizes the discrepancy between offline and online data collection to characterize the joint distribution gap of reward and transition. Note that in offline data collection $\sP_{M^+,[\mu]}$, the behavior policies $p(\mu|\kappa)$ can vary based on task identification, whereas the online data collection $\sP_{M^+,\pi^+}$ is the expected reward and transition distribution across the task distribution $p(\kappa)$.

\begin{restatable}{theorem}{DSRTExists}
	\label{thm:DSRTExists}
	There exists a BAMDP $M^+$ with task-dependent behavior policies $[\mu]$ such that, for any batch-constrained meta-policy $\pi^+$, the transition-reward distribution shift between $\sP_{M^+,[\mu]}$ and $\sP_{M^+,\pi^+}$ occurs. 
\end{restatable}

\begin{figure}
	\centering
	\vspace{-0.1in}
	\input{figures/mdps/mdp-1}
	\vspace{-0.03in}
	$p(\kappa_i) = \frac{1}{v}, p(\mu_i|\kappa_i)=1, \text{and }\mu_i(a_i|s_0)=1$
	\caption{A concrete example, which has $v$ meta-RL tasks, one state, $v$ actions, $v$ behavior policies, horizon $H=1$ in an episode, and $v$ adaptation episodes, where $v\ge 3$.}
	\label{fig:MultiTaskDDMExample}
	\vspace{-0.1in}
\end{figure}

To prove the existence of distributional shift, we construct an offline meta-RL setting shown in Figure~\ref{fig:MultiTaskDDMExample}, which has $v$ meta-RL tasks $\left\{\kappa_1,\dots,\kappa_v\right\}$, $v$ behavior policies $\left\{\mu_1,\dots,\mu_v\right\}$, and $v\ge 3$. The task distribution $p(\kappa)$ is uniform, the behavior policy of task $\kappa_i$ is $\mu_i$, and each behavior policy $\mu_i$ will perform $a_i$. After data collection, RL agents will offline meta-train policies on a given dataset $\gD^+$ and fast adapt to a meta-testing task $\kappa_{test}\sim p(\kappa)$ within $v$ online episodes.

In this example, for any action $a_i$ in $s_0^+$, the offline data collection $\sP_{M^+,[\mu]}\left(r=1\left|s_0^+,a_i\right.\right)=1$ since expert task-dependent behavior policies all collect data with reward~1. During online meta-testing, for any batch-constrained meta-policy $\pi^+$ selects an action $a_i$ in $s_0^+$, the online data collection $\sP_{M^+,\pi^+}\left(r=1\left|s_0^+,a_i\right.\right)=\frac{1}{v}$ because there is the probability of $\frac{1}{v}$ to sample a meta-testing task $\kappa_i$, whose reward function of $a_i$ is 1. Thus, $\sP_{M^+,[\mu]}\neq\sP_{M^+,\pi^+}$.

\vspace{-0.1in}
\subsection{Data Distribution Matters for Online Adaptation}\label{sec:dis-matters}

To investigate the impact of data distribution mismatch, we analyze the gap of policy evaluation between offline dataset $\gJ_{\gD^+}$ and online adaptation $\gJ_{M^+}$ in offline meta-RL.

\begin{restatable}{proposition}{DSRTOOD}
	\label{prop:DSRTOOD}
	There exists a BAMDP $M^+$ with task-dependent behavior policies such that, for any batch-constrained meta-policy $\pi^+$, (i) RL agents will visit out-of-distribution hyper-states and (ii) the gap between offline policy evaluation $\gJ_{\gD^+}\left(\pi^+\right)$ and online policy evaluation $\gJ_{M^+}\left(\pi^+\right)$ is at least $\frac{H^+-1}{2}$.
\end{restatable}

Proposition \ref{prop:DSRTOOD} states that RL agents will go out of the distribution of the offline dataset $\gD^+$ due to the shifts in the reward and transition distribution. Thus, the offline policy evaluation of $\pi^+$ in meta-training cannot provide a reference for the online mest-testing. For example in Figure \ref{fig:MultiTaskDDMExample}, the agent will visit out-of-distribution belief states when receiving reward~0 with probability $1-\frac{1}{v}$ in $s_0^+$. In addition, the offline policy evaluation $\gJ_{D^+}\left(\pi^+\right)=H^+=v$ since the dataset $D^+$ only contains reward $1$. For each $\pi^+$, we have $\gJ_{M^+}\left(\pi^+\right)\le\frac{H^++1}{2}$ and detailed proof is deferred to Appendix~\ref{appendix:theory-main-results-part1}. Hence, the gap between $\gJ_{\gD^+}\left(\pi^+\right)$ and $\gJ_{M^+}\left(\pi^+\right)$ is at least $\frac{H^+-1}{2}$. 

To address this inconsistency dilemma, we choose to trust the offline dataset within few-shot online adaptation and derive the following theorem.

\begin{restatable}{theorem}{DSRTID}\label{thm:DSRTID}
	In a BAMDP $M^+$, for each task-dependent behavior policy $[\mu]$, denoting a transformed BAMDP $\overline{M}^{+}$ by incorporating $[\mu]$ into the belief of $M^+$, we have (i) for feasible Bayesian belief updating, $\overline{M}^{+}$ confines the agent in the in-distribution hyper-states, (ii) for each $\bar{\pi}^+$, the distribution of $\sP_{\overline{M}^{+},[\mu]}$ and $\sP_{\overline{M}^{+},\bar{\pi}^+}$ matches, and (iii) policy evaluation $\gJ_{\overline{\gD}^+}\left(\bar{\pi}^+\right)$ and $\gJ_{\overline{M}^{+}}\left(\bar{\pi}^+\right)$ will be asymptotically consistent, as the offline dataset grows.
\end{restatable}

To achieve in-distribution online adaptation, transformed BAMDPs incorporate additional information about offline data collection into the beliefs of BAMDPs. We prove that transformed BAMDPs require RL agents to filter out out-of-distribution episodes to support feasible belief updating of behavior policies. In this way, the distribution of reward and transition between offline and online data collection coincides, which can provide the guarantee of consistent policy evaluation between $\gJ_{\overline{\gD}^+}\left(\bar{\pi}^+\right)$ and $\gJ_{\overline{M}^{+}}\left(\bar{\pi}^+\right)$. Theorem \ref{thm:DSRTID} shows that we can meta-train policies with offline policy evaluation and utilize in-distribution online adaptation to guarantee the final performance in meta-testing.

\vspace{-0.1in}
\subsection{Generating In-Distribution Online Adaptation}\label{sec:generate-in-dis}

In this subsection, we will incorporate Thompson sampling \citep{strens2000bayesian} into the meta-policies to generate in-distribution episodes during online adaptation as follows.

\begin{restatable}{theorem}{ThompsonGenID}\label{thm:ThompsonGenID}
	In a transformed BAMDP $\overline{M}^{+}$, for each batch-constrained meta-policy with Thompson sampling $\bar{\pi}^{+,T}$ in a meta-testing task $\kappa_{test}\sim p(\kappa)$, there exists a task hypothesis from the current belief, executing $\bar{\pi}^{+,T}$ in $\kappa_{test}$ can generate in-distribution online adaptation episodes with high probability, as the offline dataset grows.
\end{restatable}

\vspace{-0.03in}
Theorem \ref{thm:ThompsonGenID} indicates that for each adaptation episode, we can sample task hypotheses from the current task belief and execute $\bar{\pi}^{+,T}$ to interact with the environment until finding an in-distribution episode. For example in Figure~\ref{fig:MultiTaskDDMExample}, after offline meta-training, a meta-policy with Thompson sampling $\bar{\pi}^{+,T}$ will perform $a_i$ with a task hypothesis of $\kappa_i$ and expect to receive a reward 1. During online meta-testing, $\kappa_{test}$ is drawn from $p(\kappa)$ and the agent needs to infer the task identification. To achieve in-distribution online adaptation, $\bar{\pi}^{+,T}$ will try various actions according to diverse task hypotheses until sampling an in-distribution episode with a reward 1. Updating the task belief with the in-distribution episode, RL agents can infer and solve this task.

In contrast, when updating task belief using an out-of-distribution episode with a reward 0, the posterior task belief will be out of the offline dataset $\overline{\gD}^+$. Note that offline training paradigm can not well-optimize $\bar{\pi}^{+,T}$ on out-of-distribution states \citep{fujimoto2019off} and policy $\bar{\pi}^{+,T}$ will fail in this case. Moreover, Thompson sampling is very popular in context-based deep meta-RL \citep{rakelly2019efficient} and we will generalize these theoretical implications.

\vspace{-0.1in}
\section{IDAQ: In-Distribution Online Adaptation with Uncertainty Quantification}\label{sec:IDAQ}

In our setting, offline meta-RL contains two phases: offline meta-training and online adaptation. For offline meta-training, we employ an off-the-shelf context-based algorithm, e.g., FOCAL \citep{li2020focal}, which follows the learning paradigm of latent task embedding (see Section \ref{sec:task-embedding}). In this section, we will focus on investigating a practical scheme to address the major challenge of the transition-reward distribution shift during online adaptation. Motivated by our theory in Section~\ref{sec:theory}, we aim to \textbf{\textit{distinguish}} whether an adaptation episode is in the distribution of the offline dataset, and utilize meta-policies with Thompson sampling \citep{strens2000bayesian} to generate in-distribution online adaptation. Therefore, we will introduce a novel context-based online adaptation algorithm, called \textit{\textbf{I}n-\textbf{D}istribution online \textbf{A}daptation with uncertainty \textbf{Q}uantification} (IDAQ), which infers in-distribution context for solving meta-testing tasks. The overall algorithm of IDAQ is illustrated in Algorithm~\ref{alg:idaq}. IDAQ consists of two main components: (i) a general in-distribution online adaptation framework, and (ii) a plug-in uncertainty quantification function. We will describe these components in detail as follows.

\vspace{-0.1in}
\subsection{In-Distribution Online Adaptation Framework}
\vspace{-0.05in}
As motivated by Theorem~\ref{thm:ThompsonGenID}, our adaptation protocol adopts the popular framework of \textit{Thompson sampling} \citep{rakelly2019efficient} for online meta-testing. IDAQ will iteratively update posterior task belief based on online interactions with environment and execute the meta-policy with a sampled ``task hypothesis''. For in-distribution online adaptation, IDAQ utilizes a given uncertainty quantification $\sQ(\tau)$ to empirically measure the confidence that online experience $\{\tau_i\}$ are in-distribution. Note that the context encoder $q(z|\bm{c})$ (i.e., a task inference module) is meta-trained in the offline dataset and cannot handle out-of-distribution adaptation \citep{mendonca2020meta}, where $\bm{c}$ is the episode-based context. To realize reliable task belief updating, IDAQ needs to estimate a reference threshold $\delta$ and defines the in-distribution context 
\begin{align}\label{eq:in-dis-context}
	\bm{c}_{in} = \left\{\tau_i\left|\sQ(\tau_i)\le \delta, \forall \tau_i \in \bm{c}\right.\right\}.
\end{align}
In this way, IDAQ will perform a two-stage paradigm of online adaptation: (i) a \textit{Reference Stage} to estimate the uncertainty threshold $\delta$ and (ii) an \textit{Iterative Updating Stage} to update the in-distribution context $\bm{c}_{in}$, posterior task belief $q(z|\bm{c}_{in})$, and execution meta-policy $\pi(a|s, z)$.

\textbf{Reference Stage} collects $n_r$ online adaptation episodes $\left\{\tau_i\right\}_{i=1}^{n_r}$ using the prior task distribution $q(z)$ and meta-policy $\pi(a|s, z)$ in a meta-testing task $\kappa_{test}$. IDAQ will calculate the in-distribution confidence of adaptation episodes $\left\{\sQ(\tau_i)\right\}_{i=1}^{n_r}$ and estimate the reference threshold $\delta$ that is the bottom $k\%$-quantile of $\left\{\sQ(\tau_i)\right\}_{i=1}^{n_r}$, where $k$ is a hyperparameter to divide the range of uncertainty of in-distribution episodes. Hence, IDAQ can derive the in-distribution context $\bm{c}_{in}$ and posterior task belief $q(z|\bm{c}_{in})$.

\textbf{Iterative Updating Stage} will update the posterior task belief $q(z|\bm{c}_{in})$ in $n_i$ iterations. In each iteration, IDAQ collects an online adaptation episode $\tau_j$ using the current task belief $q(z|\bm{c}_{in})$ and meta-policy $\pi(a|s, z)$ in $\kappa_{test}$. When the uncertainty of this episode $\sQ(\tau_j)$ is less than the reference threshold $\delta$, IDAQ will update the in-distribution context, i.e., $\bm{c}_{in}\leftarrow\bm{c}_{in}\cup\{\tau_j\}$, and derive the posterior task belief $q(z|\bm{c}_{in})$. The final policy $\pi_{out}(a|s, z)$ is executed with the total in-distribution context $\bm{c}_{in}$.

\begin{algorithm}
	\caption{IDAQ: In-Distribution online Adaptation with uncertainty Quantification}\label{alg:idaq}
	\begin{algorithmic}[1]
		\STATE {\bfseries Require:} An offline dataset $\gD^+$, a meta-testing task $\kappa_{test}$, the number of iterations $n_i$, a context-based offline meta-training algorithm $\sA$ (i.e., FOCAL), and an in-distribution uncertainty quantification $\sQ$
		\STATE Offline meta-train a context encoder $q(z|\bm{c})$ and a meta-policy $\pi(a|s, z)$ using an algorithm $\sA$ in a dataset $\gD^+$ \COMMENT{\textbf{\textit{Offline meta-training}}}
		\STATE Perform reference stage of online adaptation and estimate the in-distribution threshold $\delta$ using $\sQ$ \COMMENT{\textbf{\textit{Start online meta-testing}}}
		\STATE Derive the in-distribution context $\bm{c}_{in}$ with Eq. (\ref{eq:in-dis-context}) and posterior task belief $q(z|\bm{c}_{in})$
		\FOR[\textbf{\textit{Iterative updating stage}}]{$t=1\dots n_i$}
		\STATE Collect an online adaptation episode using the posterior task belief $q$ and meta-policy $\pi$ in $\kappa_{test}$
		\STATE Update the in-distribution context $\bm{c}_{in}$ using $\sQ,\delta$ and derive the posterior task belief $q(z|\bm{c}_{in})$
		\ENDFOR
            \STATE {\bfseries Return:} $\pi$, $q(z|\bm{c}_{in})$
	\end{algorithmic}
\end{algorithm}

\vspace{-0.1in}
\subsection{Uncertainty Quantification}
\vspace{-0.05in}
\label{quant}
Uncertainty quantification is a popular tool for empirically measuring the confidence that data is in the distribution of offline RL \cite{yu2020mopo} or noisy oracle \cite{ren2022efficient}. In this subsection, we will analyze three practical uncertainty quantifications to adapt IDAQ to complex domains: \textit{Prediction Error}, \textit{Prediction Variance}, and \textit{Return-based}. The empirical evaluation is deferred to Section \ref{equant}.

To realize the uncertainty quantification of prediction error and prediction variance, we adopt a model-based approach to learn an ensemble of $L$ reward and dynamics models $\left\{r_{\phi_i}(s,a,z), p_{\psi_i}(s,a,z)\right\}_{i=1}^{L}$ according to the latent task embedding $z$. We parameterize them by $\{\phi_i,\psi_i\}_{i=1}^{L}$ and optimize $\left\{r_{\phi_i}, p_{\psi_i}\right\}_{i=1}^{L}$ on the offline multi-task dataset $\gD^+$ by minimizing the MSE loss function during meta-training. Formal loss function is deferred to Appendix \ref{appendix-sec:mse-loss}.

\textbf{Prediction Error} quantifies the model error to estimate the confidence that data is trained during offline meta-training. This metric is also called \textit{``curiosity''}, a popular intrinsic reward in exploration of single-task RL \citep{pathak2017curiosity}, which encourages the agent to visit new areas. In offline meta-RL, we utilize this quantification to filter out out-of-distribution adaptation episodes and denote by
\vspace{-0.03in}
\begin{align}
	\sQ_{PE}(\tau_i, z) =&~ \frac{1}{HL} \sum_{t=0}^{H-1}\sum_{i=1}^{L} \left|r_t-r_{\phi_i}(s_t,a_t,z)\right|  \\
	 &~\qquad\qquad~~~~~+\left\|s_{t+1}-p_{\psi_i}(s_t,a_t,z)\right\|_2, \nonumber
\end{align}

\vspace{-0.1in}
where $z$ is the ``task hypothesis'' of the episode $\tau_i$ in IDAQ. $\sQ_{PE}$ averages model errors across timesteps and an ensemble. The challenge of $\sQ_{PE}$ is that the hyperparameter $k$ for the reference threshold $\delta$ will be sensitive for different multi-task datasets since the model error of in-distribution episodes can be various.

\textbf{Prediction Variance} captures the epistemic and aleatoric uncertainty of the true models using a bootstrap ensemble \citep{yu2020mopo}. This metric is popular to measure whether data is in the dataset of offline single-task RL \citep{kidambi2020morel}. In offline meta-RL, we denote this quantification by
\vspace{-0.05in}
\begin{align}
	\sQ_{PV}(\tau_i, z) =&~ \frac{1}{H}\sum_{t=0}^{H-1}\max_{i,j}  \left|r_{\phi_i}(s_t,a_t,z)-r_{\phi_j}(s_t,a_t,z)\right|  \nonumber\\
	&~+\left\|p_{\psi_i}(s_t,a_t,z)-p_{\psi_j}(s_t,a_t,z)\right\|_2,
\end{align}
where $z$ is the ``task hypothesis'' of $\tau_i$. $\sQ_{PV}$ averages the ensemble discrepancy across timesteps. However, $\sQ_{PV}$ cannot handle cases with higher prediction error and lower prediction variance. For example in Figure \ref{fig:MultiTaskDDMExample}, the learned reward model will deterministically output 1 for each action, in which the prediction error is $1-\frac{1}{v}$ with no variance.

\textbf{Return-based} uncertainty quantification is our newly designed metric for offline meta-RL with medium or expert datasets. To address the limitations of prediction error and prediction variance, we utilize a bias of offline RL \citep{fujimoto2019off} that few-shot out-of-distribution episodes generated by an offline-learned meta-policy $\pi$ usually have lower returns since offline meta-training can not well-optimize meta-policies on out-of-distribution states. Its contrapositive statement is that executing $\pi$ with higher returns has a higher probability of being in-distribution and online policy evaluation of $\pi$ presents a good in-distribution confidence:
\vspace{-0.05in}
\begin{align}
\sQ_{RE}\left(\left\{\tau_i\right\}_{i=1}^{n_e}\right) =&~- \frac{1}{n_e}\sum_{i=1}^{n_e}\sum_{t=0}^{H-1} r_t^i,
\end{align}

\vspace{-0.15in}
where $n_e$ is the number of episodes generated by $\pi$ to approximate the online policy evaluation. With the mild assumption (i.e., the bias of offline RL above), we can prove that return-based uncertainty quantification $\sQ_{RE}$ can theoretically derive in-distribution contexts using Theorem \ref{thm:ThompsonGenID}. The formal analysis is deferred to Appendix \ref{appendix:omitted-section-IDAQ}. Moreover, in empirical, IDAQ can adopt a conservative (i.e., low) reference threshold $\delta$ to achieve in-distribution online adaptation. In this case, IDAQ may neglect some in-distribution episodes with lower returns. We will argue that, in medium or expert datasets, our method can utilize informative episodes with higher returns to perform task inference. It is an interesting and exciting future direction to differentiate in-distribution episodes with lower returns in offline meta-RL with online adaptation. 

\vspace{-0.1in}
\section{Experiments}

In this section, we first evaluate the three uncertainty quantifications mentioned in Section \ref{quant} and empirically demonstrate that the \textbf{Return-based} quantification works the best on various task sets.
Then we conduct large-scale experiments on Meta-World ML1\citep{yu2019meta}, a popular meta-RL benchmark that consists of 50 robot arm manipulation task sets. Finally, we perform ablation studies to analyze IDAQ's sensitivity to hyper-parameter settings and \textit{dataset qualities}. Datasets are collected by script policies that solve corresponding tasks.
We compare against FOCAL \citep{li2020focal} and MACAW \citep{mitchell2021offline}, as well as their online adaptation variants. We also compare against BOReL \citep{dorfman2021offline}
. For a fair comparison, we evaluate a variant of BOReL that does not utilize oracle reward functions, as introduced in the original paper \citep{dorfman2021offline}. FOCAL is built upon PEARL \citep{rakelly2019efficient} and uses contrastive losses to learn context embeddings, while MACAW is a MAML-based \citep{finn2017model} algorithm and incorporates AWR \citep{peng2019advantage}. Both FOCAL and MACAW are originally proposed for the offline adaptation settings  (i.e., with expert context). For online adaptation, we use online experience instead of expert contexts, and adopt the adaptation protocol of PEARL and MAML, respectively. Evaluation results are averaged over six random seeds, and variance is measured by 95\% bootstrapped confidence interval. Detailed hyper-parameter settings are deferred to Appendix \ref{exp-app1}. A didactic example that empirically demonstrates the distributional shift problem proposed in Section \ref{sec:theory} is deferred to Appendix \ref{sec:didactic-example}. An open-source implementation of our algorithm is available online\footnote{\url{https://github.com/NagisaZj/IDAQ_Public}}.

\vspace{-0.1in}
\subsection{Evaluation of Uncertainty Quantifications}\label{equant}

We evaluate the three uncertainty quantifications mentioned in Section \ref{quant} on some representative tasks. As shown in Table \ref{tab:51}, the \textbf{Return-based} quantification significantly outperforms the other two quantifications and the baseline algorithm FOCAL (which uses all online experiences as contexts). To further investigate how these quantifications behave, we illustrate the uncertainty quantifications of various episodes collected in the reference stage on one of the meta-training tasks. As shown in Figure \ref{fig:pe}, the \textbf{Prediction Error} quantification cannot find a good reference threshold to distinguish in-distribution episodes. Figure \ref{fig:pv} shows that the \textbf{Prediction Variance} quantification fails and may suffer from situations with higher prediction error and lower prediction variance in the medium or expert datasets. Figure \ref{fig:pp3} illustrates the minimal distance between episodes collected in the reference stage and the offline dataset on one of the meta-training tasks. Results show that episodes with higher returns are closer to the medium or expert datasets, which implies that the \textbf{Return-based} quantification can correctly identify in-distribution episodes. Formal distance function and additional visualizations of other tasks are deferred to Appendix \ref{formal} and \ref{avr}, respectively.

\begin{figure*}
	\centering
	\vspace{-0.1in}
\subfigure[]{\includegraphics[width=0.28\linewidth]{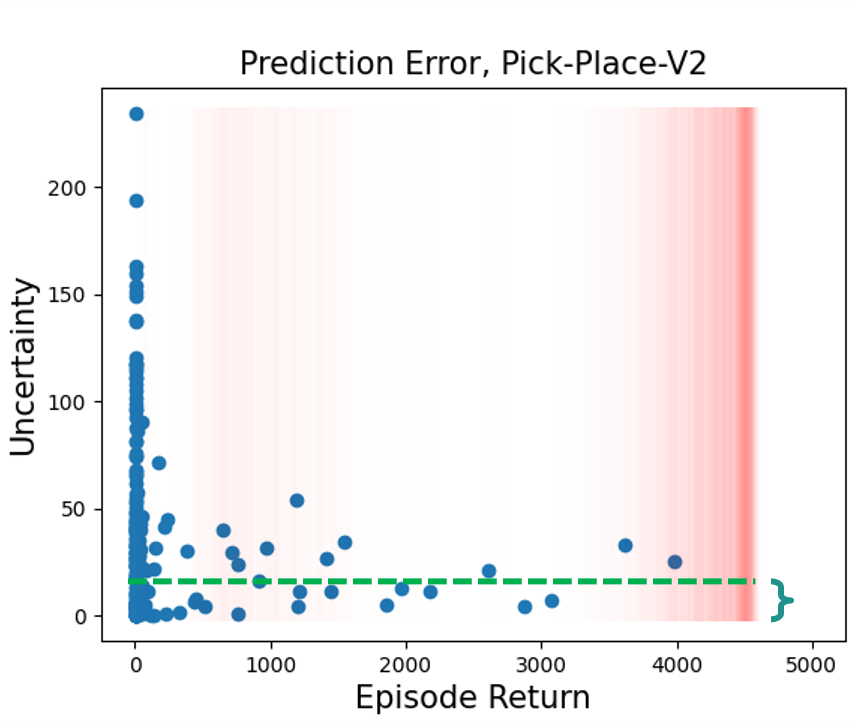}\label{fig:pe}}
\subfigure[]
{\includegraphics[width=0.28\linewidth]{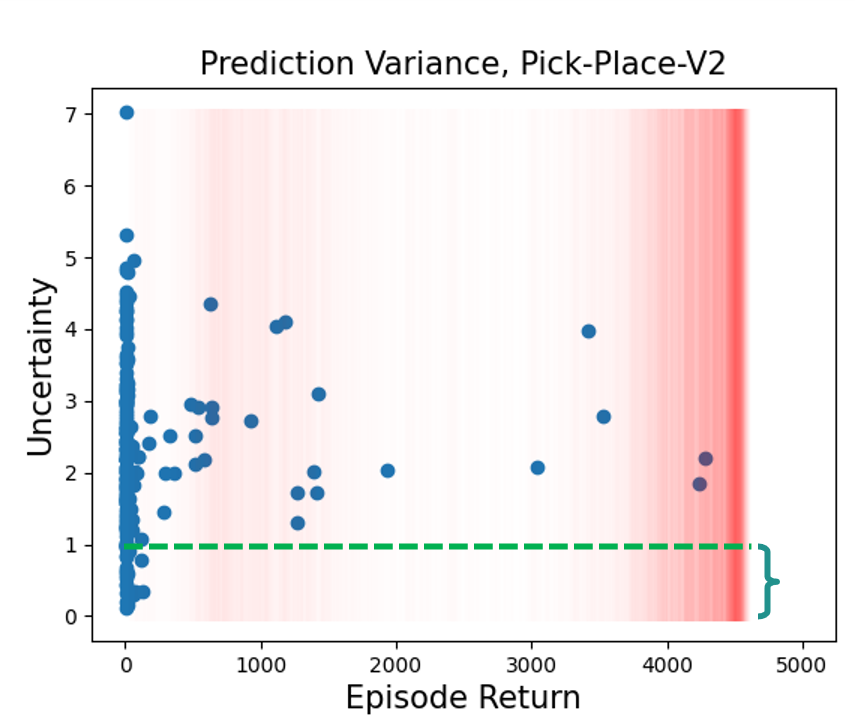}\label{fig:pv}} 
\subfigure[]{\includegraphics[width=0.28\linewidth]{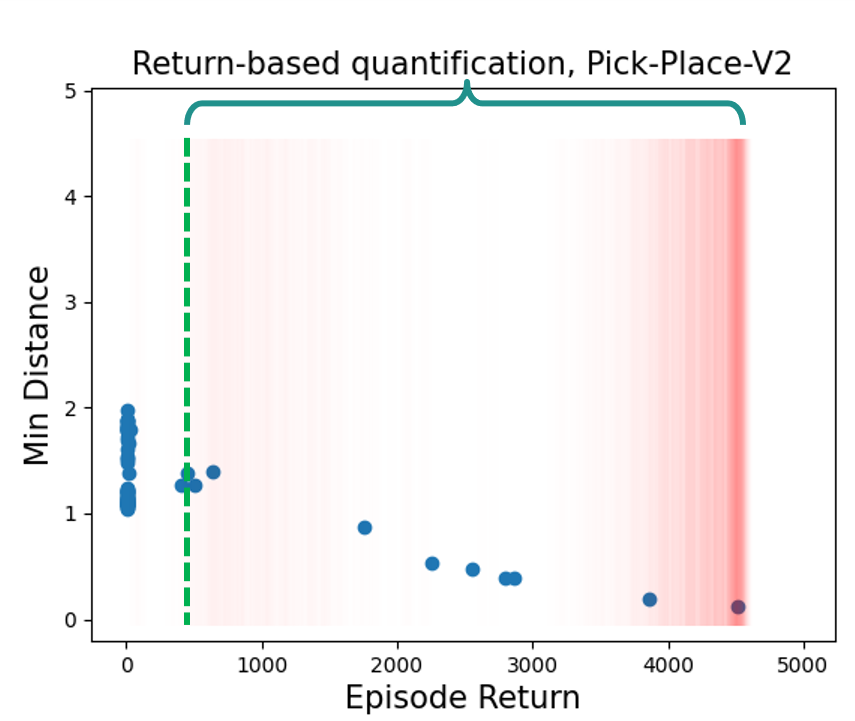}\label{fig:pp3}} 
\includegraphics[width=0.1\linewidth]{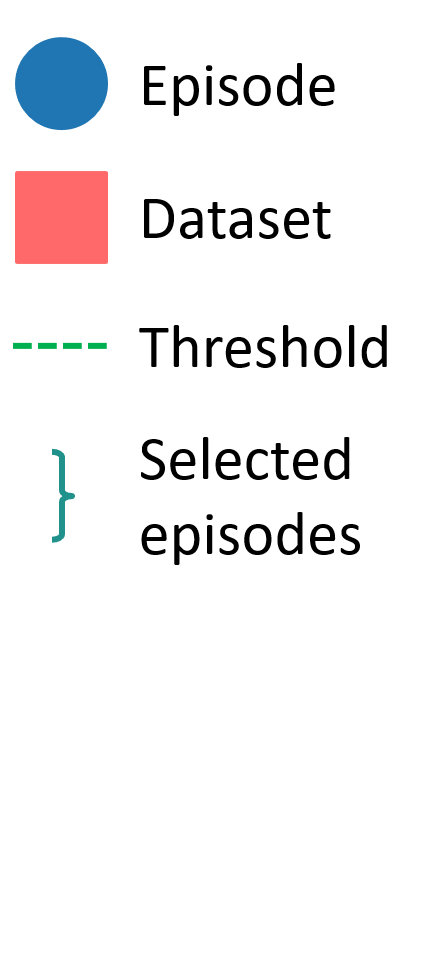}
\vspace{-0.05in}
	\caption{(a) and (b) illustrate the uncertainty quantifications of \textbf{Prediction Error} and \textbf{Prediction Variance} on episodes collected in the reference stage on one of the meta-training tasks, respectively. Red shades indicate the density of episode returns in the dataset. Both quantifications give low uncertainty measures on the out-of-distribution episodes in the bottom left corner, and fail to identify in-distribution episodes. (c) illustrates the minimal distance between episodes collected in the reference stage and the offline dataset on one of the meta-training tasks. The dotted green line illustrates the reference threshold found by each quantification. Results imply that the \textbf{Return-based} quantification successfully identifies in-distribution episodes.
 }
\end{figure*}

\begin{table*}[t]
	\centering
        \vspace{-0.15in}
	\caption{Performance of the three uncertainty quantifications and FOCAL on example tasks, a bunch of Meta-World ML1 tasks with normalized scores. ``IDAQ+Return'' is short for IDAQ with the \textbf{Return-based} quantification. For Meta-World tasks, ``-V2'' is omitted for brevity. ``Med'' represents results trained on medium quality datasets.
 }
	\begin{tabular}{l|c|c|c|c}
		\toprule
		Example Env & IDAQ+Prediction Error & IDAQ+Prediction Variance & IDAQ+Return & FOCAL\\
		\midrule
   	{Push} & 0.31$~\pm~$0.13 &  0.13$~\pm~$0.07 &  \textbf{0.55}$~\pm~$0.10 &0.34 $~\pm~$0.14 \\
   {Pick-Place} & 0.07$~\pm~$0.05&  0.04$~\pm~$0.03  &  \textbf{0.20}$~\pm~$0.03 & 0.07 $~\pm~$0.02 \\
   {Soccer} &0.18$~\pm~$0.03  &  0.23$~\pm~$0.03  &  \textbf{0.44}$~\pm~$0.04 & 0.11$~\pm~$0.03 \\
   {Drawer-Close}& \textbf{1.00}$~\pm~$0.00&\textbf{0.99}$~\pm~$0.01&\textbf{0.99}$~\pm~$0.02 & \textbf{0.96} $~\pm~$ 0.04 \\  
   {Reach}& \textbf{0.87}$~\pm~$0.01 & 0.49$~\pm~$0.03 & \textbf{0.85}$~\pm~$0.03 & 0.62$~\pm~$0.05\\
   \midrule
   {Sweep (Med)} & 0.15 $~\pm~$ 0.03 & 0.06 $~\pm~$ 0.02 &  \textbf{0.59}$~\pm~$0.13 & 0.38 $~\pm~$ 0.13\\
   {Peg-Insert-Side (Med)} &0.03 $~\pm~$ 0.02  & 0.03 $~\pm~$ 0.01 &  \textbf{0.30}$~\pm~$0.14  & 0.10 $~\pm~$ 0.07\\
   \midrule
   
		Point-Robot &  -5.70$~\pm~$0.05  & -21.29$~\pm~$0.85 & \textbf{-5.10}$~\pm~$0.26 & -15.38 $~\pm~$ 0.95  \\

		\bottomrule
	\end{tabular}
	\label{tab:51}
 \vspace{-0.1in}
\end{table*}

\vspace{-0.1in}
\subsection{Main Results}

Following results in Section \ref{equant}, we use the \textbf{Return-based} quantification as the default quantification for IDAQ in the following experiments.
We evaluate on Meta-World ML1\citep{yu2019meta}, a popular meta-RL benchmark that consists of 50 robot arm manipulation task sets. Each task set consists of 50 tasks with different goals. For each task set, we use 40 tasks as meta-training tasks, and remain the other 10 tasks as meta-testing tasks. As shown in Table \ref{tab:avg}, IDAQ significantly outperforms baselines under the online context setting. With expert contexts, FOCAL and MACAW both achieve reasonable performance. IDAQ achieves better or comparable performance to baselines with expert contexts, which implies that expert contexts may not be necessary for offline meta-RL. Under online contexts, FOCAL fails due to the data distribution mismatch between offline training and online adaptation. MACAW has the ability of online fine-tuning as it is based on MAML, but it also suffers from the distribution mismatch problem, and online fine-tuning can hardly improve its performance within a few adaptation episodes.  BOReL fails on most of the tasks, as BOReL without oracle reward functions will also suffer from the distribution mismatch problem, which is consistent with the results in the original paper.

\begin{table*}[t]
	\centering
        \vspace{-0.05in}
	\caption{Algorithms' normalized scores averaged over 50 Meta-World ML1 task sets. Scores are normalized by expert-level policy return.}
	\begin{tabular}{c|c|c|c|c|c}
		\toprule
		 IDAQ & FOCAL  & MACAW &\makecell[c]{FOCAL with \\Expert Context}&\makecell[c]{MACAW with \\Expert Context} & BOReL\\
		\midrule\textbf{0.73}$~\pm~$0.07&0.53$~\pm~$0.1&0.18$~\pm~$0.1&0.67$~\pm~$0.07&0.68$~\pm~$0.07&0.04$~\pm~$0.01\\
		\bottomrule
	\end{tabular}
	\label{tab:avg}
\end{table*}

Table \ref{tab:maze} shows algorithms' performance on 20 representative Meta-World ML1 task sets, as well a sparse-reward version of Point-Robot and Cheetah-Vel, which are popular meta-RL tasks \citep{li2020focal}. IDAQ achieves remarkable performance in most tasks and may fail in some hard tasks as offline meta-training is difficult. We also find that IDAQ achieves better or comparable performance to baselines with expert contexts on 33 out of the 50 task sets. Detailed algorithm performance on all 50 tasks and comparison to baselines with expert contexts are deferred to Appendix \ref{ss2}. 

\begin{table*}[t]
	\centering
        \vspace{-0.1in}
	\caption{Performance on example tasks, a bunch of Meta-World ML1 tasks with normalized scores.
 }
	\begin{tabular}{l|c|c|c|c}
		\toprule
		Example Env & IDAQ & FOCAL & MACAW & BOReL\\
		\midrule
   {Coffee-Push}&\textbf{1.26}$~\pm~$0.13&0.66$~\pm~$0.07&0.01$~\pm~$0.01 & 0.00$~\pm~$0.00\\
   	{Faucet-Close} &  \textbf{1.12}$~\pm~$0.01 &  1.06$~\pm~$0.02 &  0.07$~\pm~$0.01   & 0.13$~\pm~$0.03\\
   {Faucet-Open} & \textbf{1.05}$~\pm~$0.02 &  1.01$~\pm~$0.02 &  0.08$~\pm~$0.04  & 0.12$~\pm~$0.05 \\
   {Door-Close} & \textbf{0.99}$~\pm~$0.00 &  0.97$~\pm~$0.01 &  0.00$~\pm~$0.00  & 0.37$~\pm~$0.19 \\
   {Drawer-Close}&\textbf{0.99}$~\pm~$0.02 &\textbf{0.96}$~\pm~$0.04 &0.53$~\pm~$0.50  & 0.00$~\pm~$0.00  \\  
   {Door-Lock} & \textbf{0.97}$~\pm~$0.01 &  0.90$~\pm~$0.02 &  0.25$~\pm~$0.11   & 0.14$~\pm~$0.00 \\
   {Plate-Slide-Back} & \textbf{0.96}$~\pm~$0.02 &  0.58$~\pm~$0.06 &  0.21$~\pm~$0.17  & 0.01$~\pm~$0.00\\
   Dial-Turn &  \textbf{0.91}$~\pm~$0.05 &  0.84$~\pm~$0.09 &  0.00$~\pm~$0.00   & 0.00$~\pm~$0.00 \\
   {Handle-Press} & \textbf{0.88}$~\pm~$0.05 &  \textbf{0.87}$~\pm~$0.02 &  0.28$~\pm~$0.10  & 0.01$~\pm~$0.00\\
   {Hammer} & \textbf{0.84}$~\pm~$0.06 &  0.59$~\pm~$0.07 &  0.10$~\pm~$0.01 & 0.09$~\pm~$0.01   \\
		{Button-Press} &  \textbf{0.74}$~\pm~$0.08 &  \textbf{0.68}$~\pm~$0.14 &  0.02$~\pm~$0.01 & 0.01$~\pm~$0.01 \\ 
   Push-Wall&\textbf{0.71}$~\pm~$0.15 &0.43$~\pm~$0.06 &0.23$~\pm~$0.18   & 0.00$~\pm~$0.00  \\
		{Hand-Insert} &  \textbf{0.63}$~\pm~$0.04 &  0.29$~\pm~$0.07 &  0.02$~\pm~$0.01  & 0.00$~\pm~$0.00 \\
   {Peg-Unplug-Side} & \textbf{0.56}$~\pm~$0.07 &  0.19$~\pm~$0.09 &  0.00$~\pm~$0.00 & 0.00$~\pm~$0.00  \\
   {Bin-Picking} &  0.53$~\pm~$0.16 &  0.31$~\pm~$0.21 &  \textbf{0.66}$~\pm~$0.11  & 0.00$~\pm~$0.00\\
  Soccer&\textbf{0.44}$~\pm~$0.04 &0.11$~\pm~$0.03 &\textbf{0.38}$~\pm~$0.31   & 0.04$~\pm~$0.02  \\

   Coffee-Pull &  \textbf{0.40}$~\pm~$0.05 &   0.23$~\pm~$0.04 &  0.19$~\pm~$0.12  & 0.00$~\pm~$0.00 \\
   {Pick-Place-Wall} & 0.28$~\pm~$0.12 &  0.09$~\pm~$0.04 & \textbf{  0.39}$~\pm~$0.25  & 0.00$~\pm~$0.00  \\
Pick-Out-Of-Hole&
   0.26$~\pm~$0.25 &0.16$~\pm~$0.16 &\textbf{0.59}$~\pm~$0.06   & 0.00$~\pm~$0.00 \\ 
  
Handle-Pull-Side&
   \textbf{0.14}$~\pm~$0.04 &\textbf{0.13}$~\pm~$0.09 &0.00$~\pm~$0.00  & 0.00$~\pm~$0.00 \\   

   \midrule
   
		Cheetah-Vel &  \textbf{-171.5}$~\pm~$22.00 & -287.7$~\pm~$30.6 & -234.0$~\pm~$23.5  & -301.4$~\pm~$36.8\\
		Point-Robot &  \textbf{-5.10}$~\pm~$0.26 & -15.38$~\pm~$0.95 & -14.61$~\pm~$0.98  & -17.28$~\pm~$1.16 \\
		Point-Robot-Sparse &  \textbf{7.78} $~\pm~$0.64 & 0.83$~\pm~$0.37 & 0.00$~\pm~$0.00  & 0.00$~\pm~$0.00 \\
		\bottomrule
	\end{tabular}
	\label{tab:maze}
\end{table*}

\vspace{-0.1in}
\subsection{Ablation Studies}

We further perform ablation studies on dataset qualities. As shown in Table \ref{tab:51}, IDAQ with the \textbf{Return-based} quantification achieves state-of-the-art performance on medium-quality datasets. The other two quantifications perform poorly, which may suggest that medium datasets are more challenging to design a good uncertainty quantification and the return-based metric can perform effectively in these settings. Further ablation studies on hyper-parameter settings are deferred to Appendix \ref{abl}. Results demonstrate that IDAQ is generally robust to the choice of hyper-parameters.



\vspace{-0.1in}
\section{Related Work}

In the literature, offline meta-RL methods utilize a context-based \citep{rakelly2019efficient} or gradient-based \citep{finn2017model} meta-RL framework to solve new tasks with few-shot adaptation. They utilize the techniques of contrastive learning \citep{li2020focal,yuan2022robust,li2020multi}, more expressive power \citep{mitchell2021offline}, or reward relabeling \citep{dorfman2021offline,pong2022offline} with various popular offline single-task RL tricks, i.e., using KL divergence \citep{wu2019behavior,peng2019advantage,nair2020awac} or explicitly constraining the policy to be close to the dataset \citep{fujimoto2019off,zhou2020plas}. However, these methods always require extra information for fast adaptation, such as offline context for testing tasks \citep{li2020focal,mitchell2021offline,yuan2022robust}, oracle reward functions \citep{dorfman2021offline}, or available interactions without reward supervision \citep{pong2022offline}. To address the challenge, we propose IDAQ, a context-based online adaptation algorithm, to utilize an uncertainty quantification for in-distribution adaptation without requiring additional information.

Similar to single-task offline RL \citep{levine2020offline}, SMAC \cite{pong2022offline} finds the policy or state-action distribution shift between learning policies and datasets in offline meta-RL. In this paper, we characterize the transition-reward distribution shift between offline datasets and online adaptation (see Eq. (\ref{eq:DSRT})), which is fundamentally different from state-action distribution shift. The reward-transition distribution shift is induced by task-dependent data collection and is unique in offline meta-RL. It specifies the discrepancy of reward and transition distribution given the state-action pairs. When using a behavior meta-policy to collect an offline dataset, reward-transition distribution shift will not appear but state-action or policy distribution shift still exists. SMAC claims that the distribution shift in z-space occurs due to a ''policy'' mismatch between behavior policies and online adaptation policy. In contrast, the reward-transition distribution shift is a general challenge in the setting of offline meta-RL. This distribution shift challenge occurs in any offline meta-RL algorithms, including gradient-based algorithms \citep{finn2017model}, and is beyond the z distribution shift tailored for context-based algorithms \citep{rakelly2019efficient}. 

BOReL \cite{dorfman2021offline} focuses on \textit{MDP ambiguity} for task inference. MDP ambiguity and transition-reward distribution shift are two orthogonal challenges in offline meta-RL with task-dependent behavior policies. MDP ambiguity arises from offline datasets with task-dependent data collection, where it may be difficult to differentiate between different MDPs due to narrow sub-datasets of various tasks. On the other hand, the reward-transition distribution shift studies the discrepancy of reward and transition distributions between task-dependent offline dataset and online adaptation. Our work leverages off-the-shelf context-based offline meta-training algorithms, e.g., FOCAL \citep{li2020focal}, for solving the MDP ambiguity problem during offline training, and proposes IDAQ to tackle the reward-transition distribution shift during online adaptation.

PEARL-based online adaptation \citep{rakelly2019efficient} may generate out-of-distribution episodes (see Figure \ref{fig:motivation} and Section \ref{sec:theory}). Meta-policies with Thompson sampling can generate in-distribution episodes, but the episodes generated by meta-policies with Thompson sampling are not all in-distribution. This is the motivation that we propose our method, which filters out out-of-distribution episodes to support in-distribution online adaptation. Moreover, IDAQ utilizes an exploration method (i.e., Thompson sampling \citep{strens2000bayesian}) for in-distribution online adaptation, which is supported by Theorem \ref{thm:ThompsonGenID}. Thompson sampling is a popular approach for temporally-extended exploration in the literature of meta-RL \citep{rakelly2019efficient}.

\vspace{-0.1in}
\section{Conclusion}

This paper formalizes the transition-reward distribution shift in offline meta-RL and introduces IDAQ, a novel in-distribution online adaptation approach. We find that IDAQ with a return-based uncertainty quantification performs effectively in medium or expert datasets. Experiments show that IDAQ can conduct accurate task inference and achieve state-of-the-art performance on Meta-World ML1 benchmark with 50 tasks. IDAQ also performs better or comparably than offline adaptation baselines with expert context, suggesting that offline context may not be necessary for the testing environments. One limitation of the greedy quantification is that it may not utilize in-distribution episodes with lower returns for random datasets and requires more adaptation episodes to sample in-distribution ``task hypotheses''. Two interesting future directions are to design a more accurate uncertainty quantification and to extend IDAQ to gradient-based in-distribution online adaptation algorithms.

\vspace{-0.1in}
\section*{Acknowledgements}
The authors would like to thank the anonymous reviewers and Zhizhou Ren for valuable and insightful discussions and helpful suggestions. This work is supported in part by Science and Technology Innovation 2030 - “New Generation Artificial Intelligence” Major Project (No. 2018AAA0100904) and the National Natural Science Foundation of China (62176135).

\nocite{langley00}

\bibliography{example_paper}
\bibliographystyle{icml2023}

\newpage

\appendix
\onecolumn
\allowdisplaybreaks

\section{Theory} \label{appendix:theory}

Our theory is the first to formalize the offline meta-RL with online adaptation using task-dependent behavior policies. We adopt the perspective of Bayesian RL to formalize task distribution, i.e., Bayes-Adaptive MDP (BAMDP) \citep{zintgraf2019varibad}, which is a popular theoretical framework for meta-RL. In this paper, we incorporate offline datasets with task-dependent behavior policies into BAMDPs and present a unique challenge: reward-transition distributional shift, which differs from state-action distributional shift in SMAC and single-task offline RL \citep{levine2020offline}. The consistency between offline and online policy evaluation is a very important criterion to measure the efficiency of algorithms in offline RL \citep{levine2020offline}. We find that filtering out out-of-distribution episodes in online adaptation can ensure the consistency of offline and online policy evaluation. Moreover, some insights are general for meta-RL. For example, Lemma \ref{lem:meta-task-dis-bound} shows that, for a meta-testing task drawn from \textit{arbitrary} task distribution, the distance from the closest meta-training task will asymptotically approach zero with high probability, as the number of sampled meta-training tasks grows. 

\subsection{Background}\label{appendix:theory-bg}
Throughout this paper, for a given non-negative integer $N\in\sZ_{+}$, we use $[N]$ to denote the set $\{0,1,\dots,N-1\}$. For any object that is a function
of/distribution over $\gS$,  $\gS\times\gA$, $\gS\times\gA\times\gS$, or $\gS\times\gA\times\gR$, we will treat it as a vector whenever convenient. 

\subsubsection{Finite-Horizon Single-Task RL}\label{appendix:epi-rl}
In single-task RL, an agent interacts with a Markov Decision Process (MDP) to maximize its cumulative reward \citep{sutton2018reinforcement}. A finite-horizon MDP is defined as a tuple $M=(\gS,\gA,\gR,H,P,R)$ \citep{zintgraf2019varibad,du2019good}, where $\gS$ is the state space, $\gA$ is the action space, $\gR$ is the reward space, $H\in\sZ_+$ is the planning horizon, $P:\gS\times\gA\rightarrow\Delta\left(\gS\right)$ is the transition function which takes a state-action pair and returns a distribution
over states, and $R:\gS\times\gA\rightarrow\Delta\left(\gR\right)$ is the reward distribution. In particular, we consider finite state, action, and reward spaces in the theoretical analysis, i.e., $|\gS|<\infty, |\gA|<\infty, |\gR|<\infty$. Without loss of generality, we assume a fixed initial state $s_0$\footnote{Some papers assume the initial state is sampled from a distribution $P_1$. Note this is equivalent to assuming a fixed initial state $s_0$, by setting $P(s_0, a) = P_1$ for all $a\in\gA$ and now our state $s_1$ is equivalent to the initial state in their assumption.}. A policy $\pi:\gS\rightarrow\Delta\left(\gA\right)$ prescribes a distribution over actions for each state. The policy $\pi$ induces a (random) $H$-horizon trajectory $\tau_H^\pi=\left(s_0, a_0, r_0, s_1, a_1, \dots, s_{H-1}, a_{H-1},  r_{H-1}\right)$, where $a_0\sim \pi(s_0), r_0\sim R(s_0,a_0), s_1\sim P(s_0,a_0), a_1\sim\pi(s_1)$, etc. To streamline our analysis, for each $h\in[H]$, we use $\gS_h \subseteq \gS$ to denote the set of states at $h$-th timestep, and we assume $\gS_h$ do not intersect with each other. To simplify notation, we assume the transition from any state in $S_{H-1}$ and any action to the initial state $s_0$, i.e., $\forall s\in S_{H-1},a\in\gA$, we have $P(s_0|s,a) = 1$\footnote{The transition from the state in $S_{H-1}$ does not affect learning in the finite-horizon MDP $M$.}. We also assume $r_t\in[0,1], \forall t\in[H]$ almost surely.  Denote the probability of $\tau_H$:
\begin{align}\label{appendix-eq:prob-tau}
p\left(\tau_H^\pi\right)=\left(\prod_{t\in[H]}\pi(a_t|s_t)\cdot R(r_t|s_t,a_t)\right)\prod_{t\in[H-1]}P(s_{t+1}|s_t,a_t).
\end{align}
For any policy $\pi$, we define a value function $V_\pi:\gS\rightarrow\sR$ as: $\forall h\in[H], \forall s\in\gS_h$,
\begin{align}\label{appendix-eq:value-function}
V_\pi(s)&=\E_{s_h=s,a_t\sim\pi(\cdot|s_t),r_t\sim R(\cdot{|s_t,a_t}),s_{t+1}\sim P\left(\cdot{|s_t,a_t}\right)}\left[\sum_{t=h}^{H-1}r_t\right] \\
&=\left\{
\begin{aligned}
&\sum_{a\in\gA}\pi(a|s)\cdot \E_{r\sim R(\cdot{|s,a})}\left[r\right], &\text{ if }h=H-1, \\
&\sum_{a\in\gA}\pi(a|s)\left(\E_{r\sim R(\cdot{|s,a})}\left[r\right]+\sum_{s'\in\gS_{h+1}}P(s'|s,a)V_\pi(s')\right), &\text{ otherwise},
\end{aligned}\nonumber
\right.
\end{align} 
and a visitation distribution of $\pi$ is defined by $\rho_{\pi}(\cdot): \Delta\left(\gS\right)$ which is $\forall h\in[H],\forall s\in\gS_h$,
\begin{align}\label{appendix-eq:visit-dis} 
\rho_{\pi}(s) =\left\{
\begin{aligned}
&\frac{1}{H}, &\text{ if }h=0\text{ and }s=s_0, \\
&\sum_{\tilde{s}\in\gS_{h-1},\tilde{a}\in\gA}\rho_{\pi}(\tilde{s})\cdot\pi(\tilde{a}|\tilde{s})\cdot P(s|\tilde{s},\tilde{a}), &\text{ if }h>0, \\
&0, &\text{ otherwise},
\end{aligned}
\right.
\end{align}
and $\forall s\in\gS,a\in\gA,r\in\gR$, 
\begin{align}\label{appendix-eq:visit-dis-r}
	\rho_{\pi}(s,a)=\rho_{\pi}(s)\cdot\pi(a|s)\quad\text{and}\quad \rho_{\pi}(s,a,r)=\rho_{\pi}(s)\cdot\pi(a|s)\cdot R(r|s,a).
\end{align}
The expected total reward induced by policy $\pi$, i.e., the policy evaluation of $\pi$, is defined by 
\begin{align}\label{eq:pi-eva}
\gJ_M(\pi)=V_\pi\left(s_0\right)=H\sum_{s\in\gS,a\in\gA} \rho_{\pi}(s,a)\cdot \E_{r\sim R(\cdot{|s,a})}\left[r\right].
\end{align}
The goal of RL is to find a policy $\pi$ that maximizes its expected return $\gJ(\pi)$.

\subsubsection{Offline Finite-Horizon Single-Task RL} \label{appendix:off-rl}
We consider the offline finite-horizon single-task RL setting, that is, a learner only has access to a dataset $\gD$ consisting of $K$ trajectories $\left\{\left(s_t^k,a_t^k,r_t^k\right)\right\}^{k\in[K]}_{t\in[H]}$ (i.e., $|\gD|=KH$ tuples) and is not allowed to interact with the environment for additional online explorations. The data can be collected through multi-source logging policies and denote the unknown behavior policy $\mu$. Similar with related work \citep{ren2021nearly,yin2020near,yin2021towards,yin2021near,shi2022pessimistic}, we assume that $\gD$ is collected through interacting $K$ i.i.d. episodes using policy $\mu$ in $M$. Define the reward and transition distribution of data collection with $\mu$ in $M$ by $\sP_{M}$ \citep{jin2021pessimism}, i.e., $\forall t\in[H]$ in each episode, 
\begin{align}\label{appedix-eq:mdp-DC}
\sP_M\left(r_t,s_{t+1}\left|s_t,a_t\right.\right)=R^M\left(r_t\left|s_t,a_t\right.\right)\cdot P^M\left(s_{t+1}\left|s_t,a_t\right.\right),
\end{align}
where the action $a_t$ is drawn from a behavior policy $\mu$. Denote a dataset collected following the i.i.d. data collecting process, i.e., $\gD\sim\left(\sP_M,\mu\right)$ is an i.i.d. dataset. Note that the offline dataset $\gD$ can be narrowly collected by some behavior policy $\mu$ and a large amount of state-action pairs are not contained in $\gD$. These unseen state-action pairs will be erroneously estimated to have unrealistic values, called a phenomenon \textit{extrapolation error} \citep{fujimoto2019off}. To overcome extrapolation error in policy learning of finite MDPs, Fujimoto et al. \citep{fujimoto2019off} introduces batch-constrained RL, which restricts the action space in order to force policy selection of an agent with respect to a subset of the given data. Thus, define a batch-constrained policy set is
\begin{align}\label{appedix-eq:constrained-pi}
\Pi^\gD=\left\{\pi\left|\pi(a|s)=0 \text{ whenever } (s,a)\not\in\gD\right.\right\},
\end{align}
where denoting $(s,a)\in\gD$ if there exists a trajectory containing $(s,a)$ in the dataset $\gD$, and similarly for $s\in\gD$, $(s,a,r)\in\gD$, or $(s,a,r,s')\in\gD$. The batch-constrained policy set $\Pi^\gD$ consists of the policies that for any state $s$ observed in the dataset $\gD$, the agent will not select an action outside of the dataset. Thus, for any batch-constrained policy $\pi\in\Pi^\gD$, define the approximate value function $V_\pi^{\gD}:\gS\rightarrow\sR$ estimated from $\gD$ \citep{fujimoto2019off,liu2020provably} as: $\forall h\in[H], \forall s\in\gS_h$,
\begin{align}
V_\pi^{\gD}(s)&=\E_{s_h=s, a_t\sim\pi(\cdot|s_t),(s_t,a_t,r_t,s_{t+1})\sim\gD}\left[\sum_{t=h}^{H-1}r_t\right] \\
&=\left\{
\begin{aligned}
&\sum_{a\in\gA}\pi(a|s)\E_{(s,a,r)\in \gD}\left[r\right], &\text{ if }h=H-1, \\
&\sum_{a\in\gA}\pi(a|s)\mathbb{E}_{(s,a,r,s')\in \gD}\left[r+V_\pi^{\gD}(s')\right], &\text{ otherwise},
\end{aligned}
\right.
\end{align} 
which is called Approximate Dynamic Programming (ADP) \citep{bertsekas1995neuro} and such methods take sampling data as input and approximate the value-function \citep{liu2020provably,chen2019information}. In addition, define the approximate policy evaluation of $\pi$ estimated from $\gD$ as 
\begin{align}\label{eq:pi-eva-dataset}
\gJ_{\gD}(\pi)=V_\pi^{\gD}(s_0).
\end{align}
The offline RL literature \citep{fujimoto2019off,liu2020provably,chen2019information,kumar2019stabilizing,kumar2020conservative} aims to utilize approximate expected total reward $\gJ_{\gD}(\pi)$ with various conservatism regularizations (i.e., policy constraints, policy penalty, uncertainty penalty, etc.) \citep{levine2020offline} to find a good policy within a batch-constrained policy set $\Pi^\gD$.

Similar to offline finite-horizon single-task RL theory \citep{ren2021nearly,yin2020near,yin2021towards,yin2021near,shi2022pessimistic}, define  
\begin{align}\label{appendix-eq:dm}
d_\mu^M=\min\left\{\rho_{\mu}(s,a)\left|\rho_{\mu}(s,a)>0,\forall s\in\gS, a\in\gA\right.\right\},
\end{align}
which is the minimal visitation state-action distribution induced by the behavior policy $\mu$ in $M$ and is an intrinsic quantity required by theoretical offline learning \citep{yin2020near}. Note that, different from recent offline episodic RL theory \citep{ren2021nearly,yin2020near,yin2021towards,yin2021near,shi2022pessimistic}, we do not assume any weak or uniform coverage assumption in the dataset because we focus on the policy evaluation of all batch-constrained policies in $\Pi^\gD$ rather than the optimal policy in the MDP $M$.

\subsubsection{Standard meta-RL} \label{appendix:meta-rl}
The goal of meta-RL \citep{finn2017model,rakelly2019efficient} is to train a meta-policy that can quickly adapt to new tasks using $N$ adaptation episodes. The standard meta-RL setting deals with a distribution $p(\kappa)$ over MDPs, in which each task $\kappa_i$ sampled from $p(\kappa)$ presents a finite-horizon MDP \citep{zintgraf2019varibad,du2019good}. $\kappa_i$ is defined by a tuple $\left(\gS, \gA, \gR, H, P^{\kappa_i}, R^{\kappa_i}\right)$, including state space $\gS$, action space $\gA$, reward space $\gR$, planning horizon $H$, transition function $P^{\kappa_i}(s'|s,a)$, and reward function $R^{\kappa_i}(r|s,a)$. Denote $\gK$ is the space of task $\kappa_i$. In this paper, we assume dynamics function $P$ and reward function $R$ may vary across tasks and share a common structure.  The meta-RL algorithms repeatedly sample batches of tasks to train a meta-policy. In the meta-testing, agents aim to rapidly adapt a good policy for new tasks drawn from $p(\kappa)$.

\paragraph{POMDPs.} We can formalize the meta-RL with few-shot adaptation as a specific finite-horizon Partially Observable Markov Decision Process (POMDP), which is defined by a tuple $\widehat{M} = \left(\widehat{\gS},\gA,\gR,\mathcal{\varOmega},\widehat{H},\widehat{P}, \widehat{P}_0, O, \widehat{R}\right)$, where $\widehat{\gS}=\gS\times\gK$ is the state space, $\gA$ and $\gR$ are the same action and reward spaces as the finite-horizon MDP $M$ defined in Appendix \ref{appendix:epi-rl}, respectively, $\mathcal{\varOmega}=\gS$ is the observation space, $\widehat{H}=N\times H$ is the planning horizon which represents $N$ adaptation episodes for a single meta-RL MDP $\kappa_i$, as discussed in Zintgraf et al. \citep{zintgraf2019varibad}, $\widehat{P}:\widehat{\gS}\times\gA\rightarrow\Delta\left(\widehat{\gS}\right)$ is the transition function: $\forall \hat{s},\hat{s}'\in\widehat{\gS}, a\in\gA$, where denoting $\hat{s}=(s,\kappa_i)$ and $\hat{s}'=(s',\kappa_j)$,
\begin{align}
\widehat{P}\left(\hat{s}'|\hat{s},a\right)&=\left\{
\begin{aligned}
&P^{\kappa_i}(s'|s,a), &\text{ if }\kappa_i=\kappa_j, \\
&0, &\text{ otherwise},
\end{aligned}
\right.
\end{align}
$\widehat{P}_0:\Delta\left(\widehat{\gS}\right)$ is the initial state distribution: $\forall \hat{s}=(s,\kappa_i)\in\widehat{\gS}$,
\begin{align}
\widehat{P}_0\left(\hat{s}\right)&=\left\{
\begin{aligned}
&p(\kappa_i), &\text{ if }s=s_0, \\
&0, &\text{ otherwise},
\end{aligned}
\right.
\end{align}
$O:\widehat{\gS}\rightarrow\Delta\left(\mathcal{\varOmega}\right)$ is the observation probability distribution conditioned on a state: $\forall \hat{s}=(s,\kappa_i)\in\widehat{\gS},o\in\mathcal{\varOmega}$,
\begin{align}
O\left(o|\hat{s}\right)=\left\{
\begin{aligned}
&1, &\text{ if }o=s, \\
&0, &\text{ otherwise},
\end{aligned}
\right.
\end{align}
and $\widehat{R}:\widehat{\gS}\times\gA\rightarrow\Delta\left(\gR\right)$ is the reward distribution: $\forall \hat{s}=(s,\kappa_i)\in\widehat{\gS}, a\in\gA,r\in\gR$,
\begin{align}
\widehat{R}\left(r|\hat{s},a\right)=R^{\kappa_i}(r|s,a).
\end{align}
Denote context $c_t=\left(a_t,r_t,s_{t+1}\right)$ as an experience collected at timestep $t$, and $c_{:t}=\left\langle s_0,c_0,\dots,c_{t-1}\right\rangle$\footnote{For clarity, we denote $c_{:0}^{\kappa_i}=s_0$.}$\in\gC_t\equiv\mathcal{\varOmega}\times\left(\gA\times\gR\times\mathcal{\varOmega}\right)^t$ indicates all experiences collected during $t$ timesteps. Note that $t$ may be larger than $H$, and when it is the case, $c_{:t}$ represents experiences collected across episodes in the single meta-RL MDP $\kappa_i$. Denote the entire context space $\gC=\bigcup_{t=0}^{\widehat{H}-1}\gC_t$ and a meta-policy $\hat{\pi}:\gC\rightarrow\Delta\left(\gA\right)$ \citep{wang2016learning,duan2016rl} prescribes a distribution over actions for each context. The goal of meta-RL is to find a meta-policy on history contexts $\hat{\pi}$ that maximizes the expected return within $N$ adaptation episodes: 
\begin{align}
\gJ_{\widehat{M}}(\hat{\pi})&=\E_{\hat{s}_0\sim P_0,o_t\sim O(\cdot{|s_t}),a_t\sim\hat{\pi}(\cdot|c_{:t}),r_t\sim \widehat{R}(\cdot{|s_t,a_t}),\hat{s}_{t+1}\sim \widehat{P}\left(\cdot{|\hat{s}_t,a_t}\right)}\left[\sum_{t=0}^{\widehat{H}-1}r_t\right] \\
&=\E_{\kappa_i\sim p(\kappa)}\left[\sum_{j=0}^{N-1}\E_{a_t\sim\hat{\pi}\left(\cdot\left|c_{:(jH+t)}\right.\right),r_t\sim R^{\kappa_i}(\cdot{|s_t,a_t}),s_{t+1}\sim P^{\kappa_i}\left(\cdot{|s_t,a_t}\right)}\left[\sum_{t=0}^{H-1}r_t\right]\right].
\end{align}

\paragraph{BAMDPs.} A Markovian belief state allows a POMDP to be formulated as a Markov decision process where every belief is a state \citep{cassandra1994acting}. We can transform the finite-horizon POMDP $\widehat{M}$ to a finite-horizon belief MDP, which is called Bayes-Adaptive MDP (BAMDP) in the literature \citep{zintgraf2019varibad,ghavamzadeh2015bayesian,dorfman2021offline} and is defined by a tuple $M^{+}=\left(\gS^{+},\gA,\gR,H^+,P^{+},P_0^{+},R^{+}\right)$, $\gS^{+}=\gS\times\gB$ is the hyper-state space, where $\gB =\left\{p(\kappa|c)\left|c\in\gC\right.\right\}$ is the set of task beliefs over the meta-RL MDPs, the prior 
\begin{align}
b_0^{\kappa}=p\left(\kappa|c_{:0}\right)=p(\kappa)
\end{align}
is the meta-RL MDP distribution, and $\forall t\in\left[\widehat{H}-1\right]$, $\forall c_{:(t+1)}\in\gC$, denoting $b_t^{\kappa}=p\left(\kappa|c_{:t}\right)$ and 
\begin{align}
b_{t+1}^{\kappa}&=p\left(\kappa|c_{:(t+1)}\right)=p\left(p\left(\kappa|c_{:t}\right)|c_{:(t+1)}\right)=p\left(p\left(\kappa|c_{:t}\right)|s_t,c_t\right)= p\left(b_t^{\kappa}|s_t,c_t\right) \\
&\propto p\left(b_t^{\kappa},c_t|s_t\right)=p\left(c_t|s_t,b_t^{\kappa}\right)p\left(b_t^{\kappa}|s_t\right)=p\left(c_t|s_t,b_t^{\kappa}\right)b_t^{\kappa} \\
&=\E_{\kappa_i\sim b_t^{\kappa}}\left[R^{\kappa_i}(r_t|s_t,a_t)\cdot P^{\kappa_i}(s_{t+1}|s_t,a_t)\right]\cdot b_t^{\kappa}
\end{align}
is the posterior over the MDPs given the context $c_{:(t+1)}$, $\gA,\gR$ are the same action space and reward space as the finite-horizon POMDP $\widehat{M}$, respectively, $H^+=N\times H$ is the planning horizon across adaptation episodes, $P^{+}:\gS^{+}\times\gA\times\gR\rightarrow\Delta\left(\gS^{+}\right)$ is the transition function: $\forall s_t^+,s_{t+1}^+\in\gS^{+}, a_t\in\gA, r_t\in\gR$, where denoting $s_t^+=\left(s_t,b_t^{\kappa}\right)$ and $s_{t+1}^+=\left(s_{t+1},\tilde{b}_{t+1}^{\kappa}\right)$,
\begin{align}
P^{+}\left(s_{t+1}^+{}\left|s_t^+,a_t,r_t\right.\right)&= P^{+}\left(s_{t+1},\tilde{b}_{t+1}^{\kappa}\left|s_t,b_t^{\kappa},a_t,r_t\right.\right)\\
&=P^{+}\left(s_{t+1}\left|s_t,b_t^{\kappa},a_t\right.\right)P^{+}\left(\tilde{b}_{t+1}^{\kappa}\left|s_t,b_t^{\kappa},c_t\right.\right) \\
&=\E_{\kappa_i\sim b_t^{\kappa}}\left[P^{\kappa_i}(s_{t+1}|s_t,a_t)\right]\cdot\mathbbm{1}\left[\tilde{b}_{t+1}^{\kappa}=p(b_t^{\kappa}|s_t,c_t)\right],
\end{align}
$P_0^+:\Delta\left(\gS^+\right)$ is the initial hyper-state distribution, i.e., a deterministic initial hyper-state is
\begin{align}
s_0^+=(s_0,b_0^{\kappa})=(s_0,p(\kappa))\in\gS^+,
\end{align}
and $R^+:\gS^+\times\gA\rightarrow\Delta\left(\gR\right)$ is the reward distribution: $\forall s^+=(s,b^{\kappa})\in\gS^+, a\in\gA,r\in\gR$,
\begin{align}
R^+\left(r|s^+,a\right)=R^+\left(r|s,b^{\kappa},a\right)=\E_{\kappa_i\sim b^{\kappa}}\left[R^{\kappa_i}(r|s,a)\right].
\end{align}
In a BAMDP, the belief is over the transition and reward functions, which are constant for a given task. A meta-policy on BAMDP $\pi^+:\gS^+\rightarrow\Delta\left(\gA\right)$ prescribes a distribution over actions for each hyper-state. The agent’s objective is now to find a meta-policy on hyper-states $\pi^+$ that maximizes the expected return in the BAMDP,
\begin{align}
\gJ_{M^+}\left(\pi^+\right)&=\E_{a_t\sim\pi^+\left(\cdot\left|s_t^+\right.\right),r_t\sim R^+\left(\cdot\left|s_t^+,a_t\right.\right),s_{t+1}^+\sim P^+\left(\cdot\left|s_t^+,a_t\right.\right)}\left[\sum_{t=0}^{\widehat{H}-1}r_t\right] \\
&=\E_{\kappa_i\sim p(\kappa)}\left[\sum_{j=0}^{N-1}\E_{a_t\sim\pi^+\left(\cdot\left|s_{jH+t}^+\right.\right),r_t\sim R^{\kappa_i}(\cdot{|s_t,a_t}),s_{t+1}\sim P^{\kappa_i}\left(\cdot{|s_t,a_t}\right)}\left[\sum_{t=0}^{H-1}r_t\right]\right].
\end{align}
For any meta-policy on hyper-states $\pi^+$, denote the corresponding meta-policy on history contexts $\hat{f}_{\pi^+}:\gC\rightarrow\Delta\left(\gA\right)$, i.e., $\forall t\in\left[\widehat{H}-1\right],\forall c_{:t}\in\gC_t$, s.t., $\hat{f}_{\pi^+}\left(\cdot|c_{:t}\right)=\pi^+\left(\cdot|s_t^+\right)$, where $s_t^+=(s_t,b_t^{\kappa})=(s_t, p(\kappa|c_{:t}))$, and we have 
\begin{align}
\gJ_{\widehat{M}}\left(\hat{f}_{\pi^+}\right)&=\E_{\kappa_i\sim p(\kappa)}\left[\sum_{j=0}^{N-1}\E_{a_t\sim\hat{f}_{\pi^+}\left(\cdot\left|c_{:(jH+t)}\right.\right),r_t\sim R^{\kappa_i}(\cdot{|s_t,a_t}),s_{t+1}\sim P^{\kappa_i}\left(\cdot{|s_t,a_t}\right)}\left[\sum_{t=0}^{H-1}r_t\right]\right] \\
&=\gJ_{M^+}\left(\pi^+\right).
\end{align}
The belief MDP is such that an optimal policy for it, coupled with the correct state estimator, will give rise to optimal behavior for the original POMDP \citep{astrom1965optimal,smallwood1973optimal,kaelbling1998planning}, which indicates that 
\begin{align}
\gJ_{M^+}\left(\pi^{+,*}\right)=\gJ_{\widehat{M}}\left(\hat{f}_{\pi^{+,*}}\right)=\gJ_{\widehat{M}}\left(\hat{\pi}^*\right),
\end{align}
where $\pi^{+,*}$ and $\hat{\pi}^*$ are the optimal policies for BAMDP $M^+$ and POMDP $\widehat{M}$, respectively. Thus, the agent can find a policy $\pi^+$ to maximize the expected return in the BAMDP $M^+$ to address the POMDP $\widehat{M}$ by the transformed policy $\hat{f}_{\pi^+}$.

\subsubsection{Offline meta-RL} \label{appendix:off-meta-rl}
In the offline meta-RL setting, a meta-learner only has access to an offline multi-task dataset $\gD^+$ and is not allowed to interact with the environment during meta-training \citep{li2020focal}. Recent offline meta-RL methods \citep{dorfman2021offline} always utilize task-dependent behavior policies $p(\mu|\kappa)$, which represents the random variable of the behavior policy $\mu(a|s)$ conditioned on the random variable of the task $\kappa$. For brevity, we overload $[\mu]=p(\mu|\kappa)$. Similar to related work on offline RL \citep{shi2022pessimistic}, we assume that $\gD^+$ is collected through interacting multiple i.i.d. trajectories using task-dependent policies $[\mu]$ in $M^+$. Define the reward and transition distribution of the task-dependent data collection by $\sP_{M^+,[\mu]}$ \citep{jin2021pessimism}, i.e., for each step $t$ in a trajectory,
\begin{align}\label{appendix-eq:TaskDependentRTDD}
\sP_{M^+,[\mu]}\left(r_t,s_{t+1}\left|s^+_t,a_t\right.\right)\propto\E_{\kappa_i\sim p(\kappa),\mu_i\sim p\left(\mu|\kappa_i\right)}\left[\sP_{\kappa_i}\left(r_t,s_{t+1}\left|s_t,a_t\right.\right)\cdot p_{M^+}\left(s^+_t\left|\kappa_i,\mu_i\right.\right)\right],
\end{align}
where $\sP_{\kappa_i}$ is the reward and transition distribution of $\kappa_i$ defined in Eq. (\ref{appedix-eq:mdp-DC}), and $p_{M^+}\left(s^+_t\left|\kappa_i,\mu_i\right.\right)$ denotes the probability of $s^+_t$ when executing $\mu_i$ in a task $\kappa_i$, i.e., 
\begin{align}\label{appendix-eq:prob-s-with-mu}
	p_{M^+}\left(s^+_t\left|\kappa_i,\mu_i\right.\right)=\sum_{c_{:t}\in\gC_t} p_{\kappa_i}^{\mu_i}(c_{:t})\cdot \mathbbm{1}\left[b_t^{\kappa}=p(\kappa|c_{:t})\right],
\end{align}
where the state in $c_{:t}$ is $s_t$ and $p_{\kappa_i}^{\mu_i}(c_{:t})$ is defined in Eq. (\ref{appendix-eq:prob-tau}). Similar to offline single-task RL (see Appendix \ref{appendix:off-rl}), offline dataset $\gD^+$ can be narrow and a large amount of state-action pairs are not contained. These unseen state-action pairs will be erroneously estimated to have unrealistic values, called a phenomenon \textit{extrapolation error} \citep{fujimoto2019off}. To overcome extrapolation error in offline RL, related works \citep{fujimoto2019off} introduce batch-constrained RL, which restricts the action space in order to force policy selection of an agent with respect to a given dataset. Define a policy $\pi^+$ to be batch-constrained by $\gD^+$ if $\pi^+\left(a\left|s^+\right.\right)=0$ whenever a tuple $\left(s^+,a\right)$ is not contained in $\gD^+$. Offline RL \citep{liu2020provably,chen2019information} approximates policy evaluation for a batch-constrained policy $\pi^+$ by sampling from an offline dataset $\gD^+$, which is denoted by $\gJ_{\gD^+}\left(\pi^+\right)$ and called 
\textit{Approximate Dynamic Programming} \citep[ADP;][]{bertsekas1995neuro}. During meta-testing, RL agents perform online adaptation using a meta-policy $\pi^+$ in new tasks drawn from meta-RL task distribution. The reward and transition distribution of data collection with $\pi^+$ in $M^+$ during adaptation is defined by
\begin{align}
\sP_{M^+,\pi^+}\left(r_t,s_{t+1}\left|s^+_t,a_t\right.\right)=R^+\left(r_t\left|s_t^+,a_t\right.\right)\cdot P^+\left(s_{t+1}\left|s^+_t,a_t\right.\right),
\end{align}
where $P^+\left(s_{t+1}\left|s^+_t,a_t\right.\right)$ is the marginal transition functions of $M^+$, i.e., 
\begin{align}
	P^+\left(s_{t+1}\left|s^+_t,a_t\right.\right) = \E_{\kappa_i\sim b_t^{\kappa}}\left[P^{\kappa_i}\left(s_{t+1}\left|s_t,a_t\right.\right)\right].
\end{align}

\newpage
\subsection{Main Results in Section \ref{sec:DSRT}}\label{appendix:theory-main-results-part1}

\DSRT*

This definition utilizes the discrepancy between offline and online data collection to characterize the joint distribution gap of reward and transition. Note that in offline data collection $\sP_{M^+,[\mu]}$, the behavior policies $p(\mu|\kappa)$ can vary based on task identification, whereas the online data collection $\sP_{M^+,\pi^+}$ is the expected reward and transition distribution across the task distribution $p(\kappa)$. Formally, we use $p_{M^+}^{\pi^+}\left(s_t^+,a_t\right)>0$ to present that $\left(s_t^+, a_t\right)$ can be reached by $\pi^+$ in $M^+$, where 
\begin{align}
p_{M^+}^{\pi^+}\left(s_t^+,a_t\right)=p_{M^+}^{\pi^+}(s_t^+)\cdot\pi^+\left(a_t\left|s_t^+\right.\right),
\end{align}
and $p_{M^+}^{\pi^+}(s_t^+)$ is defined in Eq. (\ref{appendix-eq:prob-tau}). The data distribution induced by $\pi^+$ and $[\mu]$ mismatches when the reward and transition distribution of $\pi^+$ and $[\mu]$ differs in a tuple $\left(s_t^+,a_t\right)$, in which the agent can reach this tuple by executing $\pi^+$ in $M^+$, i.e., $p_{M^+}^{\pi^+}\left(s_t^+,a_t\right)>0$. Note that if $\pi^+$ can reach a tuple $\left(s_t^+,a_t\right)$, this tuple is guaranteed to be contained in the offline dataset, i.e., $p_{M^+}^{[\mu]}\left(s_t^+,a_t\right)>0$, because a batch-constrained policy $\pi^+$ will not select an action outside of the dataset collected by $[\mu]$, as introduced in Section \ref{sec:offline-meta-rl}.

\begin{figure}[H]
	\centering
	\input{figures/mdps/mdp-1}
	$p(\kappa_i) = \frac{1}{v}, p(\mu_i|\kappa_i)=1, \text{and }\mu_i(a_i|s_0)=1$
	\caption{A concrete example, which has $v$ meta-RL tasks, one state, $v$ actions, $v$ behavior policies, horizon $H=1$ in a episode, and $v$ adaptation episodes, where $v\ge 3$.}
	\label{appendix-fig:MultiTaskDDMExample}
\end{figure}

\DSRTExists*
\begin{proof}
	To serve a concrete example, we construct an offline meta-RL setting shown in Figure~\ref{appendix-fig:MultiTaskDDMExample}. In this example, there are $v$ meta-RL tasks $\gK=\left\{\kappa_1,\dots,\kappa_v\right\}$ and $v$ behavior policies $\left\{\mu_1,\dots,\mu_v\right\}$, where $v\ge 3$. Each task $\kappa_i$ has one state $\gS=\{s_0\}$, $2v$ actions $\gA=\{a_1,\dots,a_v\}$, and horizon in an episode $H=1$. For each task $\kappa_i$, RL agents can receive reward 1 performing action $a_i$. During adaptation, the RL agent can interact with the environment within $v$ episodes. The task distribution is uniform, the behavior policy of task $\kappa_i$ is $\mu_i$, and each behavior policy $\mu_i$ will perform $a_i$. When a batch-constrained meta-policy $\pi^+$ selects an action $\tilde{a}$ in the initial state $s_0^+$, we find that
	\begin{align}
	\sP_{M^+,[\mu]}\left(r=1\left|s_0^+,\tilde{a}\right.\right)=1\neq\sP_{M^+,\pi^+}\left(r=1\left|s_0^+,\tilde{a}\right.\right)=\frac{1}{v},
	\end{align}
	in which there is the probability of $\frac{1}{v}$ to sample a corresponding testing task, whose reward function of $\tilde{a}$ is 1, whereas the reward in the offline dataset collected by $[\mu]$ is all 1. 
\end{proof}

\subsection{Main Results in Section \ref{sec:dis-matters}}\label{appendix:theory-main-results-part2}

\subsubsection{Out-of-Distribution Analyses}
\DSRTOOD*
\begin{proof}
	\textbf{Part (i)}
	In the example shown in Figure \ref{appendix-fig:MultiTaskDDMExample}, an offline multi-task dataset $D^+$ is drawn from the task-dependent data collection $\sP_{M^+,[\mu]}$. Since the reward of $D^+$ is all 1, the task beliefs in $D^+$ have two types: (i) all task possible $s_0^+$ and (ii) determining task $i$ with receiving reward 1 in action $a_i$. For any batch-constrained meta-policy $\pi^+$ selecting an action $\tilde{a}_j$ on $s_0^+$ during meta-testing, there has probability $1-\frac{1}{v}$ to receive reward 0 and the task belief will become ``excluding task $j$'', which is not contained in $D^+$ with $v\ge 3$. For any $\delta \in(0,1]$, let $v>\frac{1}{\delta}$, with probability $1-\delta$, the agent will visit out-of-distribution hyper-states during adaptation. 
	
	\textbf{Part (ii)} In Figure \ref{appendix-fig:MultiTaskDDMExample}, an offline dataset $D^+$ only contains reward $1$, thus for each batch-constrained meta-policy $\pi^+$, the offline evaluation of $\pi^+$ in $D^+$ is $\gJ_{D^+}\left(\pi^+\right)=H^+=vH$. The optimal meta-policy $\pi^{+,*}$ in this example is to enumerate $a_1,\dots,a_v$ until the task identification is inferred from an action with a reward of 1. A meta-policy $\pi^{+,*}$ needs to explore in the testing environments and its online policy evaluation is
	\begin{align}
	\gJ_{M^+}\left(\pi^{+,*}\right)&=\sum_{k=0}^{N-1}\frac{N-k}{v-k}\prod_{j=0}^{k-1}\left(1-\frac{1}{v-j}\right) \\
	&=\sum_{k=0}^{N-1}\prod_{j=0}^{k-1}\frac{v-j-1}{v-j} \\
	&= \sum_{k=0}^{N-1}\frac{v-k}{v} = N-\frac{N(N-1)}{2v} \\
	&= \frac{v+1}{2} =\frac{H^++1}{2},
	\end{align}
	where $N=v$ is the number of adaptation episodes. Thus, the gap of policy evaluation of $\pi^+$ between offline meta-training and online adaptation is
	\begin{align}
	\left|\gJ_{M^+}\left(\pi^+\right)-\gJ_{D^+}\left(\pi^+\right)\right|\ge \gJ_{D^+}\left(\pi^+\right)-\gJ_{M^+}\left(\pi^{+,*}\right)=\frac{H^+-1}{2}.
	\end{align}
\end{proof}
Proposition \ref{prop:DSRTOOD} states that RL agents will go out of the distribution of the offline dataset $\gD^+$ due to shifts in the reward and transition distribution. Thus, the offline policy evaluation of $\pi^+$ in meta-training cannot provide a reference for the online meta-testing. 

\subsubsection{In-Distribution Analyses}

\begin{restatable}[Transformed BAMDPs]{definition}{TBAMDP}\label{def:TBAMDP}
	A transformed BAMDP is defined as a tuple $\overline{M}^{+}=\left(\overline{\gS}^{+},\gA,\gR,H^+,\overline{P}^{+},\overline{P}_0^{+},\overline{R}^{+}\right)$, where $\overline{\gS}^+=\gS\times\overline{\gB}$ is the hyper-state space, $\overline{\gB}$ is the space of overall beliefs over meta-RL MDPs with behavior policies, $\gA, \gR, H^+$ are the same action space, reward space, and planning horizon as the original BAMDP $M^{+}$, respectively, $\overline{P}_0^+$ is the initial hyper-state distribution presenting joint distribution of task and behavior policies $p(\kappa,\mu)=p(\kappa)p(\mu|\kappa)$, and $\overline{P}^+,\overline{R}^+$ are the transition and reward functions. The goal of meta-RL agents is to find a meta-policy $\bar{\pi}^+\left(a_t\left|\bar{s}^+_t\right.\right)$  that maximizes online policy evaluation $\gJ_{\overline{M}^+}\left(\bar{\pi}^+\right)$. Denote the reward and transition distribution of the task-dependent data collection in a transformed BAMDP $\overline{M}^+$ by $\sP_{\overline{M}^+,[\mu]}$, as defined in Eq. (\ref{appendix-eq:TaskDependentRTDD}). Denote the offline multi-task dataset collected by task-dependent data collection $\sP_{\overline{M}^+,[\mu]}$ by $\overline{\gD}^+$.
\end{restatable}

More specifically, a finite-horizon transformed BAMDP is defined by a tuple $\overline{M}^{+}=\left(\overline{\gS}^{+},\gA,\gR,H^+,\overline{P}^{+},\overline{P}_0^{+},\overline{R}^{+}\right)$, $\overline{\gS}^{+}=\gS\times\overline{\gB}$ is the hyper-state space, where $\overline{\gB} =\left\{p(\kappa,\mu|c)\left|c\in\gC\right.\right\}$ is the space of beliefs over meta-RL MDPs with behavior policies, the prior
\begin{align}\label{appendix-eq:transbelief}
b_0^{\kappa,\mu}=p\left(\kappa,\mu\left|c_{:0}\right.\right)=p(\kappa,\mu)
\end{align}
is the distribution of meta-RL MDPs with behavior policies, and $\forall t\in\left[\widehat{H}-1\right]$, $\forall c_{:(t+1)}\in\gC$, denoting $b_t^{\kappa,\mu}=p\left(\kappa,\mu|c_{:t}\right)$ and 
\begin{align}
b_{t+1}^{\kappa,\mu}&=p\left(\kappa,\mu\left|c_{:(t+1)}\right.\right)=p\left(p\left(\kappa,\mu|c_{:t}\right)\left|c_{:(t+1)}\right.\right)=p\left(p\left(\kappa,\mu|c_{:t}\right)|s_t,c_t\right)= p\left(\left.b_t^{\kappa,\mu}\right|s_t,c_t\right) \\
&\propto p\left(\left.b_t^{\kappa,\mu},c_t\right|s_t\right)=p\left(c_t\left|s_t,b_t^{\kappa,\mu}\right.\right)p\left(\left.b_t^{\kappa,\mu}\right|s_t\right)=p\left(c_t\left|s_t,b_t^{\kappa,\mu}\right.\right)b_t^{\kappa,\mu} \\
&=\E_{\left(\kappa_i,\mu_i\right)\sim b_t^{\kappa,\mu}}\left[\mu_i\left(a_t|s_t\right)\cdot R^{\kappa_i}(r_t|s_t,a_t)\cdot P^{\kappa_i}(s_{t+1}|s_t,a_t)\right]\cdot b_t^{\kappa,\mu}
\end{align}
is the posterior over the meta-RL MDPs with behavior policies given the context $c_{:(t+1)}$, $\gA$, $\gR$ and $\widehat{H}$ are the same action space, reward space, and planning horizon as the finite-horizon BAMDP $M^+$, respectively, $\overline{P}^{+}:\overline{\gS}^{+}\times\gA\times\gR\rightarrow\Delta\left(\overline{\gS}^{+}\right)$ is the transition function: $\forall \bar{s}_t^+,\bar{s}_{t+1}^+\in\overline{\gS}^{+}, a_t\in\gA, r_t\in\gR$, where denoting $\bar{s}_t^+=\left(s_t,b_t^{\kappa,\mu}\right)$ and $\bar{s}_{t+1}^+=\left(s_{t+1},\tilde{b}_{t+1}^{\kappa,\mu}\right)$,
\begin{align}
\overline{P}^{+}\left(\bar{s}_{t+1}^+\left|\bar{s}_t^+,a_t,r_t\right.\right)&= \overline{P}^{+}\left(\left.s_{t+1},\tilde{b}_{t+1}^{\kappa,\mu}\right|s_t,b_t^{\kappa,\mu},a_t,r_t\right)\\
&=\overline{P}^{+}\left(s_{t+1}\left|s_t,b_t^{\kappa,\mu},a_t\right.\right)P^{+}\left(\left.\tilde{b}_{t+1}^{\kappa,\mu}\right|s_t,b_t^{\kappa,\mu},c_t\right) \\
&=\E_{\left(\kappa_i,\mu_i\right)\sim b_t^{\kappa,\mu}}\left[P^{\kappa_i}(s_{t+1}|s_t,a_t)\right]\cdot\mathbbm{1}\left[\tilde{b}_{t+1}^{\kappa,\mu}=p(b_t^{\kappa,\mu}|s_t,c_t)\right],
\end{align}
$\overline{P}_0^+:\Delta\left(\overline{\gS}^{+}\right)$ is the initial hyper-state distribution, i.e., a deterministic initial hyper-state is
\begin{align}
\bar{s}_0^+=\left(s_0,b_0^{\kappa,\mu}\right)=(s_0,p(\kappa,\mu))\in\overline{\gS}^+,
\end{align}
and $\overline{R}^+:\overline{\gS}^+\times\gA\rightarrow\Delta\left(\gR\right)$ is the reward distribution: $\forall \bar{s}^+=\left(s,b^{\kappa,\mu}\right)\in\overline{\gS}^+, a\in\gA,r\in\gR$,
\begin{align}
\overline{R}^+\left(r|\bar{s}^+,a\right)=\overline{R}^+\left(r\left|s,b^{\kappa,\mu},a\right.\right)=\E_{\left(\kappa_i,\mu_i\right)\sim b^{\kappa,\mu}}\left[R^{\kappa_i}(r|s,a)\right].
\end{align}
In a transformed BAMDP $\overline{M}^{+}$, the overall belief is about the task-dependent behavior policies, transition function, and reward function, which are constant for a given task. A meta-policy on $\overline{M}^{+}$ is $\bar{\pi}^+:\overline{\gS}^+\rightarrow\Delta\left(\gA\right)$ prescribes a distribution over actions for each hyper-state. With feasible Bayesian belief updating, the objective of RL agents is now to find a meta-policy on hyper-states $\bar{\pi}^+$ that maximizes the expected return in the transformed BAMDP,
\begin{align}
\gJ_{\overline{M}^+}\left(\bar{\pi}^+\right)&=\E_{a_t\sim\bar{\pi}^+\left(\cdot\left|\bar{s}_t^+\right.\right),r_t\sim \overline{R}^+\left(\cdot\left|\bar{s}_t^+,a_t\right.\right),s_{t+1}^+\sim \overline{P}^+\left(\cdot\left|\bar{s}_t^+,a_t\right.\right)}\left[\sum_{t=0}^{\widehat{H}-1}r_t\right] \\
&=\E_{(\kappa_i,\mu_i)\sim p(\kappa,\mu)}\left[\sum_{j=0}^{N-1}\E_{a_t\sim\bar{\pi}^+\left(\cdot\left|\bar{s}_{jH+t}^+\right.\right),r_t\sim R^{\kappa_i}(\cdot{|s_t,a_t}),s_{t+1}\sim P^{\kappa_i}\left(\cdot{|s_t,a_t}\right)}\left[\sum_{t=0}^{H-1}r_t\right]\right] \\
&=\E_{\kappa_i\sim p(\kappa)}\left[\sum_{j=0}^{N-1}\E_{a_t\sim\bar{\pi}^+\left(\cdot\left|\bar{s}_{jH+t}^+\right.\right),r_t\sim R^{\kappa_i}(\cdot{|s_t,a_t}),s_{t+1}\sim P^{\kappa_i}\left(\cdot{|s_t,a_t}\right)}\left[\sum_{t=0}^{H-1}r_t\right]\right].
\end{align}
For any meta-policy on hyper-states $\bar{\pi}^+$, denote the corresponding meta-policy on history contexts $\hat{f}_{\bar{\pi}^+}:\gC\rightarrow\Delta\left(\gA\right)$, i.e., $\forall t\in\left[\widehat{H}-1\right],\forall c_{:t}\in\gC_t$, s.t., $\hat{f}_{\bar{\pi}^+}\left(\cdot|c_{:t}\right)=\bar{\pi}^+\left(\cdot\left|\bar{s}_t^+\right.\right)$, where $\bar{s}_t^+=\left(s_t,b_t^{\kappa,\mu}\right)=\left(s_t, p\left(\kappa,\mu|c_{:t}\right)\right)$, and we have 
\begin{align}
\gJ_{\widehat{M}}\left(\hat{f}_{\bar{\pi}^+}\right)&=\E_{\kappa_i\sim p(\kappa)}\left[\sum_{j=0}^{N-1}\E_{a_t\sim\hat{f}_{\bar{\pi}^+}\left(\cdot\left|c_{:(jH+t)}\right.\right),r_t\sim R^{\kappa_i}(\cdot{|s_t,a_t}),s_{t+1}\sim P^{\kappa_i}\left(\cdot{|s_t,a_t}\right)}\left[\sum_{t=0}^{H-1}r_t\right]\right] \\
&=\gJ_{\overline{M}^+}\left(\bar{\pi}^+\right).
\end{align}

\begin{restatable}{lemma}{OffLearnGua}\label{lem:off-learn-gua}
	In an MDP $M$, for each behavior policy $\mu$ and batch-constrained policy $\pi$, collect a dataset $\gD$ and the gap between approximate offline policy evaluation $\gJ_{\gD}(\pi)$ and accurate policy evaluation $\gJ_{M}(\pi)$ will asymptotically approach to 0, as the offline dataset $\gD$ grows.
\end{restatable}

From a given dataset $\gD$, an abstract MDP $M_{\gD}$ can be estimated \citep{fujimoto2019off,yin2021towards,BatchRLLemma}. According to concentration bounds, the estimated transition and reward function will asymptotically approach $M$ \citep{yin2021towards} during the support of $\gD$. Then, using the simulation lemma \citep{FiniteSimulationLemma,BatchRLLemma}, the gap between $\gJ_{\gD}(\pi)$ and $\gJ_{M}(\pi)$ will asymptotically approach to 0, as the offline dataset $\gD$ grows. Formal proofs are deferred in Appendix \ref{appendix:omitted-proof-Lemma-1}.

\DSRTID*
\begin{proof}
	\textbf{Part (i)} During online adaptation, RL agents construct a hyper-state $\bar{s}^+_t=\left(s_t, \bar{b}_t\right)$ from the context history and perform a meta-policy $\bar{\pi}^+\left(a_t\left|\bar{s}^+_t\right.\right)$. The new belief $\bar{b}_t$ accounts for the uncertainty of task MDPs and task-dependent behavior policies. In contrast with Proposition \ref{prop:DSRTOOD}(i), for feasible Bayesian belief updating, transformed BAMDPs do not allow the agent to visit out-of-distribution hyper-states. Otherwise, the context history will conflict with the belief about behavior policies, i.e., RL agents cannot update their beliefs $\bar{b}_t$ when they have observed an event that they believe to have probability zero. 

	\textbf{Part (ii)} We assume feasible Bayesian belief updating in this proof. At first, $\forall s_t^+, a_t$, s.t. $p_{\overline{M}^{+}}^{\bar{\pi}^+}\left(\bar{s}_t^+,a_t\right)>0$, we aim to prove
	\begin{align}
		\sP_{\overline{M}^{+},[\mu]}\left(r_t,s_{t+1}\left|\bar{s}_t^+,a_t\right.\right)=\sP_{\overline{M}^{+},\bar{\pi}^+}\left(r_t,s_{t+1}\left|\bar{s}_t^+,a_t\right.\right).
	\end{align}
	Since $\bar{\pi}^+$ is batch-constrained policy by $[\mu]$, if $p_{\overline{M}^{+}}^{\bar{\pi}^+}\left(\bar{s}_t^+,a_t\right)>0$, we have $p_{\overline{M}^{+}}^{[\mu]}\left(\bar{s}_t^+,a_t\right)>0$. Then,
	\begin{align}
		\sP_{\overline{M}^{+},\bar{\pi}^+}\left(r_t,s_{t+1}\left|\bar{s}_t^+,a_t\right.\right)&= \overline{R}^+\left(r_t\left|\bar{s}_t^+,a_t\right.\right)\cdot \overline{P}^+\left(s_{t+1}\left|\bar{s}^+_t,a_t\right.\right) \\
		&=\E_{\kappa_i\sim b_t^{\kappa,\mu}}\left[\sP_{\kappa_i}\left(r_t,s_{t+1}\left|s_t,a_t\right.\right)\right] \\
		&=\E_{(\kappa_i,\mu_i)\sim b_t^{\kappa,\mu}}\left[\sP_{\kappa_i}\left(r_t,s_{t+1}\left|s_t,a_t\right.\right)\right]
	\end{align}
	where $b_t^{\kappa,\mu}$ is defined in Eq. (\ref{appendix-eq:transbelief}) and 
	\begin{align}
		p\left(\kappa_i,\mu_i\left|b_t^{\kappa,\mu}\right.\right)&=\E_{\kappa_i\sim p(\kappa),\mu_i\sim p\left(\mu|\kappa_i\right)}\left[\E_{c_{:t}\sim p_{\overline{M}^+}\left(c_{:t}|\kappa_i,\mu_i\right)}\left[ \mathbbm{1}\left[b_t^{\kappa,\mu}=p(\kappa,\mu|c_{:t})\right]\right]\right],
	\end{align}
	where $p_{\overline{M}^+}\left(c_{:t}|\kappa_i,\mu_i\right)$ is defined in Eq. (\ref{appendix-eq:prob-tau}). According to Eq. (\ref{appendix-eq:TaskDependentRTDD}) and (\ref{appendix-eq:prob-s-with-mu}),
	\begin{align}
	&~\sP_{\overline{M}^+,[\mu]}\left(r_t,s_{t+1}\left|\bar{s}^+_t,a_t\right.\right) \\
	\propto&~\E_{\kappa_i\sim p(\kappa),\mu_i\sim p\left(\mu|\kappa_i\right)}\left[\sP_{\kappa_i}\left(r_t,s_{t+1}\left|s_t,a_t\right.\right)\cdot p_{\overline{M}^+}\left(\bar{s}^+_t\left|\kappa_i,\mu_i\right.\right)\right] \\
	=&~\E_{\kappa_i\sim p(\kappa),\mu_i\sim p\left(\mu|\kappa_i\right)}\left[\E_{c_{:t}\sim p_{\overline{M}^+}\left(c_{:t}|\kappa_i,\mu_i\right)}\left[\sP_{\kappa_i}\left(r_t,s_{t+1}\left|s_t,a_t\right.\right)\cdot \mathbbm{1}\left[b_t^{\kappa,\mu}=p(\kappa,\mu|c_{:t})\right]\right]\right] \\
	=&~\E_{(\kappa_i,\mu_i)\sim b_t^{\kappa,\mu}}\left[\sP_{\kappa_i}\left(r_t,s_{t+1}\left|s_t,a_t\right.\right)\right] \\
	=&~\sP_{\overline{M}^{+},\bar{\pi}^+}\left(r_t,s_{t+1}\left|\bar{s}_t^+,a_t\right.\right).
	\end{align}
	Thus, the data distribution induced by $\bar{\pi}^+$ and $[\mu]$ matches. 
	
	\textbf{Part (iii)} Directly use Lemma \ref{lem:off-learn-gua} in a transformed BAMDP $\overline{M}^+$, in which $\overline{M}^+$ is a belief MDP, a type of MDP. Therefore, the policy evaluation of $\bar{\pi}^+$ in offline meta-training and online adaptation will be asymptotically consistent, as the offline dataset grows.
\end{proof}

To achieve in-distribution online adaptation, transformed BAMDPs incorporate additional information about offline data collection into the beliefs of BAMDPs. We prove that transformed BAMDPs require RL agents to filter out out-of-distribution episodes to support feasible belief updating of behavior policies. In this way, the distribution of reward and transition between offline and online data collection coincide, which can provide the guarantee of consistent policy evaluation between $\gJ_{\overline{\gD}^+}\left(\bar{\pi}^+\right)$ and $\gJ_{\overline{M}^{+}}\left(\bar{\pi}^+\right)$. Theorem \ref{thm:DSRTID} shows that we can meta-train policies with offline policy evaluation and utilize in-distribution online adaptation to guarantee the final performance in meta-testing.

\subsection{Main Results in Section \ref{sec:generate-in-dis}}\label{appendix:theory-main-results-part3}

\begin{restatable}[Sub-Datasets Collected by Single Task Data Collection]{definition}{SubDatasetPerTask}\label{def:sub-dataset-per-task} 
	In a transformed BAMDP $\overline{M}^{+}$, an offline multi-task dataset $\overline{\gD}^+$ is drawn from the task-dependent data collection $\sP_{\overline{M}^+,[\mu]}$. A sub-dataset collected by a behavior policy $\mu_i$ in a task $\kappa_i$ is defined by $\gD_{\kappa_i,\mu_i}$. Note an offline multi-task dataset $\overline{\gD}^+$ is the union of sub-datasets $\gD_{\kappa_i,\mu_i}$, i.e., 
	\begin{align}
	\overline{\gD}^+ = \bigcup_{\kappa_i,\mu_i}\gD_{\kappa_i,\mu_i}.
	\end{align}
\end{restatable}

For each sub-dataset $\gD_{\kappa_i,\mu_i}$, we can define a batch-constrained policy set in a single-task $(\kappa_i,\mu_i)$ as $\Pi^{\gD_{\kappa_i,\mu_i}}$ (see the definition in Eq. (\ref{appedix-eq:constrained-pi})).

\begin{restatable}[Meta-Policy with Thompson  Sampling]{definition}{MetaPiThompsonS}\label{def:meta-pi-Thompson-sampling} 
	For each transformed BAMDP $\overline{M}^{+}$, a meta-policy set with Thompson sampling on $\overline{M}^{+}$ is defined by $\bar{\pi}^{+,T}:\gS\times\overline{\gB}\times\gK\times\Pi_{[\mu]}\rightarrow\Delta\left(\gA\right)$, where $\overline{\gB}$ is the space of beliefs over meta-RL MDPs with behavior policies, $\gK$ is the space of task $\kappa$, and $\Pi_{[\mu]}$ is the space of task-dependent behavior policies. In each episode, $\bar{\pi}^{+,T}$ samples a task hypothesis $(\kappa_i,\mu_i)$ from the current belief $b_{t'}^{\kappa,\mu}$, where $t'$ is the starting step in this episode. During this episode, $\bar{\pi}^{+,T}\left(\cdot\left|s_t,b_{t'}^{\kappa,\mu},\kappa_i,\mu_i\right.\right)$ prescribes a distribution over actions for each state $s_t$, belief $b_{t'}^{\kappa,\mu}$, and task hypothesis $(\kappa_i,\mu_i)$. Beliefs $b_{t'}^{\kappa,\mu}$ and task hypotheses $(\kappa_i,\mu_i)$ will periodically update after each episode.
\end{restatable}

In the deep-learning-based implementation, a context-based meta-RL algorithm, PEARL \citep{rakelly2019efficient}, utilizes a meta-policy with Thompson sampling \citep{strens2000bayesian} to iteratively update task belief by interacting with the environment and improve the meta-policy based on the ``task hypothesis'' sampled from the current beliefs. We can adopt such adaptation protocol to design practical offline meta-RL algorithms for transformed BAMDPs.

\begin{restatable}[Batch-Constrained Meta-Policy Set with Thompson  Sampling]{definition}{BatchConMetaPiThompsonS}\label{def:batch-constrained-meta-pi-Thompson-sampling} 
	For each transformed BAMDP $\overline{M}^{+}$ with an offline multi-task dataset $\overline{\gD}^+$, a batch-constrained meta-policy set with Thompson sampling is defined by 
	\begin{align}\label{appedix-eq:x}
	\Pi^{\overline{\gD}^+,T}=\left\{\bar{\pi}^{+,T}\left|\bar{\pi}^{+,T}\left(a_t\left|s_t,b_{t'}^{\kappa,\mu},\kappa_i,\mu_i\right.\right)=0 \text{ whenever } (s_t,a_t)\not\in\gD_{\kappa_i,\mu_i}, \forall b_{t'}^{\kappa,\mu}\right.\right\},
	\end{align}
	where denoting $(s_t,a_t)\in\gD_{\kappa_i,\mu_i}$ if there exists a trajectory containing $(s_t,a_t)$ in the dataset $\gD_{\kappa_i,\mu_i}$. 
\end{restatable}

The batch-constrained meta-policy set with Thompson sampling $\Pi^{\overline{\gD}^+,T}$ consists of the meta-policies that for any state $s_t$ observed in the hypothesis dataset $\gD_{\kappa_i,\mu_i}$, the agent will not select an action outside of the dataset. Note that in each episode with a task hypothesis $(\kappa_i,\mu_i)$, a batch-constrained meta-policy with Thompson sampling $\bar{\pi}^{+,T}$ is batch-constrained within a sub-dataset $\gD_{\kappa_i,\mu_i}$, i.e., $\forall b_{t'}^{\kappa,\mu}$, we have  $\bar{\pi}^{+,T}\left(\cdot\left|s_t,b_{t'}^{\kappa,\mu},\kappa_i,\mu_i\right.\right)\in\Pi^{\gD_{\kappa_i,\mu_i}}$.

\begin{restatable}[Probability that a Policy Leaves the Dataset]{definition}{ProbPiLeaveDataset}\label{def:prob-pi-leave-dataset} 
	In an MDP $M$ and an arbitrary offline dataset $\gD$, for each policy $\pi:\gS\rightarrow\Delta\left(\gA\right)$, the probability that executing $\pi$ in $M$ leaves the dataset $\gD$ for an episode is defined by 
	\begin{align}
	p_{out}^{M,\gD}(\pi) &= \sum_{\tau_H} p^{M}_{\pi}\left(\tau_H\right)\mathbbm{1}\left[\tau_H \text{ leaves } \gD\right] \\
	&= \sum_{\tau_H} p^{M}_{\pi}\left(\tau_H\right)\mathbbm{1}\left[\exists t\in[H] \text{ s.t. } s_t \not\in \gD \text{ or } (s_t, a_t, r_t) \not\in \gD\right],
	\end{align}
	where $p^{M}_{\pi}\left(\tau_H\right)$ is the probability of executing $\pi$ in $M$ to generate an $H$-horizon trajectory $\tau_H$ (see the definition in Eq. (\ref{appendix-eq:prob-tau})), denoting $s_t \in\gD$ if there exists a trajectory containing $s_t$ in the dataset $\gD$, and similarly for $(s_t, a_t, r_t) \in \gD$.
\end{restatable}

When we aim to confine the agent in the in-distribution states with high probability as the offline dataset $\gD$ grows, it is equivalent to binding the probability that executing a policy $\pi$ in $M$ leaves the dataset $\gD$ for an episode, i.e., $p_{out}^{M,\gD}(\pi)$.

\ThompsonGenID*
\begin{proof}
	Denote the current belief by $b_{t'}^{\kappa,\mu}$ and the task hypothesis by $(\kappa_i,\mu_i)$. Thus, for each batch-constrained meta-policy with Thompson sampling $\bar{\pi}^{+,T}\in\Pi^{\overline{\gD}^+,T}$ with $\left(b_{t'}^{\kappa,\mu}, \kappa_i,\mu_i\right)$, similar to Definition \ref{def:prob-pi-leave-dataset}, define the probability that executing $\bar{\pi}^{+,T}\left(\cdot\left|s_t,b_{t'}^{\kappa,\mu},\kappa_i,\mu_i\right.\right)$ leaves the dataset $\overline{\gD}^+$ in an adaptation episode of a meta-testing task $\kappa_{test}\sim p(\kappa)$: 
	\begin{align}
	p_{out}^{\kappa_{test},\overline{\gD}^+}\left(\bar{\pi}^{+,T},b_{t'}^{\kappa,\mu},\kappa_i,\mu_i\right) &= \sum_{\bar{\tau}_{H}^+} p^{\kappa_{test}}_{\bar{\pi}^{+,T}}\left(\left.\bar{\tau}_{H}^+\right|b_{t'}^{\kappa,\mu},\kappa_i,\mu_i\right)\mathbbm{1}\left[\bar{\tau}_{H}^+ \text{ leaves } \overline{\gD}^+\right],
	\end{align}
	where $p^{\kappa_{test}}_{\bar{\pi}^{+,T}}\left(\left.\bar{\tau}_{H}^+\right|b_{t'}^{\kappa,\mu},\kappa_i,\mu_i\right)$ is the probability of executing $\bar{\pi}^{+,T}\left(\cdot\left|s_t,b_{t'}^{\kappa,\mu},\kappa_i,\mu_i\right.\right)$ in $\kappa_{test}$ to generate an $H$-horizon trajectory $\bar{\tau}_{H}^+$ in an adaptation episode, i.e.,
	\begin{align}
	&~p^{\kappa_{test}}_{\bar{\pi}^{+,T}}\left(\left.\bar{\tau}_{H}^+\right|b_{t'}^{\kappa,\mu},\kappa_i,\mu_i\right) \\
	=&\left(\prod_{t=t'}^{t'+H-1}\bar{\pi}^{+,T}\left(a_t\left|s_t,b_{t'}^{\kappa,\mu},\kappa_i,\mu_i\right.\right)\cdot R^{\kappa_{test}}\left(r_t|s_t,a_t\right)\right) \prod_{t=t'}^{t'+H-2}P^{\kappa_{test}}\left(s_{t+1}\left|s_t,a_t,r_t\right.\right) \\
	=&~ p^{\kappa_{test}}_{\bar{\pi}^{+,T}}\left(\left.\tau_{H}\right|b_{t'}^{\kappa,\mu},\kappa_i,\mu_i\right),
	\end{align} 
	where we can transform $\bar{\tau}_{H}^+$ to $\tau_{H}$ with the same probability since the belief $b_{t'}^{\kappa,\mu}$ and task hypothsis $(\kappa_i,\mu_i)$ will periodically update after each episode. Therefore, 
	\begin{align}
	p_{out}^{\kappa_{test},\overline{\gD}^+}\left(\bar{\pi}^{+,T},b_{t'}^{\kappa,\mu},\kappa_i,\mu_i\right) &= \sum_{\bar{\tau}_{H}^+} p^{\kappa_{test}}_{\bar{\pi}^{+,T}}\left(\left.\tau_{H}\right|b_{t'}^{\kappa,\mu},\kappa_i,\mu_i\right)\mathbbm{1}\left[\bar{\tau}_{H}^+ \text{ leaves } \overline{\gD}^+\right] \\
	&\le \sum_{\tau_H} p^{\kappa_{test}}_{\bar{\pi}^{+,T}}\left(\left.\tau_{H}\right|b_{t'}^{\kappa,\mu},\kappa_i,\mu_i\right)\mathbbm{1}\left[\tau_H \text{ leaves } \gD_{\kappa_{i^*},\mu_{i^*}}\right] \\
	&= p_{out}^{\kappa_{test},\gD_{\kappa_{i^*},\mu_{i^*}}}\left(\left.\bar{\pi}^{+,T}\right|b_{t'}^{\kappa,\mu},\kappa_i,\mu_i\right),
	\end{align}
	where $p_{out}^{\kappa_{test},\gD_{\kappa_{i^*},\mu_{i^*}}}\left(\left.\bar{\pi}^{+,T}\right|b_{t'}^{\kappa,\mu},\kappa_i,\mu_i\right)$ is the probability that executing $\bar{\pi}^{+,T}\left(\cdot\left|s_t,b_{t'}^{\kappa,\mu},\kappa_i,\mu_i\right.\right)$ in $\kappa_{test}$ leaves the sub-dataset $\gD_{\kappa_{i^*},\mu_{i^*}}$ for an episode (see Definition \ref{def:prob-pi-leave-dataset}), $\gD_{\kappa_{i^*},\mu_{i^*}}$ is a sub-dataset collected in $\overline{\gD}^+$ (see Definition \ref{def:sub-dataset-per-task}) and $\kappa_{i^*}$ is the closest offline meta-training task to $\kappa_{test}$, i.e., 
	\begin{align}
	\kappa_{i^*} &= \mathop{\arg\min}_{\kappa_i\in\gK_{train}} \left\|\kappa_i-\kappa_{test}\right\|_{\infty} \\
	&=\mathop{\arg\min}_{\kappa_i\in\gK_{train}}\max\left(\left\|P^{\kappa_i}(s,a,s') - P^{\kappa_{test}}(s,a,s')\right\|_{\infty}, \left\|R^{\kappa_i}(s,a,r) - R^{\kappa_{test}}(s,a,r)\right\|_{\infty}\right), 
	\end{align}
	in which denoting the i.i.d. offline meta-training tasks sampled from $p(\kappa)$ in $\overline{\gD}^+$ by $\gK_{train}$. From Lemma \ref{lem:thm-3-formal}, as the offline dataset $\overline{\gD}^+$ grows, $\gD_{\kappa_{i^*},\mu_{i^*}}$ and $\gK_{train}$ grow monotonically, for any batch-constrained policy $\pi$ in $\gD_{\kappa_{i^*},\mu_{i^*}}$, i.e., $\forall \pi\in\Pi^{\gD_{\kappa_{i^*},\mu_{i^*}}}$, when executing $\pi$ in an episode of $\kappa_{test}\sim p(\kappa)$, the probability leaving the dataset $\overline{\gD}^+$ is $p_{out}^{\kappa_{test},\gD_{\kappa_{i^*},\mu_{i^*}}}(\pi)$, which asymptotically approaches zero. 
	
	For the first adaptation episode in a meta-testing task $\kappa_{test}\sim p(\kappa)$ with the prior belief $p(\kappa,\mu)=p(\kappa)p(\mu|\kappa)$, there exists a task hypothesis $\left(\kappa_{i^*},\mu_{i^*}\right)$ in the prior $p(\kappa,\mu)$, then due to $\bar{\pi}^{+,T}\left(\cdot\left|s_t,p(\kappa,\mu),\kappa_{i^*},\mu_{i^*}\right.\right)\in\Pi^{\gD_{\kappa_{i^*},\mu_{i^*}}}$ from Definition \ref{def:batch-constrained-meta-pi-Thompson-sampling} and 
	\begin{align}
	p_{out}^{\kappa_{test},\overline{\gD}^+}\left(\bar{\pi}^{+,T},p(\kappa,\mu),\kappa_{i^*},\mu_{i^*}\right) \le p_{out}^{\kappa_{test},\gD_{\kappa_{i^*},\mu_{i^*}}}\left(\left.\bar{\pi}^{+,T}\right|p(\kappa,\mu),\kappa_{i^*},\mu_{i^*}\right),
	\end{align}
	as the offline dataset $\overline{\gD}^+$ grows, executing $\bar{\pi}^{+,T}$ with $\left(p(\kappa,\mu), \kappa_{i^*},\mu_{i^*}\right)$ for the first episode in $\kappa_{test}$ will confine the agent in in-distribution hyper-states with high probability. In the subsequent adaptation episodes with current belief $b_{t'}^{\kappa,\mu}$ in $\kappa_{test}$, the task hypothesis $\left(\kappa_{i^*},\mu_{i^*}\right)$ is also in the belief $b_{t'}^{\kappa,\mu}$ by induction.
	
	Therefore, for each adaptation episode with current belief $b_{t'}^{\kappa,\mu}$ in $\kappa_{test}$, there exists a task hypothesis $(\kappa_i,\mu_i)$ from $b_{t'}^{\kappa,\mu}$, e.g., $\left(\kappa_{i^*},\mu_{i^*}\right)$, executing $\bar{\pi}^{+,T}$ with $\left(b_{t'}^{\kappa,\mu}, \kappa_i, \mu_i\right)$ in $\kappa_{test}$ will confine the agent in in-distribution hyper-states with high probability, as the offline dataset $\overline{\gD}^+$ grows.
\end{proof}

Theorem \ref{thm:ThompsonGenID} indicates that for each adaptation episode, we can sample task hypotheses from the current task belief and execute $\bar{\pi}^{+,T}$ to interact with the environment until finding an in-distribution episode. For example in Figure~\ref{fig:MultiTaskDDMExample}, after offline meta-training, a meta-policy with Thompson sampling $\bar{\pi}^{+,T}$ will perform $a_i$ with a task hypothesis of $\kappa_i$ and expect to receive a reward 1. During online meta-testing, $\kappa_{test}$ is drawn from $p(\kappa)$ and the agent needs to infer the task identification. To achieve in-distribution online adaptation, $\bar{\pi}^{+,T}$ will try various actions according to diverse task hypotheses until sampling an in-distribution episode with a reward 1. Updating the task belief with the in-distribution episode, RL agents can infer and solve this task.

In contrast, when updating task belief using an out-of-distribution episode with a reward 0, the posterior task belief will be out of the offline dataset $\overline{\gD}^+$. Note that offline training paradigm can not well-optimize $\bar{\pi}^{+,T}$ on out-of-distribution states \citep{fujimoto2019off} and policy $\bar{\pi}^{+,T}$ will fail in this case. Moreover, Thompson sampling is very popular in context-based deep meta-RL \citep{rakelly2019efficient} and we will generalize these theoretical implications.

Note that Theorem \ref{thm:ThompsonGenID} considers arbitrary task distribution $p(\kappa)$, since the distance between the closest meta-training task $\kappa_{i^*}$ and $\kappa_{test}$ will asymptotically approach zero with high probability, as the i.i.d. offline meta-training tasks $\gK_{train}$ sampled from $p(\kappa)$ in $\overline{\gD}^+$ grows. 

\subsection{Omitted Assumptions and Propsitions in Section \ref{sec:IDAQ}}\label{appendix:omitted-section-IDAQ}

To analyze return-based uncertainty quantification, we first present a mild assumption: 
\begin{assumption} \label{assumption:return-uncentain}
	In a transformed BAMDP $\overline{M}^{+}$, after effective offline meta-training of meta-policy with Thompson sampling $\bar{\pi}^{+,T}$, with high probability, there exists (accurate or nearly accurate) task hypotheses from the current belief, the online policy evaluation has higher returns on in-distribution episodes than out-of-distribution episodes with incorrect hypotheses.
\end{assumption}

This assumption matches the bias of offline RL \citep{fujimoto2019off} that offline meta-training can not well-optimize meta-policies on out-of-distribution states. 

\begin{restatable}{proposition}{ReturnGenID}\label{prop:ReturnGenID}
	In a transformed BAMDP $\overline{M}^{+}$, after effective offline meta-training of meta-policy with Thompson sampling $\bar{\pi}^{+,T}$, with Assumption \ref{assumption:return-uncentain}, return-based uncertainty quantification $\sQ_{RE}$ can support $\bar{\pi}^{+,T}$ generate in-distribution online adaptation episodes.
\end{restatable}
\begin{proof}
	From Theorem \ref{thm:ThompsonGenID}, meta-policy with Thompson sampling $\bar{\pi}^{+,T}$ and accurate or nearly accurate task hypotheses can generate in-distribution online episodes with high probability. Combining with Assumption \ref{assumption:return-uncentain}, for each adaptation episode, there exists (accurate or nearly accurate) task hypotheses whose online policy evaluation has higher returns. Thus, from the current belief, we can sample task hypotheses and execute $\bar{\pi}^{+,T}$ to interact with the environment until finding an in-distribution episode using $\sQ_{RE}$.
\end{proof}

\newpage
\subsection{Omitted Proof of Lemma \ref{lem:off-learn-gua}}\label{appendix:omitted-proof-Lemma-1}

\begin{restatable}[Dataset Induced Finite-Horizon MDPs]{definition}{DEFDataMDP}\label{def:data-mdp} 
	In a finite-horizon MDP $M$ with an offline dataset $\gD$, a dataset induced finite-horizon MDP is defined by $M^{\gD}=\left(\gS, \gA, \gR, H, P^{M^{\gD}}, R^{M^{\gD}}\right)$, with the same state space, action space, reward space, and horizon as $M$. The transition function is defined as follows: $\forall s,s'\in\gS,a\in\gA$,
	\begin{align}
	P^{M^{\gD}}(s'|s, a) =\left\{
	\begin{aligned}
	&\frac{N(s,a,s')}{N(s,a)}, &\text{ if }N(s,a)>0, \\
	&0, &\text{ otherwise},
	\end{aligned}
	\right.
	\end{align}
	where $N(s,a,s')$ and $N(s,a)$ are the number of times the tuples $(s,a,s')$ and $(s,a)$ are observed in $\gD$, respectively. The reward function is defined by $\forall s\in\gS,a\in\gA,r\in\gR$,
	\begin{align}
	R^{M^{\gD}}(r|s, a) =\left\{
	\begin{aligned}
	&\frac{N(s,a,r)}{N(s,a)}, &\text{ if }N(s,a)>0, \\
	&0, &\text{ otherwise},
	\end{aligned}
	\right.
	\end{align}
	where $N(s,a,r)$ is the number of times the tuple $(s,a,r)$ are observed in $\gD$. The offline policy evaluation in $\gD$ is equal to the policy evaluation in $M_{\gD}$, i.e., for any batch-constrained policy $\pi$, 
	\begin{align}
	\gJ_{\gD}(\pi)=\gJ_{M_{\gD}}(\pi).
	\end{align}
\end{restatable}

Note that dataset induced finite-horizon MDPs $M^{\gD}$ are not defined on supports outside of dataset $\gD$. For simplicity, We set all undefined numbers to 0 in the transition and reward function.

\begin{lemma}[Simulation Lemma for Offline Finite-Horizon MDPs]\label{lem:simulation-lemma-offline}
	In an MDP $M$ with an offline dataset $\gD$, for any batch-constrained policy $\pi\in\Pi^{\gD}$, if
	\begin{align}
		&\max_{s\in\gS,a\in\gA\text{ with } \rho_{\pi}^{M_{\gD}}(s,a)>0}\left\|P^{M_{\gD}}(\cdot|s,a)-P^{M}(\cdot|s,a)\right\|_1\le \epsilon_P, \\
		&\max_{s\in\gS,a\in\gA\text{ with } \rho_{\pi}^{M_{\gD}}(s,a)>0}\left|r^{M_{\gD}}(s,a)-r^{M}(s,a)\right|\le \epsilon_r, \\
		&\max_{s\in\gS,a\in\gA}\max\left(r^{M_{\gD}}(s,a), r^{M}(s,a)\right)\le r_{max}, 
	\end{align}
	where $r^M(s,a)=\E_{\tilde{r}\sim R^M(s,a)}[\tilde{r}]$ and $r^{M_{\gD}}(s,a)=\E_{\tilde{r}\sim R^{M_{\gD}}(s,a)}[\tilde{r}]$, we have
	\begin{align}
	\left|\gJ_{M_{\gD}}(\pi)-\gJ_{M}(\pi)\right|\le H\epsilon_r+ \frac{H(H-1)r_{max}}{2}\epsilon_P.
	\end{align}
\end{lemma}
\begin{proof}
	Similar to the famous Simulation Lemma in finite-horizon MDPs \citep{FiniteSimulationLemma}, the proof is as follows. Recall value function $\forall h\in[H-1]$, $\forall s\in\gS_h$ (see Eq. (\ref{appendix-eq:value-function})),
	\begin{align}
	V_\pi^{M}(s)=\sum_{a\in\gA}\pi(a|s)\left(r^{M}(s,a)+\sum_{s'\in\bar{\gS}_{h+1}}P^{M}(s'|s,a)V_\pi^{M}(s')\right),
	\end{align}
	$\gJ_{M}(\pi)=V_\pi^{M}\left(s_0\right)$, and $\forall h\in[H]$,
	\begin{align}
	\max_{s\in\bar{\gS}_h}\left|V_\pi^{\bar{\kappa}}(s)\right|\le (H-h) r_{max}.
	\end{align}
	We will prove $\forall h\in[H], \forall s\in\gS_h$ with $\rho_{\pi}^{M_{\gD}}(s)>0$,
	\begin{align}
	\left|V_\pi^{M_{\gD}}(s)-V_\pi^{M}(s)\right| \le (H-h)\epsilon_r+\frac{(H-h)(H-h-1)r_{max}}{2}\epsilon_P
	\end{align}
	by induction. When $h=H-1$, we have $\forall s\in\gS_h$ with $\rho_{\pi}^{M_{\gD}}(s)>0$,
	\begin{align}
	\left|V_\pi^{M_{\gD}}(s)-V_\pi^{M}(s)\right|=\left|\sum_{a\in\gA}\pi(a|s)r^{M_{\gD}}(s,a)-\sum_{a\in\gA}\pi(a|s)r^{M}(s,a)\right| \le \epsilon_r
	\end{align}
	holds. And $\forall h\in[H-1]$, $\forall s\in\gS_h$ with $\rho_{\pi}^{M_{\gD}}(s)>0$,
	\begin{align}
	&~\left|V_\pi^{M_{\gD}}(s)-V_\pi^{M}(s)\right| \\
	=&~\Bigg|\sum_{a\in\gA}\pi(a|s)\left(r^{M_{\gD}}(s,a)+\sum_{s'\in\gS_{h+1}}P^{M_{\gD}}(s'|s,a)V_\pi^{M_{\gD}}(s')\right)-\\
	&~\quad\quad~~\sum_{a\in\gA}\pi(a|s)\left(r^{M}(s,a)+\sum_{s'\in\gS_{h+1}}P^{M}(s'|s,a)V_\pi^{M}(s')\right)\Bigg| \\
	\le&~\sum_{a\in\gA}\pi(a|s)\left|r^{M_{\gD}}(s,a)-r^{M}(s,a)\right| + \\
	&~\sum_{a\in\gA}\pi(a|s)\sum_{s'\in\gS_{h+1}}\left|P^{M_{\gD}}(s'|s,a)V_\pi^{M_{\gD}}(s')-P^{M}(s'|s,a)V_\pi^{M}(s')\right| \\
	\le&~\epsilon_r+ \sum_{a\in\gA}\pi(a|s)\sum_{s'\in\gS_{h+1}}\left|P^{M_{\gD}}(s'|s,a)-P^{M}(s'|s,a)\right|V_\pi^{M}(s')+ \\
	&~\sum_{a\in\gA}\pi(a|s)\sum_{s'\in\gS_{h+1}}P^{M_{\gD}}(s'|s,a)\left|V_\pi^{M_{\gD}}(s')-V_\pi^{M}(s')\right| \\
	\le&~\epsilon_r+(H-(h+1))r_{max}\epsilon_P + \\
	&~\left((H-(h+1))\epsilon_r+\frac{(H-(h+1))(H-(h+1)-1)r_{max}}{2}\epsilon_P\right) \\
	=&~(H-h)\epsilon_r+\frac{(H-h)(H-h-1)r_{max}}{2}\epsilon_P.
	\end{align}
	Thus, 
	\begin{align}
	\left|\gJ_{M_{\gD}}(\pi)-\gJ_{M}(\pi)\right|&=\left|V_\pi^{M_{\gD}}\left(s_0\right)-V_\pi^{M}\left(s_0\right)\right| \le H\epsilon_r+ \frac{H(H-1)r_{max}}{2}\epsilon_P.
	\end{align}
\end{proof}

\begin{lemma}\label{lem:lemma1-formal}
	In an MDP $M$ with an offline dataset $\gD$ collected by a behavior policy $\mu$, for any batch-constrained policy $\pi\in\Pi^{\gD}$, $\forall \delta\in (0,1]$, with probability $1-\delta$,
	\begin{align}
	\left|\gJ_{M_{\gD}}(\pi)-\gJ_{M}(\pi)\right|\le H^2\left|\gS\right|\sqrt{\frac{\log\left(\frac{1}{\delta}\right)+\log\left(2\left|\gS\right|^2\left|\gA\right|\right)}{Kd_\mu^{M_{\gD}}}},
	\end{align}
	where $K$ is the number of trajectories in the dataset $\gD$ and $d_\mu^{M_{\gD}}$ is the minimal visitation state-action distribution induced by the behavior policy $\mu$ in $M_{\gD}$ (see Eq. (\ref{appendix-eq:dm})).
\end{lemma}
\begin{proof}
	$\forall s,s'\in\gS, a\in\gA$ with $\rho_{\pi}^{M_{\gD}}(s,a)>0$, note that $\rho_{\pi}^{M_{\gD}}(s,a)\ge d_\mu^{M_{\gD}}$ and according to Binomial theorem and Hoeffding’s inequality, $\forall \epsilon\in[0,1]$,
	\begin{align}
	\sP\left(\left|P^{M_{\gD}}(s'|s,a)-P^M(s'|s,a) \right|\ge \epsilon\right) &\le 2\left(1-d_\mu^{M_{\gD}}+d_\mu^{M_{\gD}}\exp\left(-2\epsilon^2\right)\right)^{K} \\
	&\le 2\left(1-d_\mu^{M_{\gD}}\epsilon^2\right)^{K}
	\end{align}
	and
	\begin{align}
	\sP\left(\left|r^{M_{\gD}}(s,a)-r^M(s,a) \right|\ge \epsilon\right) &\le 2\left(1-d_\mu^{M_{\gD}}\epsilon^2\right)^{K}.
	\end{align}
	Thus, using union bound, $\forall \delta\in (0,1]$, with probability $1-\delta$, denote
	\begin{align}
	\epsilon_P&=\max_{s\in\gS,a\in\gA\text{ with } \rho_{\pi}^{M_{\gD}}(s,a)>0}\left\|P^{M_{\gD}}(\cdot|s,a)-P^{M}(\cdot|s,a)\right\|_1 \\
	&\le\max_{s\in\gS,a\in\gA\text{ with } \rho_{\pi}^{M_{\gD}}(s,a)>0}\left|\gS\right|\left\|P^{M_{\gD}}(\cdot|s,a)-P^{M}(\cdot|s,a)\right\|_{\infty} \\ &\le\left|\gS\right|\sqrt{\frac{\log\left(\frac{1}{\delta}\right)+\log\left(2\left|\gS\right|^2\left|\gA\right|\right)}{Kd_\mu^{M_{\gD}}}},
	\end{align}
	\begin{align}
	\epsilon_r&=\max_{s\in\gS,a\in\gA\text{ with }\rho_{\pi}^{M_{\gD}}(s,a)>0}\left|r^{M_{\gD}}(s,a)-r^M(s,a) \right| &\le\sqrt{\frac{\log\left(\frac{1}{\delta}\right)+\log\left(2\left|\gS\right|\left|\gA\right|\right)}{Kd_\mu^{M_{\gD}}}},
	\end{align}
	and $r_{max}=1$, thus, according to Lemma \ref{lem:simulation-lemma-offline} (a variant of the simulation lemma in offline RL), we have 
	\begin{align}
	\left|\gJ_{M_{\gD}}(\pi)-\gJ_{M}(\pi)\right|&\le \left(H+\frac{H(H-1)}{2}\left|\gS\right|\right)\sqrt{\frac{\log\left(\frac{1}{\delta}\right)+\log\left(2\left|\gS\right|^2\left|\gA\right|\right)}{Kd_\mu^{M_{\gD}}}} \\
	&\le H^2\left|\gS\right|\sqrt{\frac{\log\left(\frac{1}{\delta}\right)+\log\left(2\left|\gS\right|^2\left|\gA\right|\right)}{Kd_\mu^{M_{\gD}}}}
	\end{align}
\end{proof}

\OffLearnGua*
\begin{proof}
	From Lemma \ref{lem:lemma1-formal}, as the size of an offline dataset $\left|\gD\right|=KH$ grows, the gap between approximate offline policy evaluation $\gJ_{\gD}(\pi)$ and accurate policy evaluation $\gJ_{M}(\pi)$ will asymptotically approach to 0.
\end{proof}

\newpage
\subsection{Omitted Lemmas for Theorem \ref{thm:ThompsonGenID}}\label{appendix:omitted-proof-Thm-3}

\begin{lemma}\label{lem:prob-pi-leave-dataset-lem-1}
	In an MDP $M$ and an arbitrary offline dataset $\gD$, for each policy $\pi$, 
	\begin{align}
		p_{out}^{M,\gD}(\pi) \le H\left(\left\|\rho_{\pi}^M(s)-\rho_{\pi}^{M_{\gD}}(s)\right\|_1 + \left\|\rho_{\pi}^M(s,a,r)-\rho_{\pi}^{M_{\gD}}(s,a,r)\right\|_1\right),
	\end{align}
	where $\rho_{\pi}^{M}(s,a,r)$ is the visitation distribution of $(s,a,r)$ in $M$, as defined in Eq. (\ref{appendix-eq:visit-dis-r}).
\end{lemma}
\begin{proof}
	For each $\tau_H$ leaving $\gD$, we use the first outlier data ($s_t \not\in \gD$ or $(s_t, a_t, r_t) \not\in \gD$) to represent $\tau_H$. Thus,
	\begin{align}
		p_{out}^{M,\gD}(\pi) &= \sum_{\tau_H} p^{M}_{\pi}\left(\tau_H\right)\mathbbm{1}\left[\exists t\in[H] \text{ s.t. } s_t \not\in \gD \text{ or } (s_t, a_t, r_t) \not\in \gD\right] \\
		&\le H\left(\sum_{s\in\gS\text{ with } \rho_{\pi}^{M_{\gD}}(s)=0} \rho_{\pi}^{M}(s)+\sum_{s\in\gS,a\in\gA,r\in\gR\text{ with } \rho_{\pi}^{M_{\gD}}(s,a,r)=0} \rho_{\pi}^{M}(s,a,r)\right) \\
		&\le H\left(\left\|\rho_{\pi}^M(s)-\rho_{\pi}^{M_{\gD}}(s)\right\|_1 + \left\|\rho_{\pi}^M(s,a,r)-\rho_{\pi}^{M_{\gD}}(s,a,r)\right\|_1\right).
	\end{align}
\end{proof}

\begin{lemma}\label{lem:thm-3-formal}
	In a transformed BAMDP $\overline{M}^{+}$ with an offline multi-task dataset $\overline{\gD}^+$ collected by task-dependent behavior policies $[\mu]$, denoting the i.i.d. offline meta-training tasks sampled from $p(\kappa)$ in $\overline{\gD}^+$ by $\gK_{train}$, for each meta-testing task $\kappa_{test}\sim p(\kappa)$, and denoting the closest offline meta-training task to $\kappa_{test}$ by $\kappa_{i^*}$, i.e., 
	\begin{align}\label{appendix-eq:closest-meta-train-task}
	\kappa_{i^*} &= \mathop{\arg\min}_{\kappa_i\in\gK_{train}} \left\|\kappa_i-\kappa_{test}\right\|_{\infty} \\ \label{appendix-eq:task-distance-linfty}
	&=\mathop{\arg\min}_{\kappa_i\in\gK_{train}}\max\left(\left\|P^{\kappa_i}(s,a,s') - P^{\kappa_{test}}(s,a,s')\right\|_{\infty}, \left\|R^{\kappa_i}(s,a,r) - R^{\kappa_{test}}(s,a,r)\right\|_{\infty}\right), 
	\end{align}
	then for any batch-constrained policy $\pi$ in $\gD_{\kappa_{i^*},\mu_{i^*}}$, where $\gD_{\kappa_{i^*},\mu_{i^*}}$ is a sub-dataset collected in $\overline{\gD}^+$ (see Definition \ref{def:sub-dataset-per-task}), i.e., $\forall\pi\in\Pi^{\gD_{\kappa_{i^*},\mu_{i^*}}}$,
	\begin{align}
	&~p_{out}^{\kappa_{test},\gD_{\kappa_{i^*},\mu_{i^*}}}(\pi) \\
	\le&~2H^2\left|\gS\right|^2\left|\gA\right|\left|\gR\right|\left(\underbrace{\sqrt{\frac{\log\left(\frac{1}{\delta}\right)+\log\left(2\left|\gS\right|^2\left|\gA\right|\right)}{K_{\kappa_{i^*},\mu_{i^*}}\cdot d_\mu^{M_{\gD_{\kappa_{i^*},\mu_{i^*}}}}}}}_{\substack{\text{\small Asymptotically approaches zero,} \\ \text{\small when $\gD_{\kappa_{i^*},\mu_{i^*}}$ is sufficiently large}}}+\underbrace{2\left(\frac{\log\left(\frac{1}{\delta}\right)}{\left|\gK_{train}\right|}\right)^{\frac{1}{\left|\gS\right|\left|\gA\right|\left(\left|\gS\right|+\left|\gR\right|\right)}}}_{\substack{\text{\small Asymptotically approaches zero,} \\ \text{\small when $\gK_{train}$ is sufficiently large}}}\right),
	\end{align}
	where $K_{\kappa_{i^*},\mu_{i^*}}$ is the number of trajectories in the sub-dataset $\gD_{\kappa_{i^*},\mu_{i^*}}$, $d_\mu^{M_{\gD_{\kappa_{i^*},\mu_{i^*}}}}$ is the minimal visitation state-action distribution induced by the behavior policy $\mu$ in $M_{\gD_{\kappa_{i^*},\mu_{i^*}}}$ (see Eq. (\ref{appendix-eq:dm})), and $\left|\gK_{train}\right|$ is the number of i.i.d. offline meta-training tasks sampled from $p(\kappa)$ in $\overline{\gD}^+$.
\end{lemma}
\begin{proof}
	According to Lemma \ref{lem:prob-pi-leave-dataset-lem-1},
	\begin{align}
	&~p_{out}^{\kappa_{test},\gD_{\kappa_{i^*},\mu_{i^*}}}(\pi) \\
	\le&~ H\left(\left\|\rho_{\pi}^{M_{\gD_{\kappa_{i^*},\mu_{i^*}}}}(s)-\rho_{\pi}^{\kappa_{test}}(s)\right\|_1 + \left\|\rho_{\pi}^{M_{\gD_{\kappa_{i^*},\mu_{i^*}}}}(s,a,r)-\rho_{\pi}^{\kappa_{test}}(s,a,r)\right\|_1\right) \\
	\le&~H\left(\left\|\rho_{\pi}^{M_{\gD_{\kappa_{i^*},\mu_{i^*}}}}(s)-\rho_{\pi}^{\kappa_{i^*}}(s)\right\|_1 + \left\|\rho_{\pi}^{M_{\gD_{\kappa_{i^*},\mu_{i^*}}}}(s,a,r)-\rho_{\pi}^{\kappa_{i^*}}(s,a,r)\right\|_1\right) + \\
	&~H\left(\left\|\rho_{\pi}^{\kappa_{i*}}(s)-\rho_{\pi}^{\kappa_{test}}(s)\right\|_1 + \left\|\rho_{\pi}^{\kappa_{i*}}(s,a,r)-\rho_{\pi}^{\kappa_{test}}(s,a,r)\right\|_1\right).
	\end{align}
	\paragraph{Part I} Similar with Lemma \ref{lem:lemma1-formal}, $\forall \delta\in (0,1]$, with probability $1-\delta$, denote
	\begin{align}
		\epsilon_P&=\max_{s\in\gS,a\in\gA\text{ with } \rho_{\pi}^{M_{\gD_{\kappa_{i^*},\mu_{i^*}}}}(s,a)>0}\left\|P^{M_{\gD_{\kappa_{i^*},\mu_{i^*}}}}(\cdot|s,a)-P^{\kappa_{i^*}}(\cdot|s,a)\right\|_1 \\ &\le\left|\gS\right|\sqrt{\frac{\log\left(\frac{1}{\delta}\right)+\log\left(2\left|\gS\right|^2\left|\gA\right|\right)}{K_{\kappa_{i^*},\mu_{i^*}}\cdot d_\mu^{M_{\gD_{\kappa_{i^*},\mu_{i^*}}}}}}, \\
		\epsilon_r&=\max_{s\in\gS,a\in\gA\text{ with }\rho_{\pi}^{M_{\gD_{\kappa_{i^*},\mu_{i^*}}}}(s,a)>0}\left|r^{M_{\gD_{\kappa_{i^*},\mu_{i^*}}}}(s,a)-r^{\kappa_{i^*}}(s,a) \right| \\
		&\le\sqrt{\frac{\log\left(\frac{1}{\delta}\right)+\log\left(2\left|\gS\right|\left|\gA\right|\right)}{K_{\kappa_{i^*},\mu_{i^*}}\cdot d_\mu^{M_{\gD_{\kappa_{i^*},\mu_{i^*}}}}}},
	\end{align}
	thus, from Lemma \ref{lem:offline-rho-l1}, we have 
	\begin{align}
		&~H\left(\left\|\rho_{\pi}^{M_{\gD_{\kappa_{i^*},\mu_{i^*}}}}(s)-\rho_{\pi}^{\kappa_{i^*}}(s)\right\|_1 + \left\|\rho_{\pi}^{M_{\gD_{\kappa_{i^*},\mu_{i^*}}}}(s,a,r)-\rho_{\pi}^{\kappa_{i^*}}(s,a,r)\right\|_1\right) \\
		\le&~ 2H^2\left|\gS\right|^2\left|\gA\right|\left|\gR\right|\sqrt{\frac{\log\left(\frac{1}{\delta}\right)+\log\left(2\left|\gS\right|^2\left|\gA\right|\right)}{K_{\kappa_{i^*},\mu_{i^*}}\cdot d_\mu^{M_{\gD_{\kappa_{i^*},\mu_{i^*}}}}}}.
	\end{align}
	\paragraph{Part II} From Lemma \ref{lem:meta-task-dis-bound}, $\forall \delta\in (0,1]$, with probability $1-\delta$, denote
	\begin{align}
		\tilde{\epsilon}_P &=\max_{s\in\gS,a\in\gA}\left\|P^{\kappa_{i^*}}(\cdot|s,a)-P^{\kappa_{test}}(\cdot|s,a)\right\|_1 \\
		&\le \left|\gS\right| \left\|\kappa_{i^*}-\kappa_{test}\right\|_{\infty} \\ 
		&\le 2\left|\gS\right| \left(\frac{\log\left(\frac{1}{\delta}\right)}{\left|\gK_{train}\right|}\right)^{\frac{1}{\left|\gS\right|\left|\gA\right|\left(\left|\gS\right|+\left|\gR\right|\right)}}, \\
		\tilde{\epsilon}_R &=\max_{s\in\gS,a\in\gA,r\in\gR}\left|R^{\kappa_{i^*}}(r|s,a)-R^{\kappa_{test}}(r|s,a)\right| \\
		&\le \left\|\kappa_{i^*}-\kappa_{test}\right\|_{\infty} \\ 
		&\le 2 \left(\frac{\log\left(\frac{1}{\delta}\right)}{\left|\gK_{train}\right|}\right)^{\frac{1}{\left|\gS\right|\left|\gA\right|\left(\left|\gS\right|+\left|\gR\right|\right)}},
	\end{align}
	then from Lemma \ref{lem:meta-task-rho-l1}, we have 
	\begin{align}
	&~H\left(\left\|\rho_{\pi}^{\kappa_{i*}}(s)-\rho_{\pi}^{\kappa_{test}}(s)\right\|_1 + \left\|\rho_{\pi}^{\kappa_{i*}}(s,a,r)-\rho_{\pi}^{\kappa_{test}}(s,a,r)\right\|_1\right) \\
	\le&~ 4H^2\left|\gS\right|^2\left|\gA\right|\left|\gR\right|\left(\frac{\log\left(\frac{1}{\delta}\right)}{\left|\gK_{train}\right|}\right)^{\frac{1}{\left|\gS\right|\left|\gA\right|\left(\left|\gS\right|+\left|\gR\right|\right)}}.
	\end{align}
	\paragraph{Overall} Combining Part I and Part II, we have 
	\begin{align}
		&~p_{out}^{\kappa_{test},\gD_{\kappa_{i^*},\mu_{i^*}}}(\pi) \\
		\le&~H\left(\left\|\rho_{\pi}^{M_{\gD_{\kappa_{i^*},\mu_{i^*}}}}(s)-\rho_{\pi}^{\kappa_{i^*}}(s)\right\|_1 + \left\|\rho_{\pi}^{M_{\gD_{\kappa_{i^*},\mu_{i^*}}}}(s,a,r)-\rho_{\pi}^{\kappa_{i^*}}(s,a,r)\right\|_1\right) + \\
		&~H\left(\left\|\rho_{\pi}^{\kappa_{i*}}(s)-\rho_{\pi}^{\kappa_{test}}(s)\right\|_1 + \left\|\rho_{\pi}^{\kappa_{i*}}(s,a,r)-\rho_{\pi}^{\kappa_{test}}(s,a,r)\right\|_1\right) \\
		\le&~2H^2\left|\gS\right|^2\left|\gA\right|\left|\gR\right|\left(\sqrt{\frac{\log\left(\frac{1}{\delta}\right)+\log\left(2\left|\gS\right|^2\left|\gA\right|\right)}{K_{\kappa_{i^*},\mu_{i^*}}\cdot d_\mu^{M_{\gD_{\kappa_{i^*},\mu_{i^*}}}}}}+2\left(\frac{\log\left(\frac{1}{\delta}\right)}{\left|\gK_{train}\right|}\right)^{\frac{1}{\left|\gS\right|\left|\gA\right|\left(\left|\gS\right|+\left|\gR\right|\right)}}\right).
	\end{align}
\end{proof}

\subsubsection{Detailed Lemmas (Part I)}

\begin{lemma}\label{lem:offline-rho-l1}
	In an MDP $M$ with an offline dataset $\gD$, for any batch-constrained policy $\pi\in\Pi^{\gD}$, if
	\begin{align}
	&\max_{s\in\gS,a\in\gA\text{ with } \rho_{\pi}^{M_{\gD}}(s,a)>0}\left\|P^{M_{\gD}}(\cdot|s,a)-P^{M}(\cdot|s,a)\right\|_1\le \epsilon_P, \\
	&\max_{s\in\gS,a\in\gA,r\in\gR\text{ with } \rho_{\pi}^{M_{\gD}}(s,a)>0}\left|R^{M_{\gD}}(r|s,a)-R^{M}(r|s,a)\right|\le \epsilon_R, 
	\end{align}
	we have
	\begin{align}
	&~\left\|\rho_{\pi}^{M}(s)-\rho_{\pi}^{M_{\gD}}(s)\right\|_1 + \left\|\rho_{\pi}^{M}(s,a,r)-\rho_{\pi}^{M_{\gD}}(s,a,r)\right\|_1 \\
	\le&~ \left|\gS\right|\left(\left|\gA\right|\left|\gR\right|+1\right)\frac{H-1}{2} \epsilon_P + \left|\gS\right|\left|\gA\right|\left|\gR\right|\epsilon_R.
	\end{align}
\end{lemma}
\begin{proof}
	For each $\hat{s}\in\gS, \hat{a}\in\gA$, create an auxiliary reward function $\tilde{r}(s,a):\gS\times\gA\rightarrow[0,1]$: $\forall s\in\gS,a\in\gA$,
	\begin{align}
	\tilde{r}(s,a) =\left\{
	\begin{aligned}
	&\frac{1}{H}, &\text{ if }s=\hat{s}\text{ and }a=\hat{a}, \\
	&0, &\text{ otherwise}.
	\end{aligned}
	\right.
	\end{align}
	Denote $\widehat{M}=\left(\gS, \gA, \gR, H, P^{M},\tilde{r} \right)$ and $\widehat{M}_{\gD}=\left(\gS, \gA, \gR, H, P^{M_{\gD}},\tilde{r} \right)$. Using the offline Simulation Lemma shown in Lemma \ref{lem:simulation-lemma-offline}, for any batch-constrained policy $\pi\in\Pi^{\gD}$,
	\begin{align}
	\rho_\pi^{M}(\hat{s},\hat{a}) = \gJ_{\widehat{M}}(\pi)\quad\text{and}\quad\rho_\pi^{M_{\gD}}(\hat{s},\hat{a}) = \gJ_{\widehat{M}_{\gD}}(\pi).
	\end{align}
	Thus, $\epsilon_r = 0, r_{max} = \frac{1}{H}$, and
	\begin{align}
	\left|\rho_\pi^{M}(\hat{s},\hat{a})-\rho_\pi^{M_{\gD}}(\hat{s},\hat{a})\right| &\le \left|\gJ_{\widehat{M}}(\pi)-\gJ_{\widehat{M}_{\gD}}(\pi)\right|\\
	&\le H\epsilon_r+ \frac{H(H-1)r_{max}}{2}\epsilon_P \\
	&=\frac{H-1}{2} \epsilon_P.
	\end{align}
	Similarly, we have $\forall s\in\gS$,
	\begin{align}
	\left|\rho_\pi^{M}(s)-\rho_\pi^{M_{\gD}}(s)\right| \le \frac{H-1}{2} \epsilon_P.
	\end{align}
	For any $s\in\gS, a\in\gA, r\in\gR$, 
	\begin{align}
	&~\left|\rho_\pi^{M}(s,a,r)-\rho_\pi^{M_{\gD}}(s,a,r)\right| \\
	=&~ \left|\rho_\pi^{M}(s,a)R^M(r|s,a)-\rho_\pi^{M_{\gD}}(s,a)R^{M_{\gD}}(r|s,a)\right| \\
	\le&~ \left|\rho_\pi^{M}(s,a)-\rho_\pi^{M_{\gD}}(s,a)\right|R^M(r|s,a)+\rho_\pi^{M_{\gD}}(s,a)\left|R^M(r|s,a)-R^{M_{\gD}}(r|s,a)\right| \\
	\le&~\left|\rho_\pi^{M}(s,a)-\rho_\pi^{M_{\gD}}(s,a)\right| + \left|R^M(r|s,a)-R^{M_{\gD}}(r|s,a)\right| \\
	\le&~\frac{H-1}{2} \epsilon_P + \epsilon_R.
	\end{align}
	Therefore,
	\begin{align}
	&~\left\|\rho_{\pi}^{M}(s)-\rho_{\pi}^{M_{\gD}}(s)\right\|_1 + \left\|\rho_{\pi}^{M}(s,a,r)-\rho_{\pi}^{M_{\gD}}(s,a,r)\right\|_1 \\
	\le&~ \left|\gS\right|\frac{H-1}{2} \epsilon_P + \left|\gS\right|\left|\gA\right|\left|\gR\right|\left(\frac{H-1}{2} \epsilon_P + \epsilon_R\right) \\
	=&~\left|\gS\right|\left(\left|\gA\right|\left|\gR\right|+1\right)\frac{H-1}{2} \epsilon_P + \left|\gS\right|\left|\gA\right|\left|\gR\right|\epsilon_R.
	\end{align}
\end{proof}

\subsubsection{Detailed Lemmas (Part II)}

\begin{lemma}[Simulation Lemma for Finite-Horizon MDPs]\label{lem:simulation-lemma}
	Given a pair of finite-horizon MDPs $M_1$ and $M_2$ with the same state space $\gS$, same action space $\gA$, same reward function $\gR$, and same horizon $H$. If
	\begin{align}
	&\max_{s\in\gS,a\in\gA}\left\|P^{M_1}(\cdot|s,a)-P^{M_2}(\cdot|s,a)\right\|_1\le \tilde{\epsilon}_P, \\
	&\max_{s\in\gS,a\in\gA}\left|r^{M_1}(s,a)-r^{M_2}(s,a)\right|\le \tilde{\epsilon}_r, \\
	&\max_{s\in\gS,a\in\gA}\max\left(r^{M_1}(s,a), r^{M_2}(s,a)\right)\le r_{max}, 
	\end{align}
	where $r^{M_1}(s,a)=\E_{\tilde{r}\sim R^{M_1}(s,a)}[\tilde{r}]$ and $r^{M_2}(s,a)=\E_{\tilde{r}\sim R^{M_2}(s,a)}[\tilde{r}]$, we have
	\begin{align}
	\left|\gJ_{M_1}(\pi)-\gJ_{M_2}(\pi)\right|\le H\tilde{\epsilon}_r+ \frac{H(H-1)r_{max}}{2}\tilde{\epsilon}_P.
	\end{align}
\end{lemma}
\begin{proof}
	Similar with famous Simulation Lemma in finite-horizon MDPs \citep{FiniteSimulationLemma} and the offline variant shown in Lemma \ref{lem:simulation-lemma-offline}, we will prove $\forall h\in[H], \forall s\in\gS_h$, 
	\begin{align}
	\left|V_\pi^{M_1}(s)-V_\pi^{M_2}(s)\right| \le (H-h)\tilde{\epsilon}_r+\frac{(H-h)(H-h-1)r_{max}}{2}\tilde{\epsilon}_P
	\end{align}
	by induction. When $h=H-1$, we have $\forall s\in\gS_h$, $\left|V_\pi^{M_1}(s)-V_\pi^{M_2}(s)\right|\le \tilde{\epsilon}_r$
	holds. And $\forall h\in[H-1]$, $\forall s\in\gS_h$,
	\begin{align}
	&~\left|V_\pi^{M_1}(s)-V_\pi^{M_2}(s)\right| \\
	\le&~\sum_{a\in\gA}\pi(a|s)\left|r^{M_1}(s,a)-r^{M_2}(s,a)\right| + \\
	&~\sum_{a\in\gA}\pi(a|s)\sum_{s'\in\gS_{h+1}}\left|P^{M_1}(s'|s,a)V_\pi^{M_1}(s')-P^{M_2}(s'|s,a)V_\pi^{M_2}(s')\right| \\
	\le&~(H-h)\tilde{\epsilon}_r+\frac{(H-h)(H-h-1)r_{max}}{2}\tilde{\epsilon}_P.
	\end{align}
	Thus, 
	\begin{align}
	\left|\gJ_{M_1}(\pi)-\gJ_{M_2}(\pi)\right|&=\left|V_\pi^{M_1}\left(s_0\right)-V_\pi^{M_2}\left(s_0\right)\right| \\
	&\le H\tilde{\epsilon}_r+ \frac{H(H-1)r_{max}}{2}\tilde{\epsilon}_P.
	\end{align}
\end{proof}

\begin{lemma}\label{lem:meta-task-rho-l1}
	Given a pair of finite-horizon MDPs $M_1$ and $M_2$ with the same state space, same action space, same reward function, and same horizon. If
	\begin{align}
	&\max_{s\in\gS,a\in\gA}\left\|P^{M_1}(\cdot|s,a)-P^{M_2}(\cdot|s,a)\right\|_1\le \tilde{\epsilon}_P, \\
	&\max_{s\in\gS,a\in\gA,r\in\gR}\left|R^{M_1}(r|s,a)-R^{M_2}(r|s,a)\right|\le \tilde{\epsilon}_R, 
	\end{align}
	we have
	\begin{align}
	&~\left\|\rho_{\pi}^{M_1}(s)-\rho_{\pi}^{M_2}(s)\right\|_1 + \left\|\rho_{\pi}^{M_1}(s,a,r)-\rho_{\pi}^{M_2}(s,a,r)\right\|_1 \\
	\le&~ \left|\gS\right|\left(\left|\gA\right|\left|\gR\right|+1\right)\frac{H-1}{2} \tilde{\epsilon}_P + \left|\gS\right|\left|\gA\right|\left|\gR\right|\tilde{\epsilon}_R.
	\end{align}
\end{lemma}
\begin{proof}
	Similar with Lemma \ref{lem:offline-rho-l1}, $\forall s\in\gS, a\in\gA, r\in\gR$, using Simulation Lemma \ref{lem:simulation-lemma}, we have 
	\begin{align}
		\left|\rho_\pi^{M_1}(s)-\rho_\pi^{M_2}(s)\right| \le \frac{H-1}{2} \tilde{\epsilon}_P\quad\text{and}\quad\left|\rho_\pi^{M_1}(s,a)-\rho_\pi^{M_2}(s,a)\right| \le \frac{H-1}{2} \tilde{\epsilon}_P,
	\end{align}
	and
	\begin{align}
		\left|\rho_\pi^{M}(s,a,r)-\rho_\pi^{M_{\gD}}(s,a,r)\right| \le \frac{H-1}{2} \tilde{\epsilon_P} + \tilde{\epsilon}_R.
	\end{align}
	Therefore,
	\begin{align}
	&~\left\|\rho_{\pi}^{M_1}(s)-\rho_{\pi}^{M_2}(s)\right\|_1 + \left\|\rho_{\pi}^{M_1}(s,a,r)-\rho_{\pi}^{M_2}(s,a,r)\right\|_1 \\
	\le&~\left|\gS\right|\left(\left|\gA\right|\left|\gR\right|+1\right)\frac{H-1}{2} \tilde{\epsilon}_P + \left|\gS\right|\left|\gA\right|\left|\gR\right|\tilde{\epsilon}_R.
	\end{align}
\end{proof}

\begin{lemma}\label{lem:xy-distance-bound}
	Let $X, Y$ be two i.i.d. random vectors that take values in $[0,1]^n$, $n\in\mathbb{N^+}$. For any $\epsilon\in(0, 1]$, we have
	\begin{align}
		\sP\left[\max_{i\in[n]}\left|X_i-Y_i\right|\geq\epsilon\right]\leq1-\left(\frac{\epsilon}{2}\right)^n.
	\end{align}
\end{lemma}
\begin{proof}
	Denote an auxiliary set
	\begin{align}
		V=\left\{x\in\mathbb R^n\middle|\max_{i\in[n]}\left|x_i\right|<\frac{\epsilon}{2}\right\},
	\end{align}
	then if $X,Y\in V$, we must have
	\begin{align}
		\max_{i\in[n]}\left|X_i-Y_i\right|<\epsilon.
	\end{align}
	For any $c\in \sN^n$, denote
	\begin{align}
		V^c=V+v^c,\quad\text{where}\quad v^c_i=\left(c_i+\frac{1}{2}\right)\epsilon,\quad \forall i\in[n].
	\end{align}
	We may construct a set of such cosets of $V$ as follows:
	\begin{align}
		S=\left\{V^c|c\in C\right\},\quad\text{where}\quad
		C=\left\{c\in\mathbb{N}^n\middle|c_i\in\left[ \left\lceil\frac{1}{\epsilon}\right\rceil\right]\right\}.
	\end{align}
	There are several properties related to these constructions:
	\begin{itemize}
		\item For any $c\in\mathbb N^n$, if $X,Y\in V^c$, $\max_{i\in[n]}|X_i-Y_i|<\epsilon$.
		\item The union of sets in $S$ contains $[0,1]^n$
		\item Any two different sets in $S$ are disjoint.
	\end{itemize}
	The only loophole is that we have not considered points in boundaries $\partial V^c\ (V^c\in S)$. These boundaries can be decomposed into disjoint union of hyperplanes in $\mathbb R^n$. For each one of the hyperplanes, arbitrarily designate it to an adjacent $V^c\in S$. New $V^c$s are the union of the original one and the hyperplanes designated to it. Note that 
	\begin{align}
		\sum_{c\in C}\left[X\in V^c\right]=1.
	\end{align}
	Therefore,
	\begin{align}
		\sP\left[\max_{i\in[n]}|X_i-Y_i|\ge\epsilon\right]&\le 1-\sum_{c\in C}\sP\left[X\in V^c\right]\sP\left[Y\in V^c\right] \\
		&= 1-\sum_{c\in C}\sP\left[X\in V^c\right]^2\\
		&\le 1-\frac{1}{|C|}\left(\sum_{c\in C}\sP\left[X\in V^c\right]\right)^2 \\
		&= 1-\frac{1}{|C|}.
	\end{align}
	Since $\frac{1}{\epsilon}\ge 1$, we have
	\begin{align}
		|C|&=\left\lceil\frac{1}{\epsilon}\right\rceil^n<\left(1+\frac{1}{\epsilon}\right)^n\le\left(\frac{2}{\epsilon}\right)^n\quad\text{and}\quad
		\sP\left[\max_{i\in[n]}\left|X_i-Y_i\right|\geq\epsilon\right]\leq1-\left(\frac{\epsilon}{2}\right)^n.
	\end{align}
\end{proof}

\begin{lemma}\label{lem:meta-task-dis-bound}
	In a transformed BAMDP $\overline{M}^{+}$ with an offline multi-task dataset $\overline{\gD}^+$, for any meta-testing task $\kappa_{test}\sim p(\kappa)$, $\forall \delta \in (0, 1]$, with probability $1-\delta$, we have 
	\begin{align}
		&~\left\|\kappa_{i^*}-\kappa_{test}\right\|_{\infty} \\
		=&~\max\left(\left\|P^{\kappa_{i^*}}(s,a,s') - P^{\kappa_{test}}(s,a,s')\right\|_{\infty}, \left\|R^{\kappa_{i^*}}(s,a,r) - R^{\kappa_{test}}(s,a,r)\right\|_{\infty}\right) \\
		\le&~ 2\left(\frac{\log\left(\frac{1}{\delta}\right)}{\left|\gK_{train}\right|}\right)^{\frac{1}{\left|\gS\right|\left|\gA\right|\left(\left|\gS\right|+\left|\gR\right|\right)}},
	\end{align}
	where $\kappa_{i^*}\in\gK_{train}$ is the closest offline meta-training task to $\kappa_{test}$ (see Eq. (\ref{appendix-eq:closest-meta-train-task})), $\left\|\kappa_{i^*}-\kappa_{test}\right\|_{\infty}$ is the distance between $\kappa_{i^*}$ and $\kappa_{test}$ (see Eq. (\ref{appendix-eq:task-distance-linfty})), and $\gK_{train}$ is the i.i.d. offline meta-training tasks sampled from $p(\kappa)$ in $\overline{\gD}^+$.
\end{lemma}
\begin{proof}
	From Lemma \ref{lem:xy-distance-bound}, we set $n=\left|\gS\right|\left|\gA\right|\left(\left|\gS\right|+\left|\gR\right|\right)$, then $\forall\epsilon\in(0,1], \forall\kappa_i\in\gK_{train}$,
	\begin{align}
		\sP\left[\left\|\kappa_i-\kappa_{test}\right\|_{\infty}\ge\epsilon\right]\le 1-\left(\frac{\epsilon}{2}\right)^{n}.
	\end{align}
	Hence
	\begin{align}
		\sP\left[\mathop{\arg\min}_{\kappa_i\in\gK_{train}}\left\|\kappa_i-\kappa_{test}\right\|_{\infty}\ge\epsilon\right]&=\prod_{\kappa_i\in\gK_{train}}\sP\left[\left\|\kappa_i-\kappa_{test}\right\|_{\infty} \ge \epsilon\right] \\
		&\le \left(1-\left(\frac{\epsilon}{2}\right)^{n}\right)^{\left|\gK_{train}\right|}.
	\end{align}
	Therefore, $\forall \delta \in (0, 1]$, with probability $1-\delta$,
	\begin{align}
		\mathop{\arg\min}_{\kappa_i\in\gK_{train}}\left\|\kappa_i-\kappa_{test}\right\|_{\infty} \le 2\left(\frac{\log\left(\frac{1}{\delta}\right)}{\left|\gK_{train}\right|}\right)^{\frac{1}{\left|\gS\right|\left|\gA\right|\left(\left|\gS\right|+\left|\gR\right|\right)}}.
	\end{align}
\end{proof}

\newpage
\section{Formal loss function of an Ensemble of Reward and Dynamics Models}\label{appendix-sec:mse-loss}

In Section \ref{quant}, we will optimize an ensemble of $L$ reward and dynamics models $\left\{r_{\phi_i}(s,a,z), p_{\psi_i}(s,a,z)\right\}_{i=1}^{L}$ on the multi-task dataset $\left\{\gD_z\right\}$ with an latent task embedding $z$. We will minimize the MSE loss function
\begin{align}
L_M(\phi,\psi,z) =&~ \mathop{\E}_{(s,a,r,s')\sim\gD_z}\left[\frac{1}{L}\sum_{i=1}^L\left(\left(r-r_{\phi}^i(s,a,z)\right)^2  + \left\|s'-p_{\psi}^i(s,a,z)\right\|_2^2\right)\right]
\end{align}
to train the parameters $\{\phi_i,\psi_i\}_{i=1}^{L}$ with the given latent task embedding $z$.

\section{Hyper-Parameter Settings}
\label{exp-app1}
\textbf{Environment Settings.} Table \ref{tab:env-hyper} shows hyper-parameter settings for the task sets used in our experiments. Most hyper-parameters are adopted from previous works \citep{li2020focal,mitchell2021offline}. For all task sets, 80\% of the tasks are meta-training tasks, and the remaining 20\% tasks are meta-testing tasks.

\begin{table}[H]
	\centering
	\caption{Environment parameter settings. }
	\begin{tabular}{l|c|c|c|c}
		\toprule
		Environment & Episode Length & \makecell[c]{\# of Adaptation\\Episodes} & \# of Tasks & \makecell[c]{\# of Trajectories\\per Task}  \\
		\midrule
		All Meta-World Envs&  500 &  10 &  50 & 45   \\
        Cheetah-Vel & 200 & 10 & 100  & 45\\
		Point-Robot & 20 & 20 & 100 & 45\\
		Point-Robot-Sparse & 20 & 20 & 100 & 45\\
		\bottomrule
	\end{tabular}
	\label{tab:env-hyper}
\end{table}

\textbf{IDAQ hyper-parameter settings.} Table \ref{sample-table} shows IDAQ's hyper-parameter settings. Most hyper-parameters are adopted from FOCAL \citep{li2020focal}. We set $n_{e}$ to 1 as the evaluation environments are all nearly deterministic.

\begin{table}[H]
	\centering
\caption{Detailed hyper-parameter settings for IDAQ.}
  \label{sample-table}
\begin{tabular}{c|c}
		\toprule
    {Hyper-Parameter}                                                                                          &{Hyper-Parameter Values}           \\   
		\midrule 
{batch size }    & {256 }                                                                                                                                  \\
{meta batch size}                                                                                      & {16}                                                                                                             \\
{learning rates for  dual critic }                                                                                      & {1e-4}                                                                                                             \\
{learning rates for  all other components }                                                                                      & {3e-4}                                                                                                             \\
{network structure for all components}                                                                     & { three fully connected layers with 200 units} \\

{optimizer}                                                                                          & {adam}                                                                                                             \\

{discount}                                                                                           & {0.99}                                                                                                             \\
{latent size}                                                                                           & {20}                                                                                                             \\
{reward scale}                                                                                       & {100 for point envs, 1 for all other envs}          \\
{$n_r$}                                                                                       & 1/2 of total adaptation episodes          \\
{$n_i$}                                                                                       & 1/2 of total adaptation episodes  \\
{$k$}                                                                                       & {10 for point envs, 20 for all other envs}  \\
{$L$}                                                                                       & {4}  \\
{$n_{e}$}                                                                                       &1    \\
		\bottomrule
\end{tabular}

\end{table}

\newpage

\section{Didactic Example on Distribution Shift}\label{sec:didactic-example}

To empirically demonstrate the distributional shift problem proposed in Section \ref{sec:theory}, we introduce Point-Robot, a simple 2D navigation task set commonly used in meta-RL \citep{rakelly2019efficient,zhang2021metacure}. Figure \ref{fig:dataset} illustrates the distribution mismatch between offline meta-training and online adaptation, as the dataset is collected by task-dependent behavior policies. As a consequence, directly performing adaptation with the online collected trajectories leads to poor adaptation performance (see FOCAL in Figure \ref{fig:example-performance}). Figure \ref{fig:uncertainty} shows that the \textbf{Prediction Variance} quantification cannot correctly detect OOD data, as there is a large error in uncertainty estimation. The \textbf{Return-based} quantification fixes this problem by greedily selecting trajectories. As shown in Figure \ref{fig:dataset}, at the end of the reference stage, the \textbf{Return-based} quantification selects the red trajectory as it has the highest return. After the reference stage, IDAQ iteratively optimizes the posterior belief to get the final policy.  As shown in Figure \ref{fig:example-performance}, IDAQ+Return achieves comparable performance to FOCAL with expert context and significantly outperforms FOCAL with online adaptation. 

To further investigate why the \textbf{Prediction Variance} quantification fails to identify in-distribution trajectories, we demonstrate the ensemble's uncertainty estimation on the first 10 trajectories collected in the reference stage. As shown in Figure \ref{uncer1}, the \textbf{Prediction Variance} quantification cannot accurately estimate the distance to the offline dataset, and fail to identify in-distribution data. On the other hand, the \textbf{Return-based} quantification can successfully select in-distribution data with its greedy selection mechanism.

Figure \ref{vvv} shows IDAQ+Return, FOCAL, and IDAQ+Prediction Variance's adaptation visualization (Episode 11-20) after the reference phase in adaptation (Episode 1-10). Results demonstrate that while the \textbf{Return-based} quantification is able to identify in-distribution data and achieve superior adaptation performance, FOCAL utilizes all the 10 trajectories for adaptation, and cannot correctly update task belief due to the distributional shift problem. The \textbf{Prediction Variance} quantification, as discussed above, fails to correctly select the in-distribution trajectory, and thus cannot successfully reach the goal.

\begin{figure*}[h]
	\centering
\subfigure[IDAQ+Return]{\includegraphics[width=0.3\linewidth]{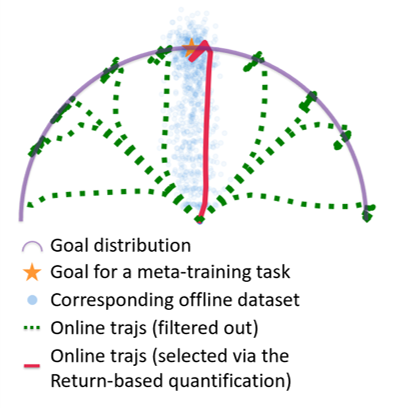}\label{fig:dataset}}
\subfigure[IDAQ+Prediction Variance]
{\includegraphics[width=0.3\linewidth]{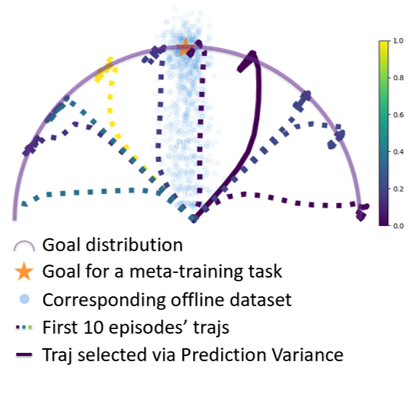}\label{fig:uncertainty}} 
\hspace{0.05in}
\subfigure[Learning Curve]{\includegraphics[width=0.35\linewidth]{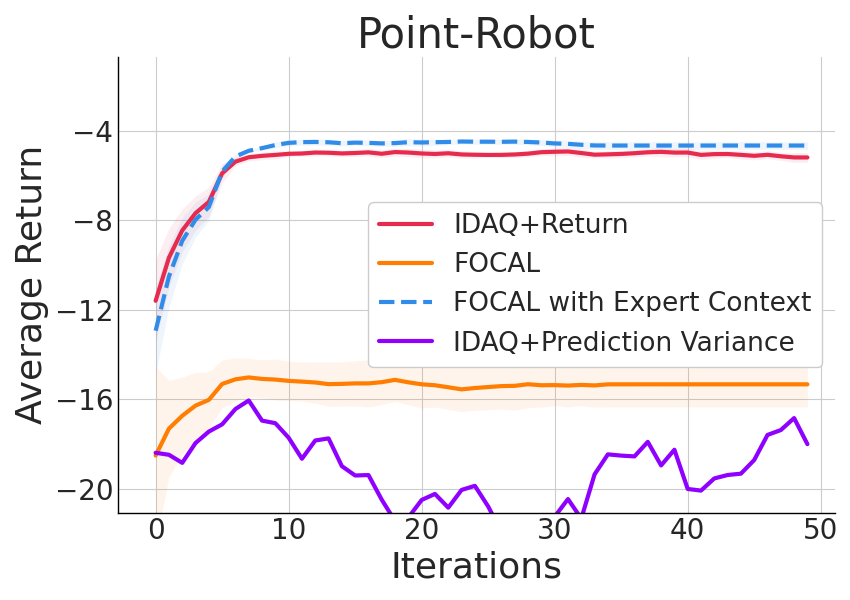}\label{fig:example-performance}} 
	\caption{(a) Illustration of data distribution mismatch between offline meta-training (blue) and IDAQ+\textbf{Return}'s online adaptation (green and red trajectories). (b) The agent fails to identify in-distribution trajectories via the \textbf{Prediction Variance} quantification. Trajectories are colored with corresponding normalized uncertainty estimation. (c) Adaptation performance of IDAQ+Return, FOCAL, FOCAL with expert context, and IDAQ+Prediction Variance.} 
\end{figure*}

\newpage

\begin{figure}[H]
	\centering
\subfigure[IDAQ+Return]{\includegraphics[width=\linewidth]{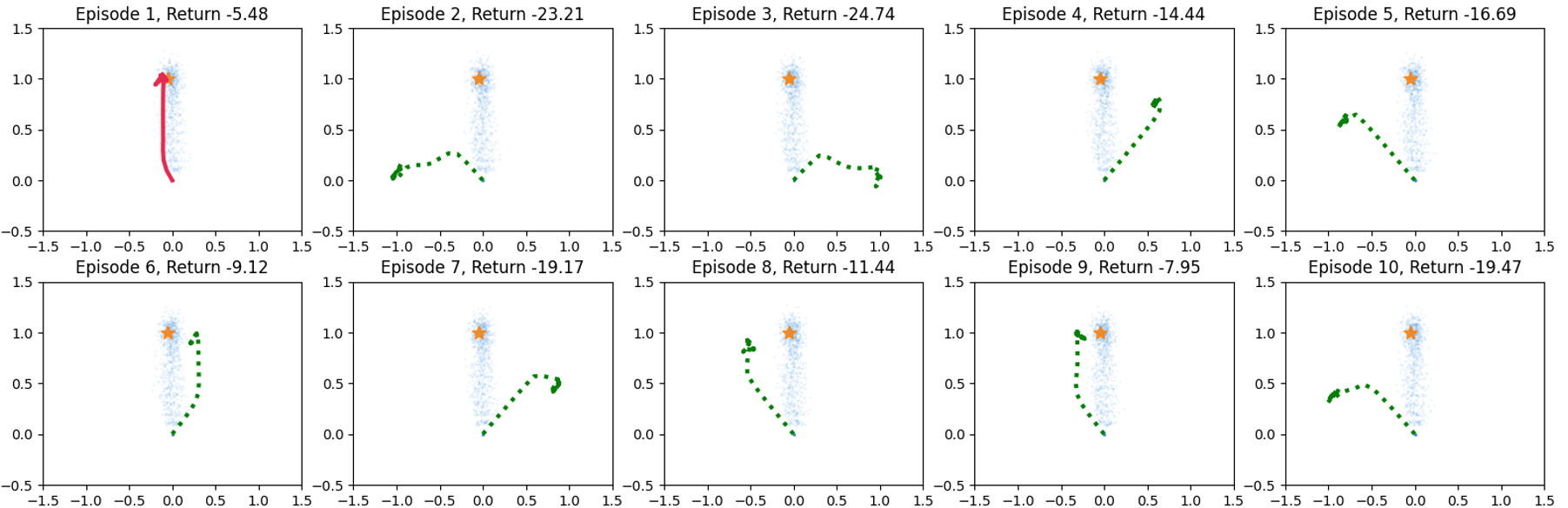}\label{v11}}
\subfigure[IDAQ+ Prediction Variance]
{\includegraphics[width=\linewidth]{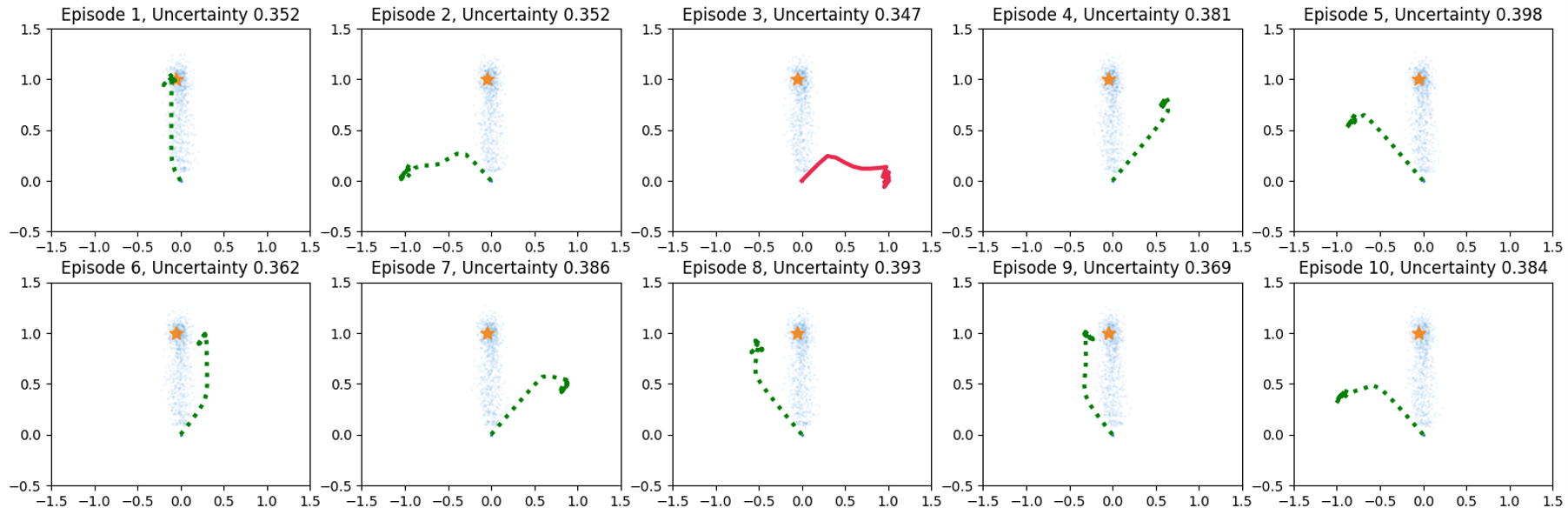}\label{v22}}
\subfigure{\includegraphics[width=\linewidth]{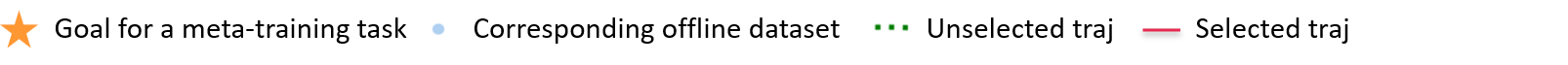}\label{v33}} 
	\caption{Visualization of the \textbf{Return-based} quantification and the \textbf{Prediction Variance} quantification's trajectory selection. (a) The \textbf{Return-based} quantification successfully selects the in-distribution trajectory. (b) The \textbf{Prediction Variance} quantification cannot identify in-distribution data, as its uncertainty estimation is not accurate.
 }
 \label{uncer1}
\end{figure}


\begin{figure}[H]
	\centering
\subfigure[IDAQ+Return]{\includegraphics[width=\linewidth]{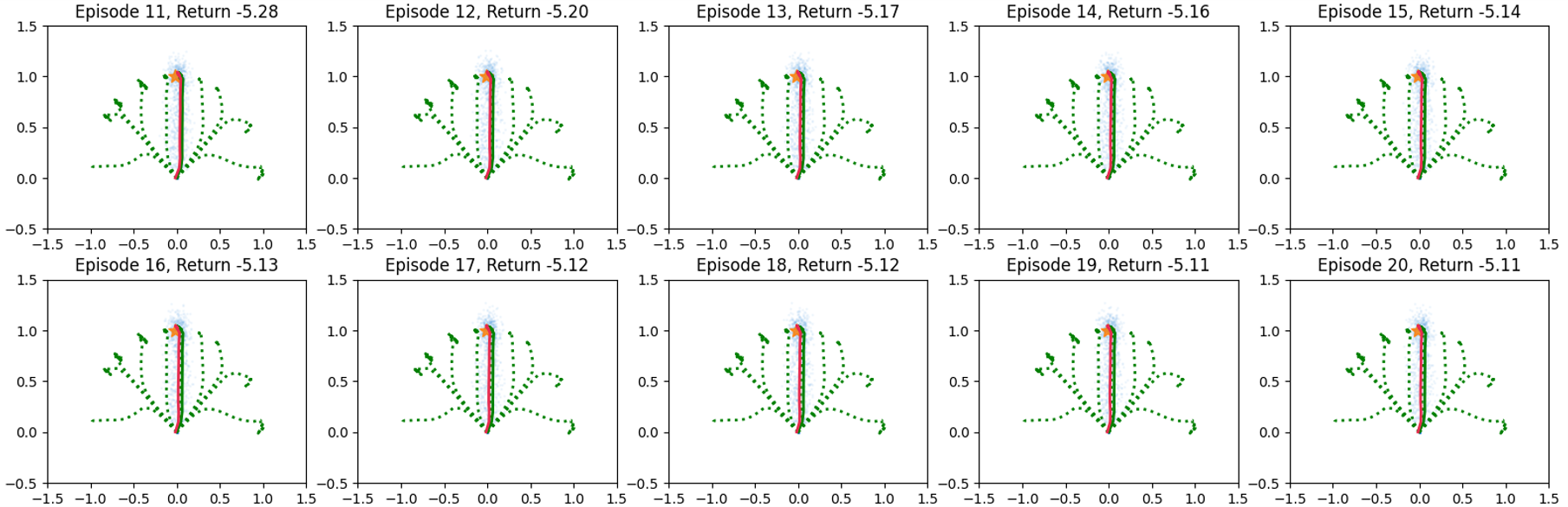}\label{v1}}
\subfigure[FOCAL]
{\includegraphics[width=\linewidth]{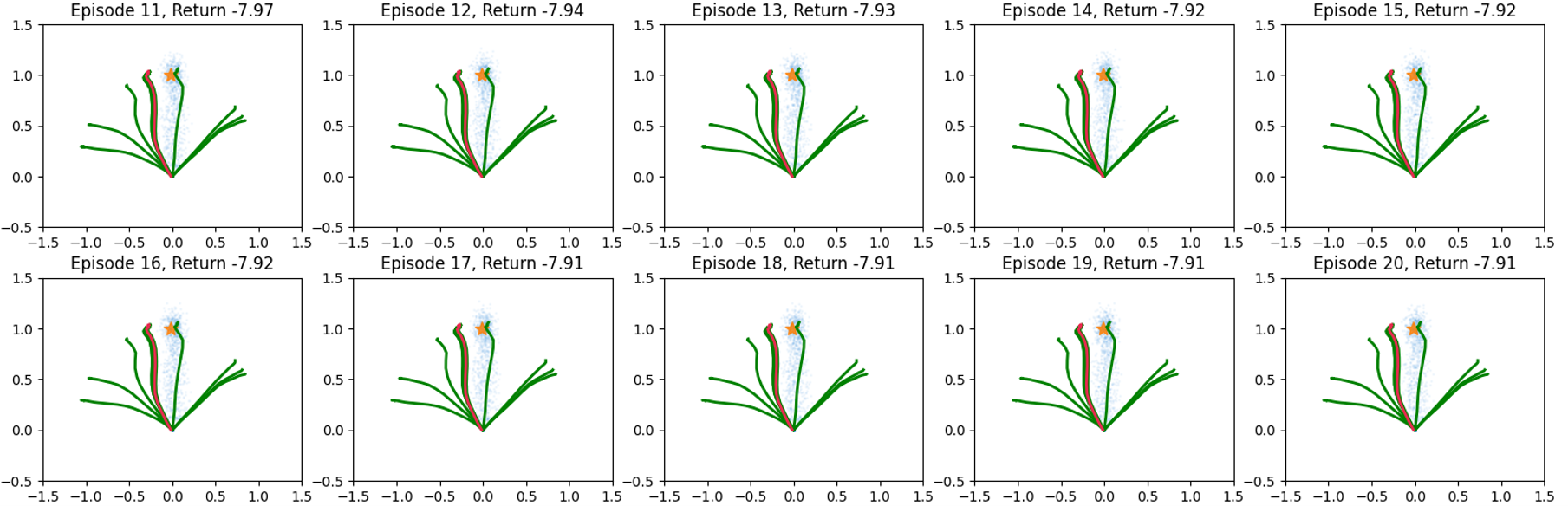}\label{v2}} 
\subfigure[IDAQ+Prediction Variance]
{\includegraphics[width=\linewidth]{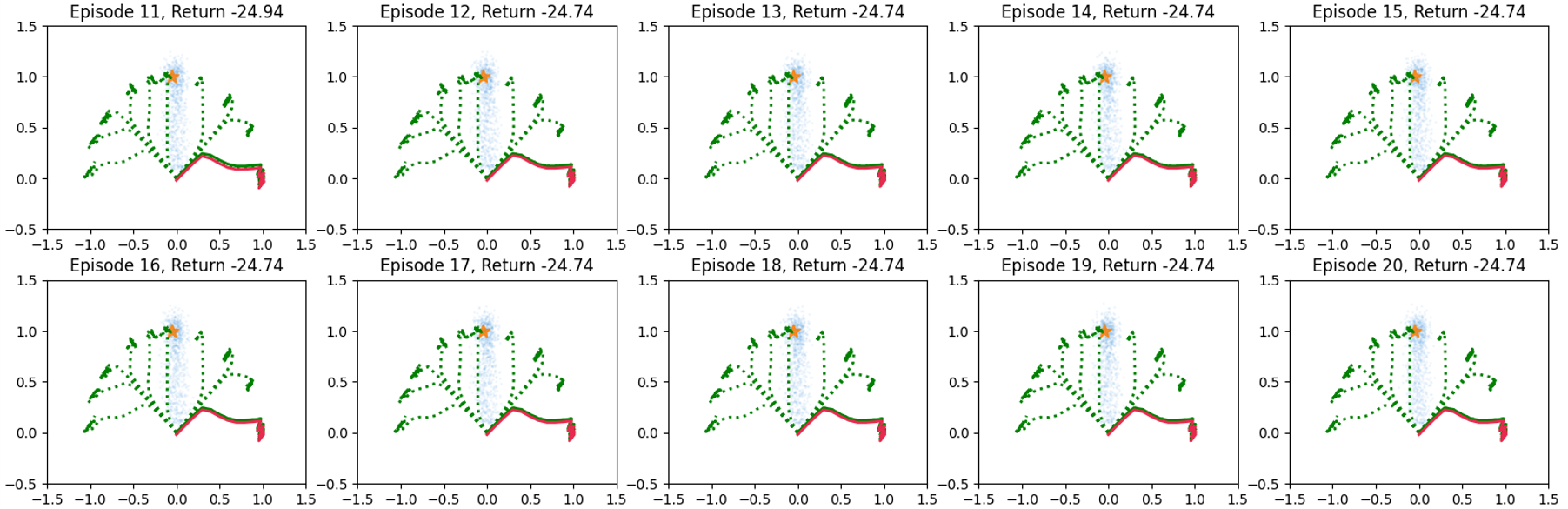}\label{v4}}
\subfigure{\includegraphics[width=\linewidth]{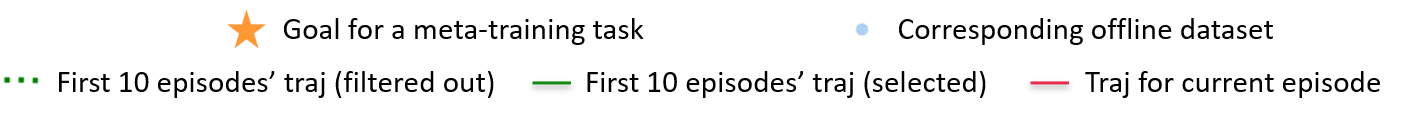}\label{v3}} 
	\caption{Visualization of IDAQ+Return, FOCAL and IDAQ+Prediction Variance's adaptation in Episode 11-20 on the Point-Robot environment. (a) The \textbf{Return-based} quantification successfully selects the in-distribution trajectory and keeps improving in adaptation. (b) FOCAL suffers from the out-of-distribution problem, and cannot correctly update posterior belief, leading to poor performance. (c) The \textbf{Prediction Variance} quantification method cannot identify in-distribution data, and also suffers from the distributional shift problem.
 }
 \label{vvv}
\end{figure}

\newpage

\section{Formal Definition of Minimal Distance Between Episode and Dataset}
\label{formal}
In Section \ref{equant} we demonstrate the minimal distance between episodes and datasets. To give a formal definition of this distance measure, we first transform an episode $\tau$ of length $H$ to a vector $v^{\tau} $ as:
\begin{equation}
    v^{\tau}=\langle s_0,a_0,r_0,s_1,a_1,r_1,...,s_{H-1},a_{H-1},r_{H-1}\rangle.
    \end{equation}
$v^{\tau}$ contains information about $\tau$'s reward and transition.
Then we define the normalized distance $d(\tau_1,\tau_2)$ between episode $\tau_1$ and episode $\tau_2$ as:
\begin{equation}
    d(\tau_1,\tau_2)=\frac{\left|v^{\tau_1}-v^{\tau_2}\right|_2}{\left|v^{\tau_2}\right|_2},
\end{equation}
where $|\cdot|_2$ is the L2 distance. $d(\tau_1,\tau_2)$ measures distance between $\tau_1$ and $\tau_2$ normalized by $v^{\tau_2}$. Finally, the minimal distance $d_{min}(\tau,B)$ between episode $\tau$ and dataset $\gD$ is defined as:
\begin{equation}
    d_{min}(\tau,\gD)=\min_{\tau' \in \gD}d(\tau,\tau').
\end{equation}
$d_{min}$ is calculated by finding the episode $\tau'$ in $\gD$ that has the minimal normalized distance to $\tau$.

\section{Additional Experiment Results}
\label{exp-app2}
\subsection{Additional Visualization Results}
\label{avr}
Figure \ref{ppee}, \ref{ppvv} and \ref{gg} show how the three quantifications behave on the tasks evaluated in Section \ref{equant}. Results demonstrate that the \textbf{Return-based} quantification achieves the best performance on identifying in-distribution episodes, which enables IDAQ to achieve superior adaptation performance. The \textbf{Prediction Error} quantification can identify in-distribution episodes in simple tasks like Reach-V2, Drawer-Close-V2, and Point-Robot, but fails in more complex tasks. The \textbf{Prediction Variance} quantification fails to identify in-distribution episodes in these eight tasks. These results correspond to the performance demonstrated in Table \ref{tab:51}.

\begin{figure}[H]
	\centering
\subfigure{\includegraphics[width=0.24\linewidth]{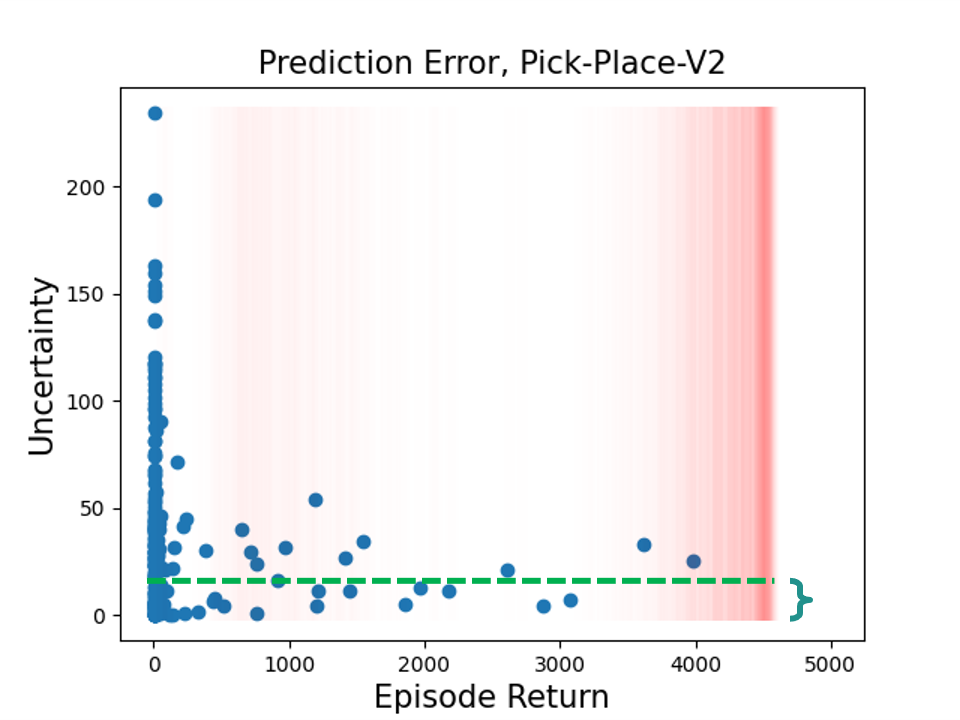}}
\subfigure{\includegraphics[width=0.24\linewidth]{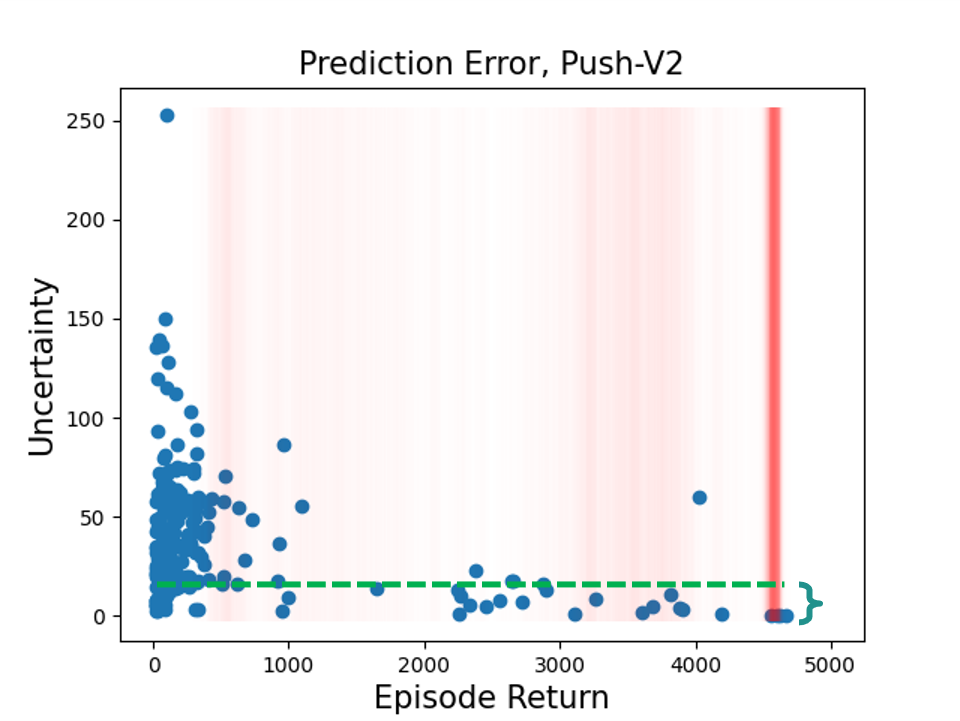}}
\subfigure{\includegraphics[width=0.24\linewidth]{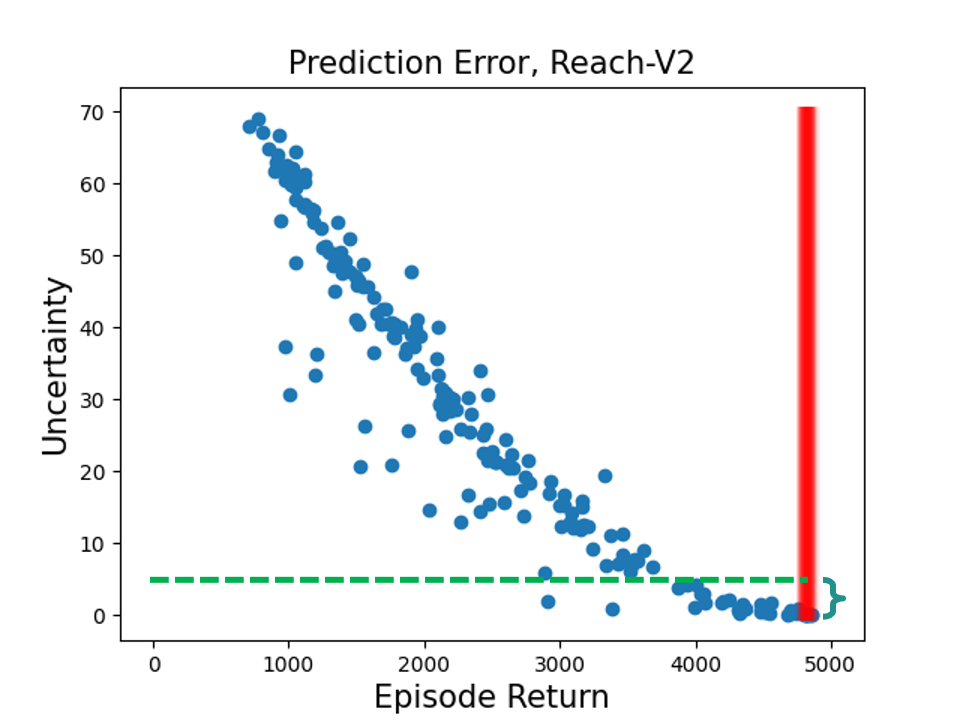}}
\subfigure{\includegraphics[width=0.24\linewidth]{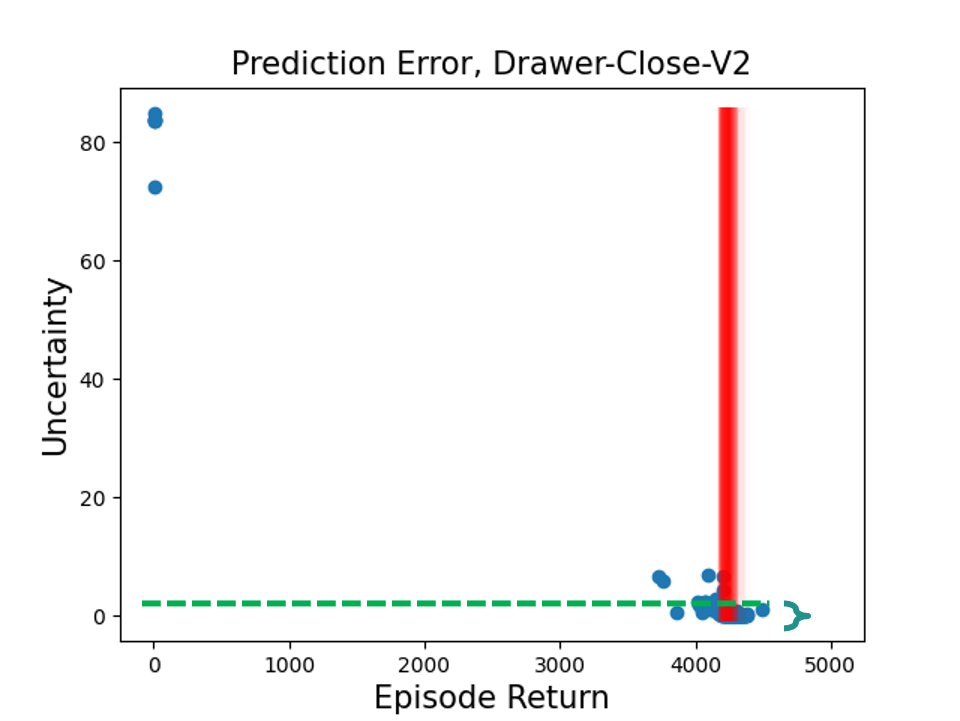}}
\subfigure{\includegraphics[width=0.24\linewidth]{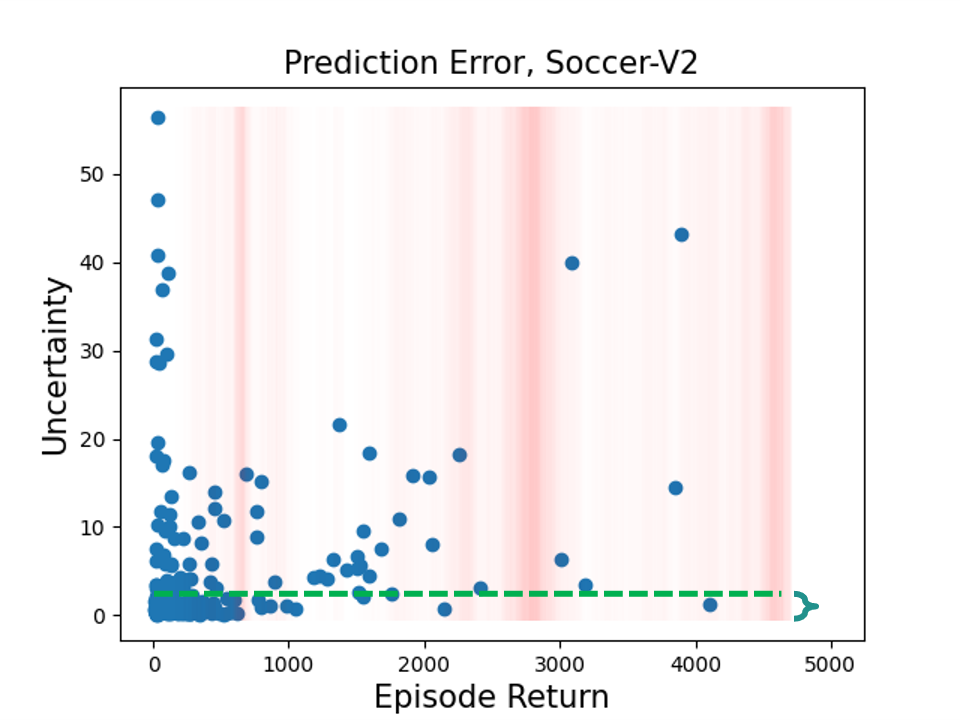}}
\subfigure{\includegraphics[width=0.24\linewidth]{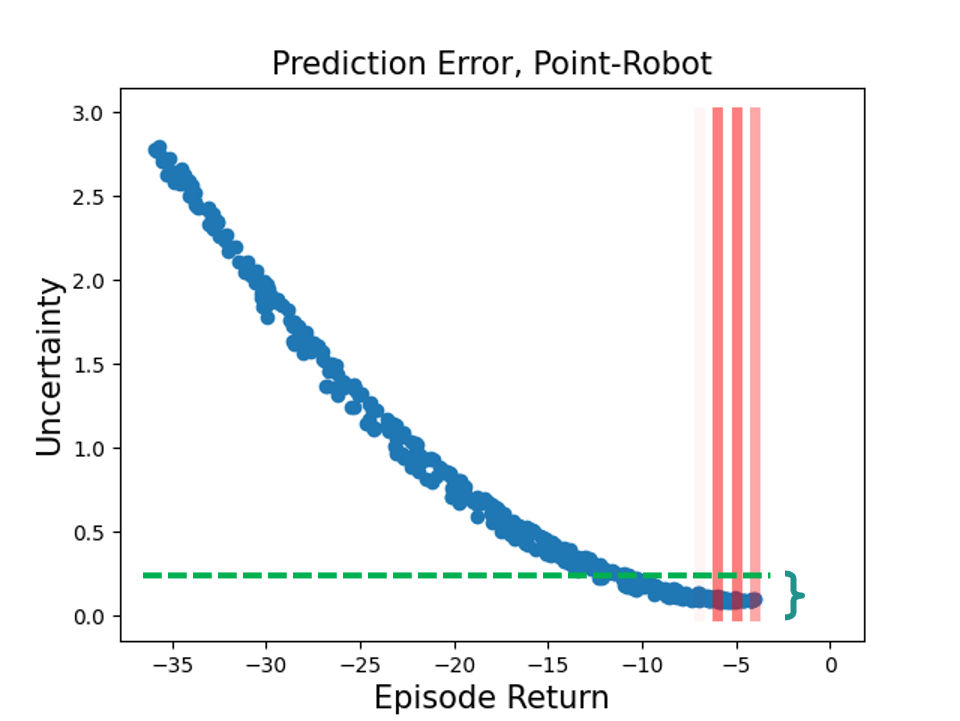}}
\subfigure{\includegraphics[width=0.24\linewidth]{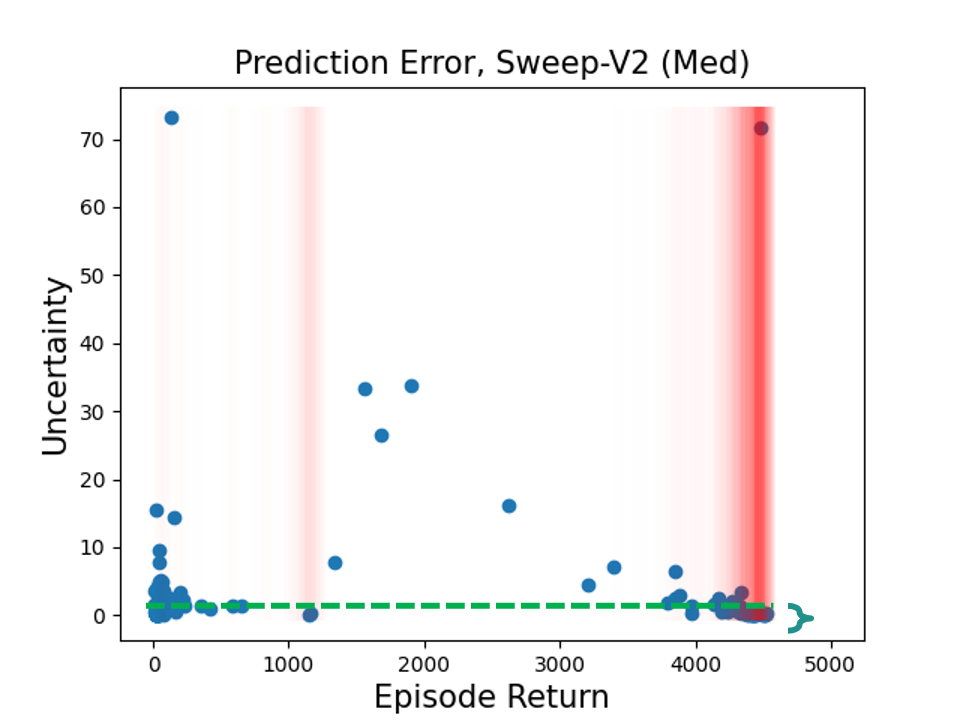}}
\subfigure{\includegraphics[width=0.24\linewidth]{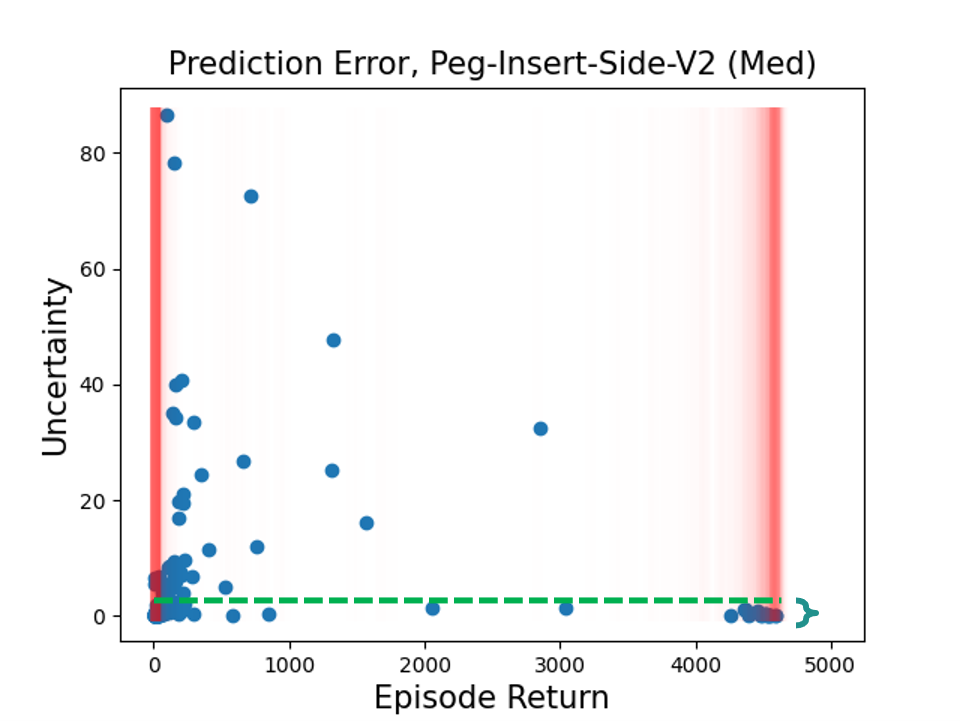}}
\subfigure{\includegraphics[width=0.7\linewidth]{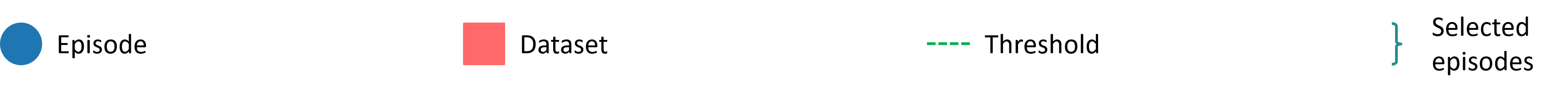}}
	\caption{Visualization of the \textbf{Prediction Error} quantification's behavior on various tasks. It cannot find a good reference threshold to distinguish in-distribution episodes.
 }
 \label{ppee}
\end{figure}

\begin{figure}[H]
	\centering
\subfigure{\includegraphics[width=0.24\linewidth]{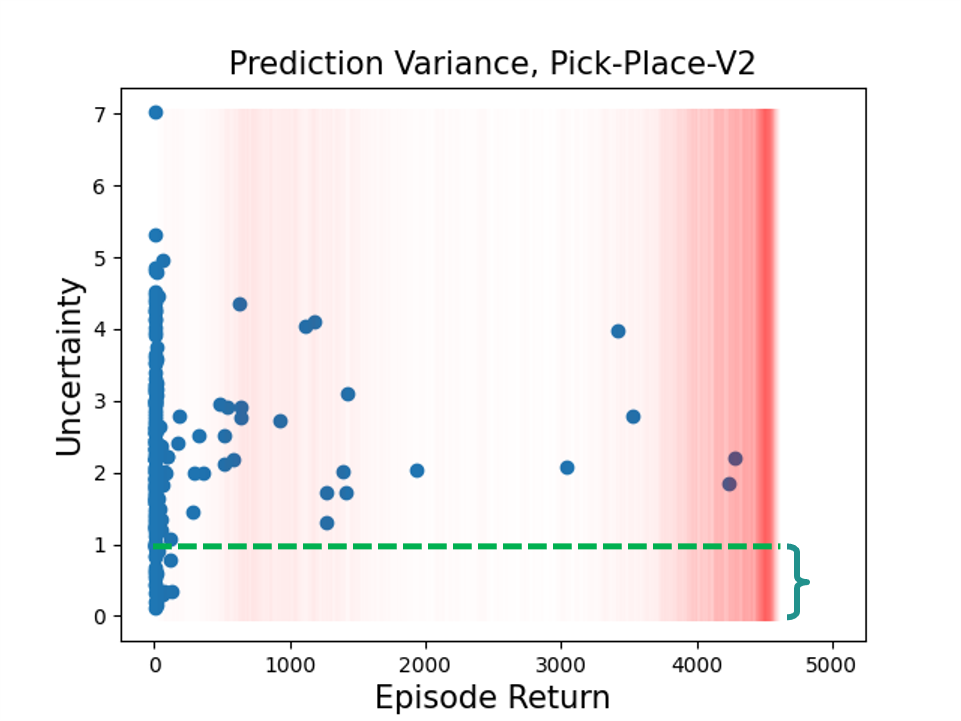}}
\subfigure{\includegraphics[width=0.24\linewidth]{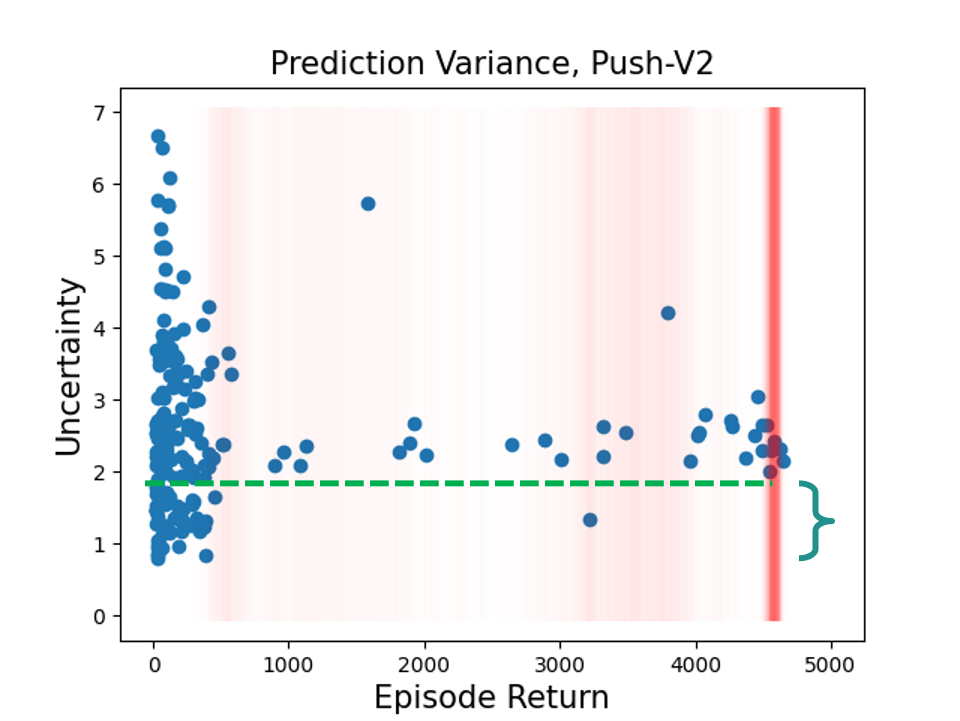}}
\subfigure{\includegraphics[width=0.24\linewidth]{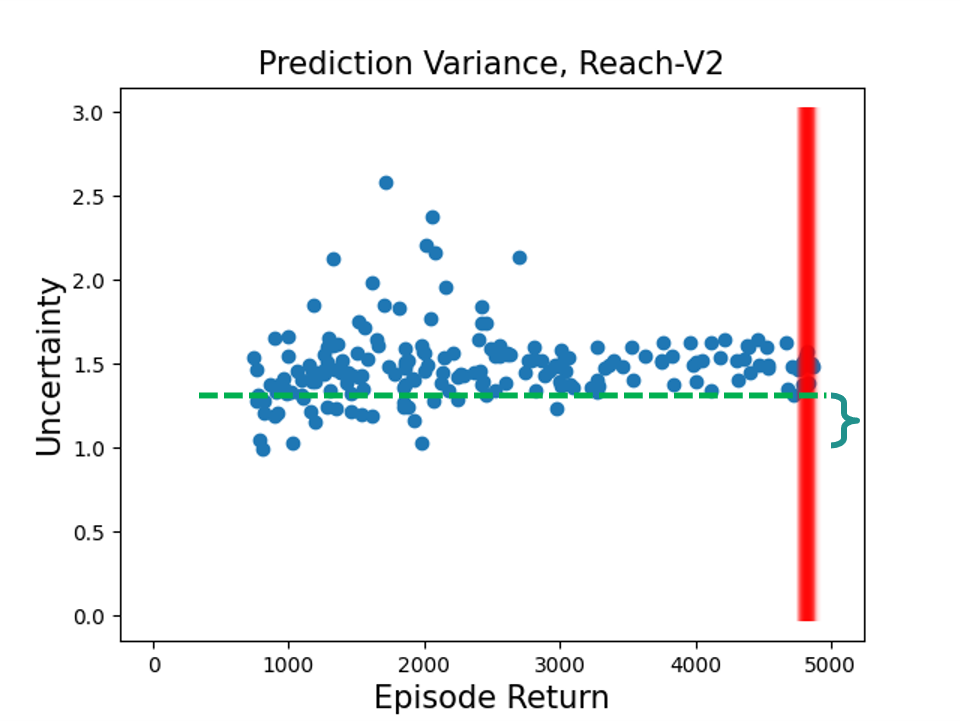}}
\subfigure{\includegraphics[width=0.24\linewidth]{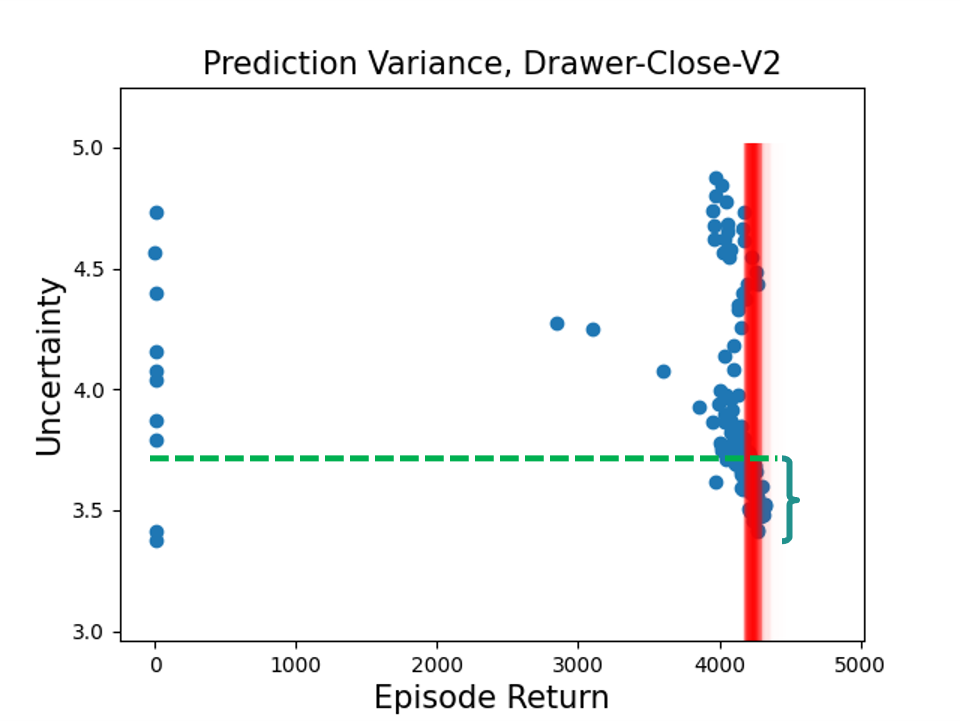}}
\subfigure{\includegraphics[width=0.24\linewidth]{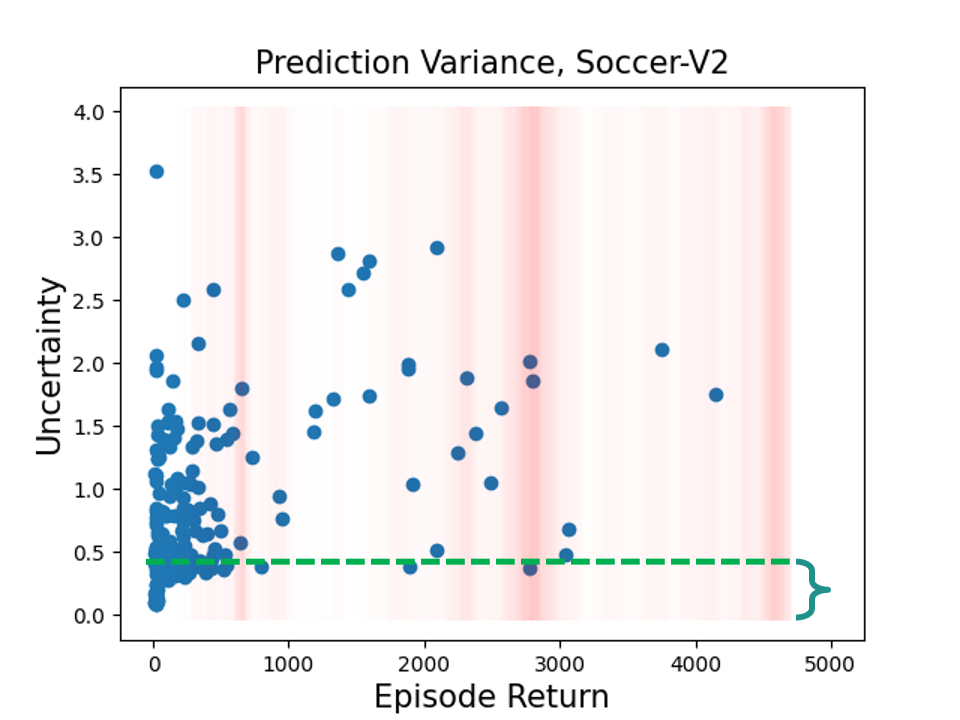}}
\subfigure{\includegraphics[width=0.24\linewidth]{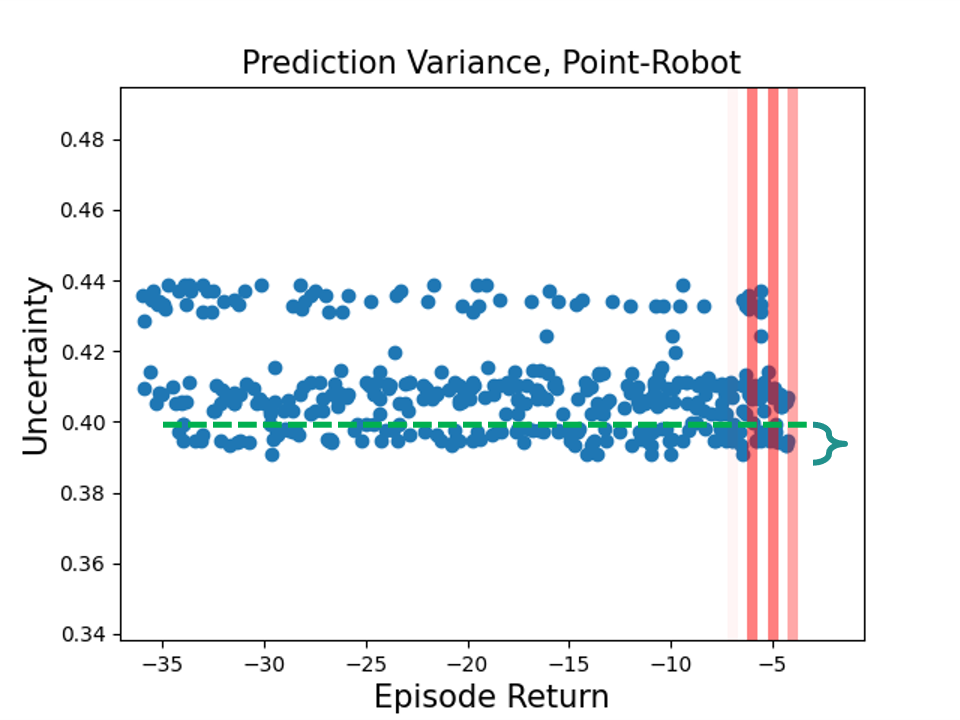}}
\subfigure{\includegraphics[width=0.24\linewidth]{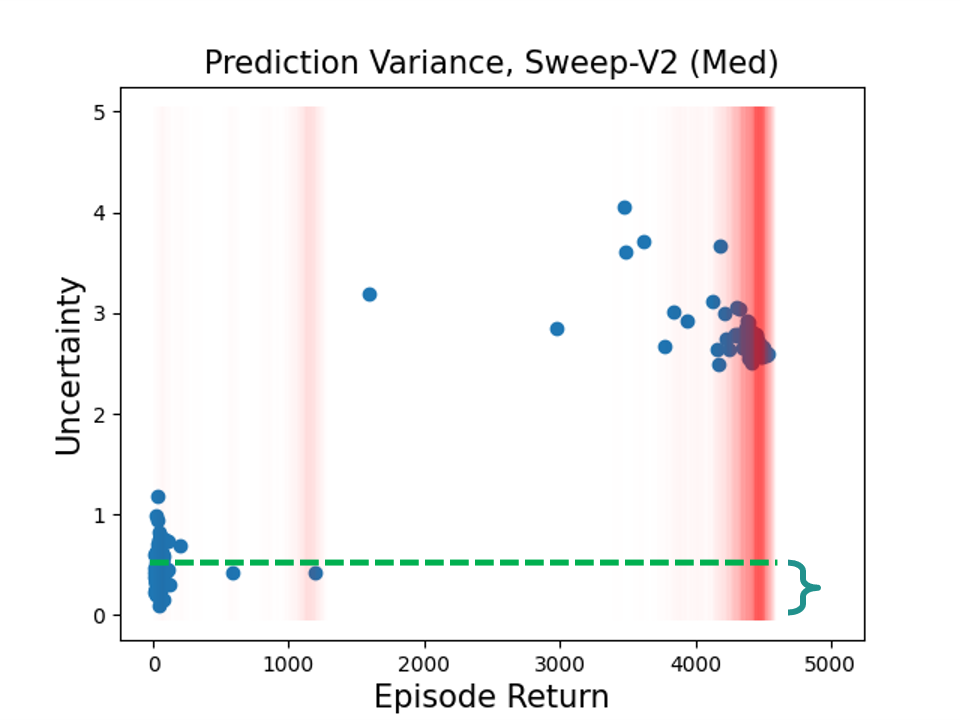}}
\subfigure{\includegraphics[width=0.24\linewidth]{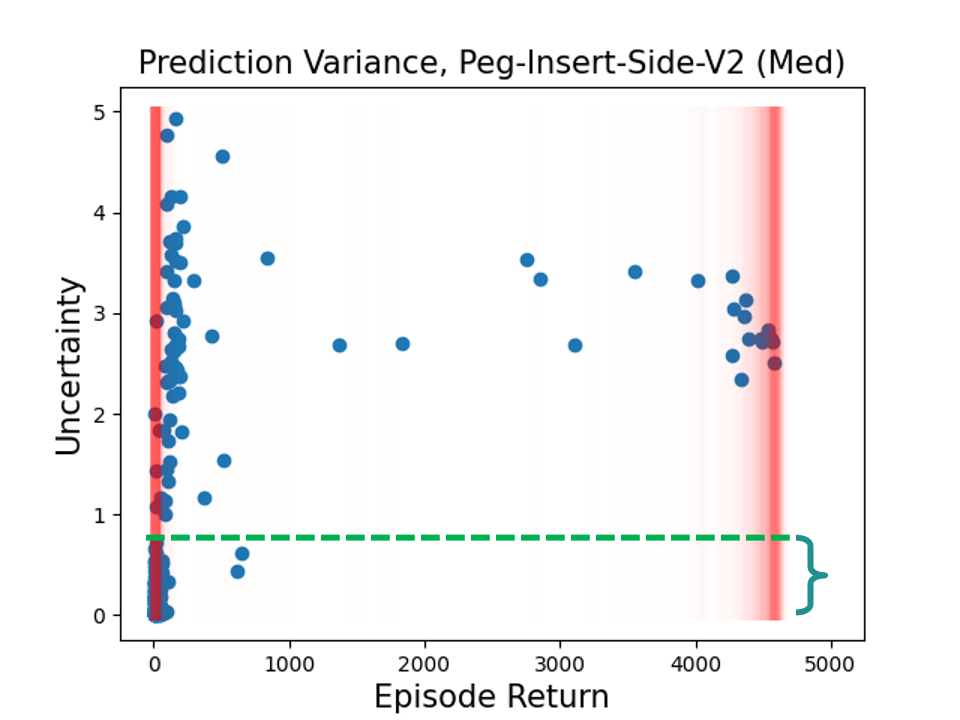}}
\subfigure{\includegraphics[width=0.7\linewidth]{figures/lll3.png}} 
	\caption{Visualization of the \textbf{Prediction Variance} quantification's behavior on various tasks. It may suffer from situations with higher prediction error and lower prediction variance in the medium or expert datasets.
 }
 \label{ppvv}
\end{figure}

\begin{figure}[H]
	\centering
\subfigure{\includegraphics[width=0.24\linewidth]{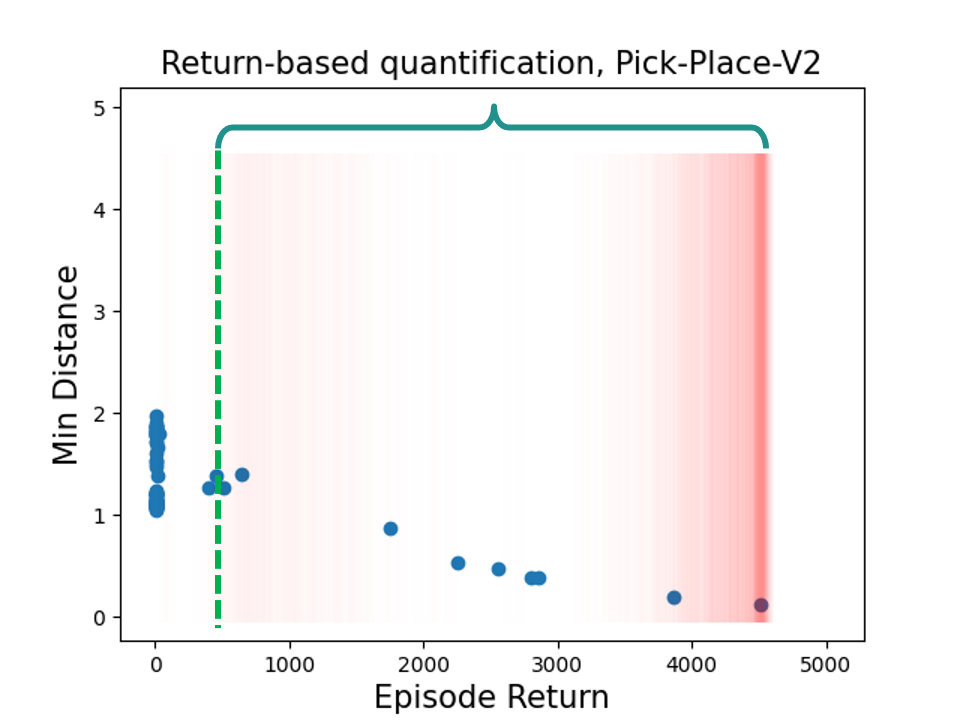}}
\subfigure{\includegraphics[width=0.24\linewidth]{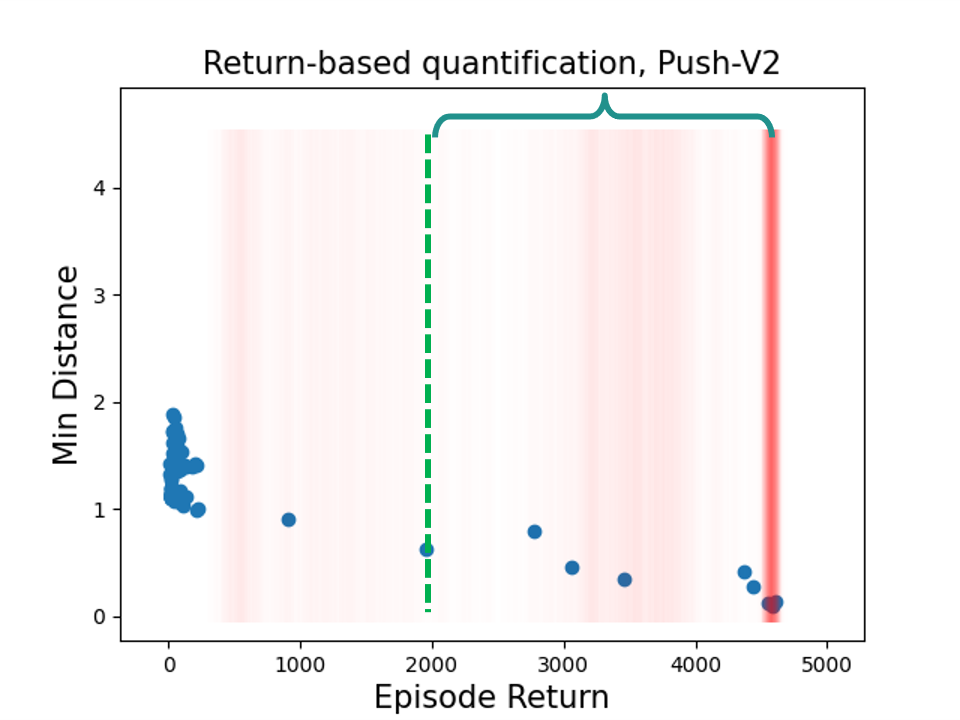}}
\subfigure{\includegraphics[width=0.24\linewidth]{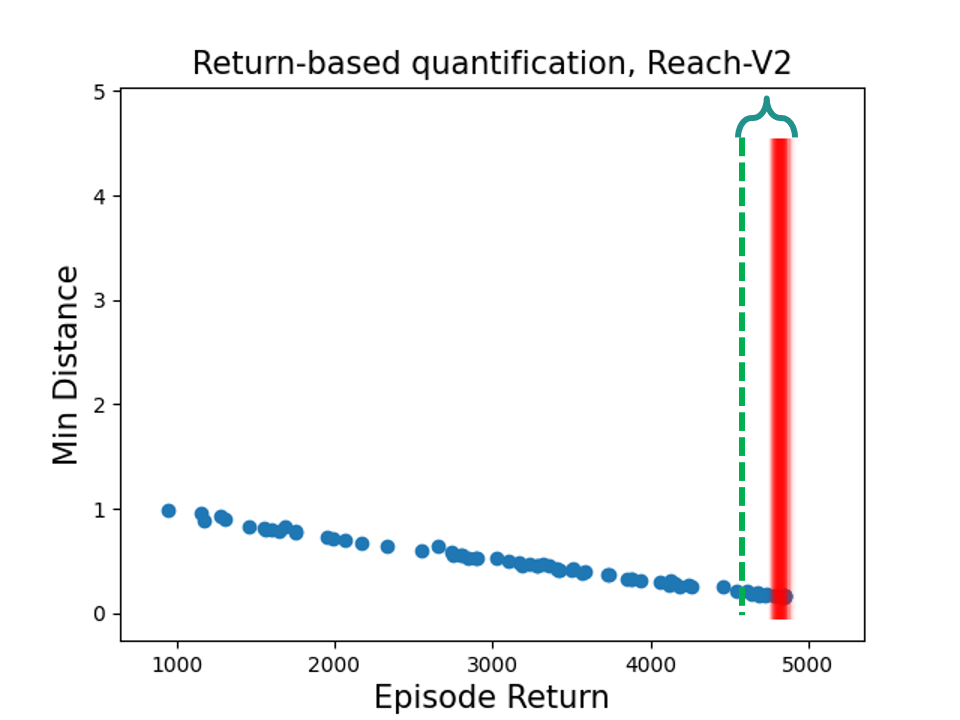}}
\subfigure{\includegraphics[width=0.24\linewidth]{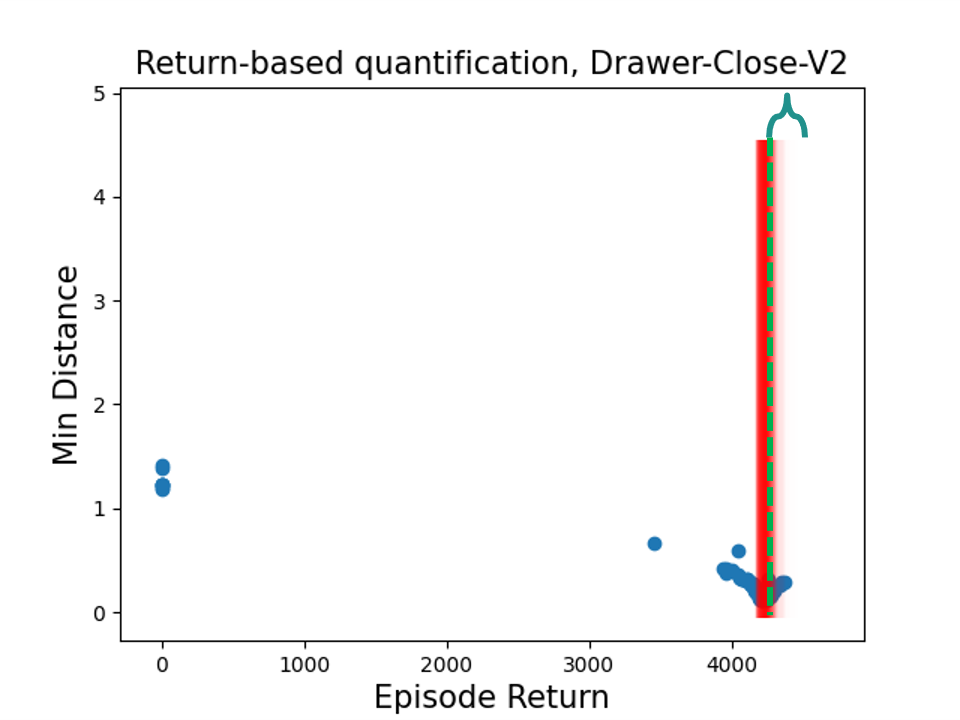}}
\subfigure{\includegraphics[width=0.24\linewidth]{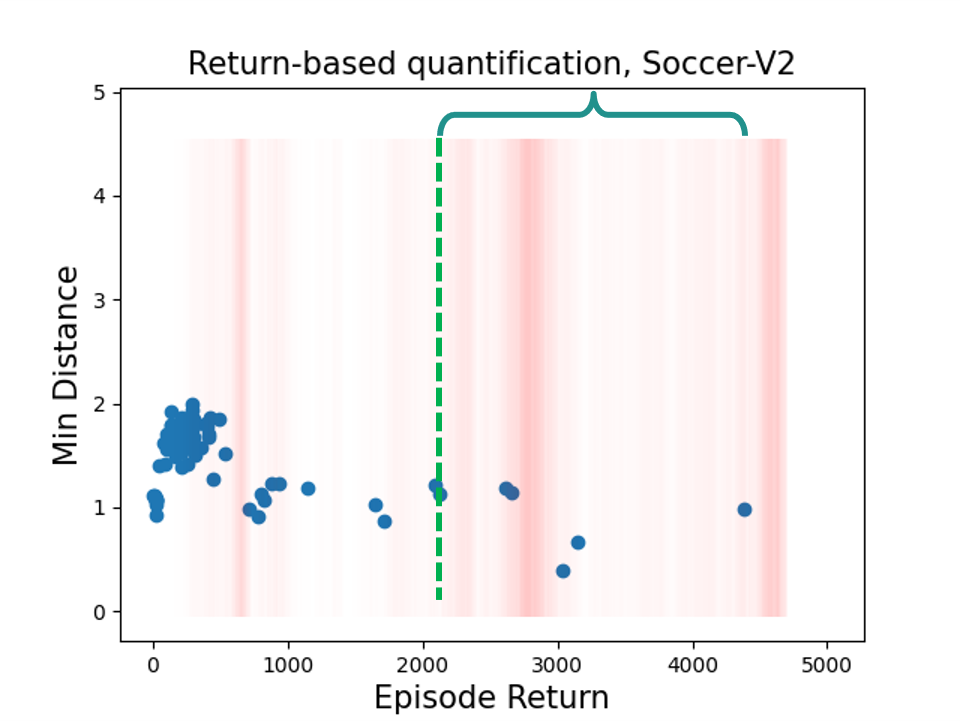}}
\subfigure{\includegraphics[width=0.24\linewidth]{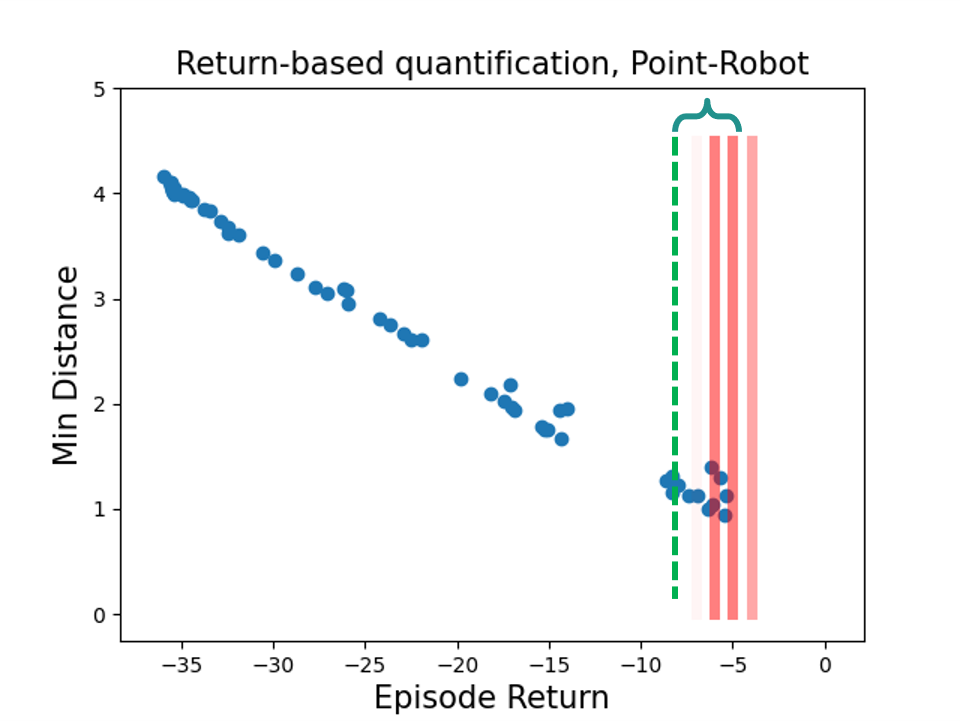}}
\subfigure{\includegraphics[width=0.24\linewidth]{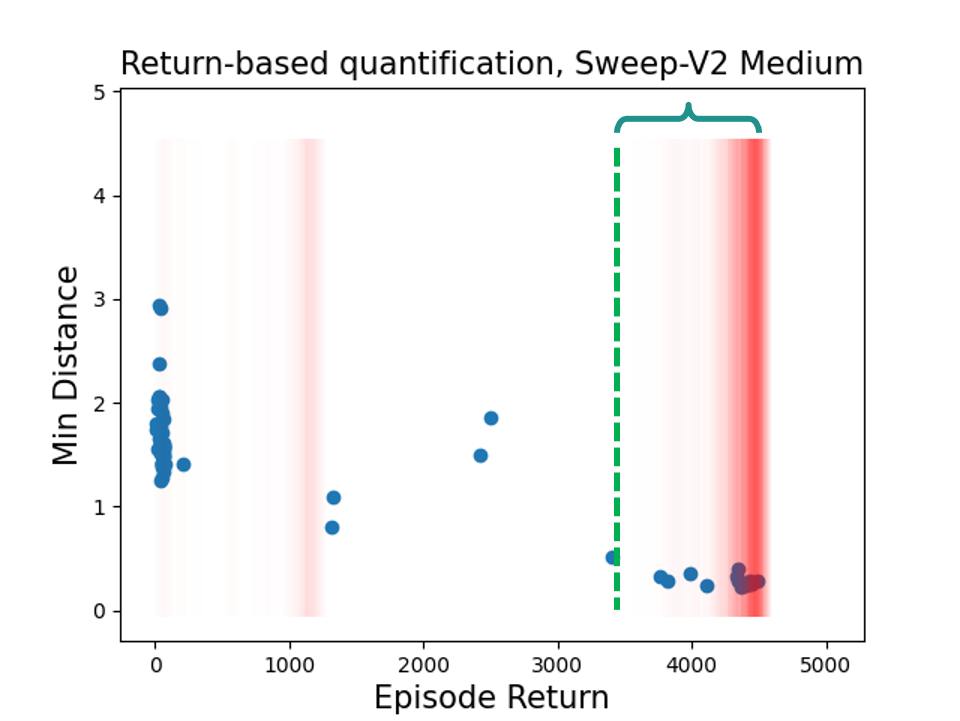}}
\subfigure{\includegraphics[width=0.24\linewidth]{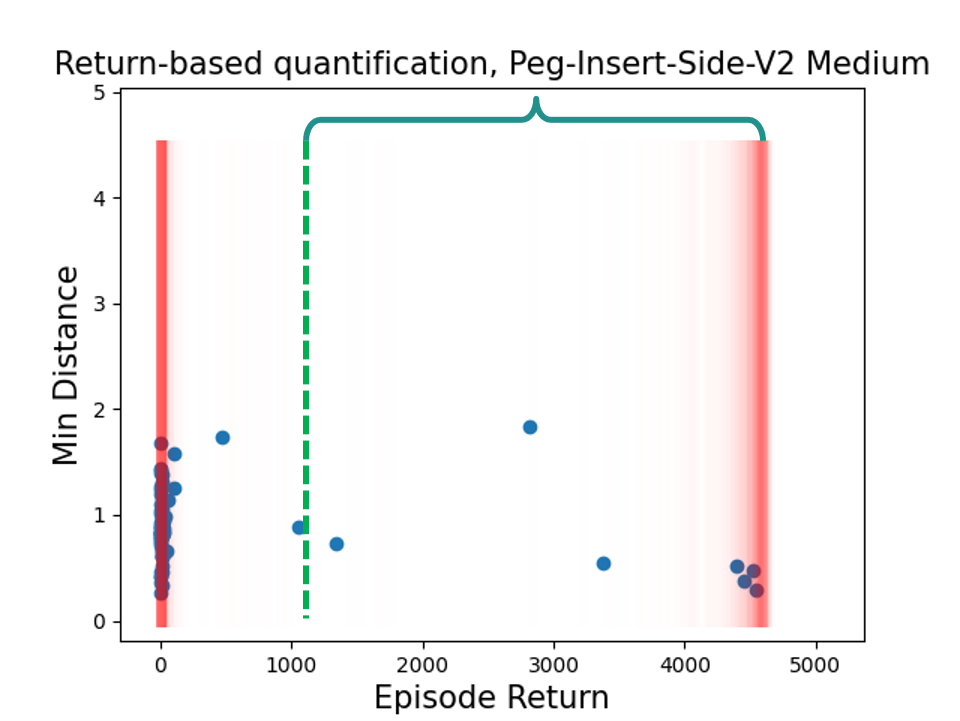}}
\subfigure{\includegraphics[width=0.7\linewidth]{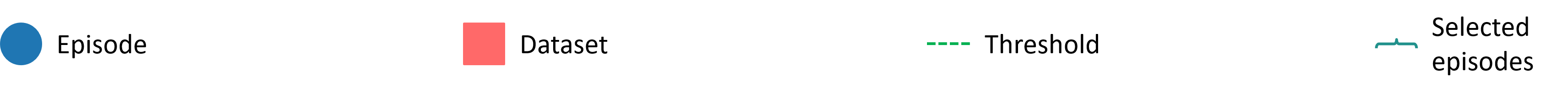}} 
	\caption{Visualization of the \textbf{Return-based} quantification's behavior on various tasks. It can successfully identify in-distribution episodes.
 }
 \label{gg}
\end{figure}



\subsection{Detailed Algorithm Performance on All Tasks}
\label{ss2}
Table \ref{tab:full}, Table \ref{tab:full2}, Table \ref{tab:full3} and Table \ref{tab:full4} show baselines' online adaptation and offline performance on all 50 Meta-World ML1 task sets and MuJoCo tasks, respectively. IDAQ significantly outperforms baselines with online adaptation, and achieves better or comparable performance to baselines with offline adaptation.

\begin{table}
	\centering
	\caption{Comparison between FOCAL, MACAW, and BOReL with online adaptation and IDAQ on Meta-World.}
	\begin{tabular}{l|c|c|c|c}
		\toprule
		Environment & IDAQ & FOCAL & MACAW &BOReL \\
		\midrule
   {Coffee-Push}&\textbf{1.26}$~\pm~$0.13&0.66$~\pm~$0.07&0.01$~\pm~$0.01&0.00$~\pm~$0.00\\
   	{Faucet-Close} &  \textbf{1.12}$~\pm~$0.01 &  1.06$~\pm~$0.02 &  0.07$~\pm~$0.01 &0.13$~\pm~$0.03 \\
   {Door-Unlock} & \textbf{1.11}$~\pm~$0.02 &  0.97$~\pm~$0.03 &  0.11$~\pm~$0.01 &0.13$~\pm~$0.03 \\
	
   {Plate-Slide-Side}&\textbf{1.07}$~\pm~$0.08 &0.70$~\pm~$0.14&0.00$~\pm~$0.00&0.00$~\pm~$0.00\\  
   {Faucet-Open} & \textbf{1.05}$~\pm~$0.02 &  1.01$~\pm~$0.02 &  0.08$~\pm~$0.04 &0.12$~\pm~$0.05 \\
		{Button-Press-Wall} &  \textbf{1.04}$~\pm~$0.04 &  0.99$~\pm~$0.06 &  0.02$~\pm~$0.00&0.01$~\pm~$0.00  \\
   {Plate-Slide}& \textbf{1.01}$~\pm~$0.03 &0.83$~\pm~$0.09 &0.01$~\pm~$0.00 &0.01$~\pm~$0.00   \\ 
   {Door-Close} & \textbf{0.99}$~\pm~$0.00 &  0.97$~\pm~$0.01 &  0.00$~\pm~$0.00 &0.37$~\pm~$0.19 \\
   {Drawer-Close}&\textbf{0.99}$~\pm~$0.02 &\textbf{0.96}$~\pm~$0.04 &0.53$~\pm~$0.50 &0.00$~\pm~$0.00  \\  
   {Plate-Slide-Back-Side}& \textbf{0.97}$~\pm~$0.02 &  0.77$~\pm~$0.20 &  0.02$~\pm~$0.01 &0.01$~\pm~$0.00\\
   {Door-Lock} & \textbf{0.97}$~\pm~$0.01 &  0.90$~\pm~$0.02 &  0.25$~\pm~$0.11 &0.14$~\pm~$0.00  \\
   {Window-Open}&\textbf{0.96}$~\pm~$0.02 &0.81$~\pm~$0.07 &0.15$~\pm~$0.11&0.03$~\pm~$0.00  \\  
   
   {Door-Open} & \textbf{0.96}$~\pm~$0.02 &  0.78$~\pm~$0.13 &  0.06$~\pm~$0.01&0.12$~\pm~$0.01  \\
   {Plate-Slide-Back} & \textbf{0.96}$~\pm~$0.02 &  0.58$~\pm~$0.06 &  0.21$~\pm~$0.17 &0.01$~\pm~$0.00\\
   {Window-Close}&\textbf{0.94}$~\pm~$0.01 &0.79$~\pm~$0.01 &0.54$~\pm~$0.44 &0.03$~\pm~$0.00  \\
   {Reach-Wall}&\textbf{0.93}$~\pm~$0.05 &0.53$~\pm~$0.18 &0.82$~\pm~$0.02 &0.06$~\pm~$0.00    \\  
   Dial-Turn &  \textbf{0.91}$~\pm~$0.05 &  0.84$~\pm~$0.09 &  0.00$~\pm~$0.00  &0.00$~\pm~$0.00 \\
   {Handle-Press-Side} & \textbf{0.91}$~\pm~$0.02 &  0.79$~\pm~$0.10 &  0.51$~\pm~$0.41 &0.02$~\pm~$0.01  \\
   {Handle-Pull}&\textbf{0.90}$~\pm~$0.02 &0.67$~\pm~$0.03 &0.00$~\pm~$0.00 &0.00$~\pm~$0.00  \\ 
   {Handle-Press} & \textbf{0.88}$~\pm~$0.05 &  \textbf{0.87}$~\pm~$0.02 &  0.28$~\pm~$0.10 &0.01$~\pm~$0.00\\
   {Reach}&\textbf{0.85}$~\pm~$0.03 &0.62$~\pm~$0.05 &0.63$~\pm~$0.04 &0.04$~\pm~$0.01  \\ 
   Lever-Pull&\textbf{0.85}$~\pm~$0.02 &0.73$~\pm~$0.07 &0.20$~\pm~$0.16 &0.05$~\pm~$0.00  \\   
   {Hammer} & \textbf{0.84}$~\pm~$0.06 &  0.59$~\pm~$0.07 &  0.10$~\pm~$0.01  &0.09$~\pm~$0.01 \\
   Drawer-Open & \textbf{0.82}$~\pm~$0.06 &  0.64$~\pm~$0.10 &  0.11$~\pm~$0.02 &0.10$~\pm~$0.00 \\
   Sweep&\textbf{0.77}$~\pm~$0.04 &0.32$~\pm~$0.08 &0.20$~\pm~$0.20 &0.00$~\pm~$0.00  \\
		{Button-Press} &  \textbf{0.74}$~\pm~$0.08 &  \textbf{0.68}$~\pm~$0.14 &  0.02$~\pm~$0.01 &0.01$~\pm~$0.01\\ 
		{Stick-Push} &  \textbf{0.73}$~\pm~$0.09 &  0.46$~\pm~$0.15 &  0.17$~\pm~$0.17 &0.00$~\pm~$0.00\\
		{Coffee-Button} &  \textbf{0.73}$~\pm~$0.14 &  \textbf{0.66}$~\pm~$0.16 &  0.15$~\pm~$0.13&0.02$~\pm~$0.00  \\
   Push-Wall&\textbf{0.71}$~\pm~$0.15 &0.43$~\pm~$0.06 &0.23$~\pm~$0.18  &0.00$~\pm~$0.00  \\
   {Shelf-Place}&\textbf{0.70}$~\pm~$0.18 &0.32$~\pm~$0.11 &0.01$~\pm~$0.01 &0.00$~\pm~$0.00  \\   
   {Basketball} & \textbf{0.64}$~\pm~$0.15 &  0.41$~\pm~$0.24 &  0.00$~\pm~$0.00  &0.00$~\pm~$0.00\\
		{Hand-Insert} &  \textbf{0.63}$~\pm~$0.04 &  0.29$~\pm~$0.07 &  0.02$~\pm~$0.01&0.00$~\pm~$0.00  \\
		{Sweep-Into} &  \textbf{0.61}$~\pm~$0.06 &  0.33$~\pm~$0.05 &  0.00$~\pm~$0.00 &0.01$~\pm~$0.00\\
    Button-Press-Topdown& \textbf{0.57}$~\pm~$0.11 &  0.45$~\pm~$0.06 &  0.38$~\pm~$0.36 &0.02$~\pm~$0.02\\
   {Peg-Unplug-Side} & \textbf{0.56}$~\pm~$0.07 &  0.19$~\pm~$0.09 &  0.00$~\pm~$0.00 &0.00$~\pm~$0.00 \\
   Assembly &  \textbf{0.55}$~\pm~$0.13 &  0.28$~\pm~$0.05 &  0.33$~\pm~$0.01 &0.04$~\pm~$0.00\\
	Push&\textbf{0.55}$~\pm~$0.10 &0.34$~\pm~$0.14 &0.28$~\pm~$0.19 &0.00$~\pm~$0.00\\  
   {Bin-Picking} &  0.53$~\pm~$0.16 &  0.31$~\pm~$0.21 &  \textbf{0.66}$~\pm~$0.11&0.00$~\pm~$0.00 \\
    {Push-Back} & \textbf{0.52}$~\pm~$0.05 &  0.16$~\pm~$0.04 &  0.00$~\pm~$0.00 &0.00$~\pm~$0.00 \\
   {Box-Close} & \textbf{0.51}$~\pm~$0.11 &  0.15$~\pm~$0.09 &  0.36$~\pm~$0.11  &0.05$~\pm~$0.01\\
  Soccer&\textbf{0.44}$~\pm~$0.04 &0.11$~\pm~$0.03 &\textbf{0.38}$~\pm~$0.31   &0.04$~\pm~$0.02 \\
  Button-Press-Topdown-Wall&  \textbf{0.43}$~\pm~$0.03 &  \textbf{0.40}$~\pm~$0.07 &  0.05$~\pm~$0.02 &0.05$~\pm~$0.01 \\
   {Disassemble} & \textbf{0.42}$~\pm~$0.14 &  0.26$~\pm~$0.04 &  0.05$~\pm~$0.00 &0.04$~\pm~$0.00 \\

   Coffee-Pull &  \textbf{0.40}$~\pm~$0.05 &   0.23$~\pm~$0.04 &  0.19$~\pm~$0.12  &0.00$~\pm~$0.00\\
   {Stick-Pull} &  \textbf{0.32}$~\pm~$0.06 &  0.17$~\pm~$0.07 &  0.00$~\pm~$0.00   &0.00$~\pm~$0.00\\
    Peg-Insert-Side&
   \textbf{0.30}$~\pm~$0.04 &0.08$~\pm~$0.03 &0.00$~\pm~$0.00    &0.00$~\pm~$0.00 \\   
   {Pick-Place-Wall} & 0.28$~\pm~$0.12 &  0.09$~\pm~$0.04 & \textbf{  0.39}$~\pm~$0.25  &0.00$~\pm~$0.00 \\
Pick-Out-Of-Hole&
   0.26$~\pm~$0.25 &0.16$~\pm~$0.16 &\textbf{0.59}$~\pm~$0.06 &0.00$~\pm~$0.00  \\ 
  
    Pick-Place&
   \textbf{0.20}$~\pm~$0.03 &0.07$~\pm~$0.02 &0.05$~\pm~$0.05   &0.00$~\pm~$0.00\\

Handle-Pull-Side&
   \textbf{0.14}$~\pm~$0.04 &\textbf{0.13}$~\pm~$0.09 &0.00$~\pm~$0.00  &0.00$~\pm~$0.00\\   
   \midrule
   Average & \textbf{0.73}$~\pm~$0.07& 0.53$~\pm~$0.08&0.18$~\pm~$0.09  &0.04$~\pm~$0.01\\
  
		\bottomrule
	\end{tabular}
	\label{tab:full}
\end{table}

\begin{table}
	\centering
	\caption{Comparison between baselines with offline adaptation and IDAQ on Meta-World tasks. ``-V2'' is omitted for brevity.}
	\begin{tabular}{l|c|c|c}
		\toprule
		Environment & IDAQ & \makecell[c]{FOCAL with Expert Context} & \makecell[c]{MACAW with Expert Context} \\
		\midrule
   {Coffee-Push}&\textbf{1.26}$~\pm~$0.13&0.50$~\pm~$0.06&\textbf{1.14}$~\pm~$0.27\\
   	{Faucet-Close} &  \textbf{1.12}$~\pm~$0.01  &  {1.07}$~\pm~$0.02 &  {1.01}$~\pm~$0.01  \\
   {Door-Unlock} & \textbf{1.11}$~\pm~$0.02  &  {0.96}$~\pm~$0.03 &   {0.99}$~\pm~$0.04  \\
	
   {Plate-Slide-Side}&\textbf{1.07}$~\pm~$0.08 &0.75$~\pm~$0.09   & {0.91}$~\pm~$0.09\\  
   {Faucet-Open} & \textbf{1.05}$~\pm~$0.02  &  \textbf{1.00}$~\pm~$0.02 &  {0.99}$~\pm~$0.01  \\
		{Button-Press-Wall} &  \textbf{1.04}$~\pm~$0.04  &  {0.98}$~\pm~$0.05 &  \textbf{0.99}$~\pm~$0.01  \\
   {Plate-Slide}& \textbf{1.01}$~\pm~$0.03  &\textbf{1.00}$~\pm~$0.03 &0.67$~\pm~$0.07    \\ 
   {Door-Close} & \textbf{0.99}$~\pm~$0.00  &  {0.97}$~\pm~$0.01 &  {0.92}$~\pm~$0.05  \\
   {Drawer-Close}&\textbf{0.99}$~\pm~$0.02  &{0.96}$~\pm~$0.04 &\textbf{1.00}$~\pm~$0.01    \\  
   {Plate-Slide-Back-Side}& \textbf{0.97}$~\pm~$0.02 &  0.90$~\pm~$0.07 &  0.80$~\pm~$0.05 \\
   {Door-Lock} & \textbf{0.97}$~\pm~$0.01  &  0.88$~\pm~$0.04 &  \textbf{0.97}$~\pm~$0.03  \\
   {Window-Open}&\textbf{0.96}$~\pm~$0.02  &\textbf{0.93}$~\pm~$0.05 &\textbf{0.98}$~\pm~$0.02  \\  
   
   {Door-Open} & \textbf{0.96}$~\pm~$0.02 &  {0.90}$~\pm~$0.02 &  \textbf{0.99}$~\pm~$0.02  \\
   {Plate-Slide-Back} & \textbf{0.96}$~\pm~$0.02  &  \textbf{0.93}$~\pm~$0.01 &  0.55$~\pm~$0.11  \\
   {Window-Close}&{0.94}$~\pm~$0.01 &0.73$~\pm~$0.02 &\textbf{1.00}$~\pm~$0.01  \\
   {Reach-Wall}&\textbf{0.93}$~\pm~$0.05 &\textbf{0.91}$~\pm~$0.05 &0.82$~\pm~$0.02    \\  
   Dial-Turn &  {0.91}$~\pm~$0.05  &  0.84$~\pm~$0.08 &  \textbf{0.98}$~\pm~$0.01  \\
   {Handle-Press-Side} & \textbf{0.91}$~\pm~$0.02 &  {0.87}$~\pm~$0.04 &  0.82$~\pm~$0.10  \\
   {Handle-Pull}&\textbf{0.90}$~\pm~$0.02 &0.70$~\pm~$0.05 &\textbf{0.95}$~\pm~$0.05    \\ 
   {Handle-Press} & \textbf{0.88}$~\pm~$0.05 &  \textbf{0.79}$~\pm~$0.08 &  0.56$~\pm~$0.19  \\
   {Reach}&\textbf{0.85}$~\pm~$0.03 &\textbf{0.83}$~\pm~$0.05 &0.64$~\pm~$0.08    \\ 
   Lever-Pull&0.85$~\pm~$0.02 &0.76$~\pm~$0.03 &\textbf{0.97}$~\pm~$0.07   \\   
   {Hammer} & \textbf{0.84}$~\pm~$0.06 &  0.78$~\pm~$0.04 &  0.40$~\pm~$0.18  \\
   Drawer-Open & {0.82}$~\pm~$0.06  &  0.73$~\pm~$0.11 &  \textbf{0.98}$~\pm~$0.01  \\
   Sweep&{0.77}$~\pm~$0.04 &0.74$~\pm~$0.02 &\textbf{0.98}$~\pm~$0.01    \\
		{Button-Press} &  \textbf{0.74}$~\pm~$0.08 &  0.63$~\pm~$0.09 &  \textbf{0.71}$~\pm~$0.04  \\ 
		{Stick-Push} &  \textbf{0.73}$~\pm~$0.09 &  0.14$~\pm~$0.09 &  \textbf{0.67}$~\pm~$0.09  \\
		{Coffee-Button} &  \textbf{0.73}$~\pm~$0.14 &  \textbf{0.61}$~\pm~$0.20 &  0.21$~\pm~$0.11  \\
   Push-Wall&{0.71}$~\pm~$0.15 &\textbf{0.90}$~\pm~$0.12 &\textbf{0.96}$~\pm~$0.09    \\
   {Shelf-Place}&\textbf{0.70}$~\pm~$0.18 &0.57$~\pm~$0.13 &0.55$~\pm~$0.03    \\   
   {Basketball} & \textbf{0.64}$~\pm~$0.15  &  0.49$~\pm~$0.17 &  0.47$~\pm~$0.18  \\
		{Hand-Insert} &  \textbf{0.63}$~\pm~$0.04 &  \textbf{0.64}$~\pm~$0.09 &   0.20$~\pm~$0.09  \\
		{Sweep-Into} &  \textbf{0.61}$~\pm~$0.06 &  \textbf{0.64}$~\pm~$0.09 &  0.00$~\pm~$0.00  \\
    Button-Press-Topdown& {0.57}$~\pm~$0.11 &  0.48$~\pm~$0.11 &  \textbf{0.92}$~\pm~$0.04  \\
   {Peg-Unplug-Side} & \textbf{0.56}$~\pm~$0.07 &  \textbf{0.57}$~\pm~$0.10 &  0.18$~\pm~$0.10  \\
   Assembly &  {0.55}$~\pm~$0.13 &  \textbf{0.64}$~\pm~$0.03 &  0.36$~\pm~$0.02 \\
	Push&{0.55}$~\pm~$0.10 &\textbf{0.98}$~\pm~$0.13 &0.86$~\pm~$0.02\\  
   {Bin-Picking} &  0.53$~\pm~$0.16 &  \textbf{0.61}$~\pm~$0.12 &  \textbf{0.63}$~\pm~$0.11  \\
    {Push-Back} & \textbf{0.52}$~\pm~$0.05&  \textbf{0.52}$~\pm~$0.16 &  0.15$~\pm~$0.09  \\
   {Box-Close} & \textbf{0.51}$~\pm~$0.11 &  \textbf{0.56}$~\pm~$0.08 &  0.35$~\pm~$0.11  \\
  Soccer&{0.44}$~\pm~$0.04 &0.45$~\pm~$0.03 &\textbf{0.59}$~\pm~$0.11    \\
  {Button-Press-Topdown-Wall}&  \textbf{0.43}$~\pm~$0.03 &  \textbf{0.40}$~\pm~$0.06 &  \textbf{0.43}$~\pm~$0.06  \\
   {Disassemble} & {0.42}$~\pm~$0.14  &  0.23$~\pm~$0.05 &  \textbf{0.50}$~\pm~$0.06  \\

   Coffee-Pull &  {0.40}$~\pm~$0.05 &  \textbf{0.58}$~\pm~$0.11 &  {0.45}$~\pm~$0.11  \\
   {Stick-Pull} &  \textbf{0.32}$~\pm~$0.06 &  0.18$~\pm~$0.06 &  \textbf{0.27}$~\pm~$0.09  \\
    Peg-Insert-Side&
   {0.30}$~\pm~$0.04  &\textbf{0.52}$~\pm~$0.08 &0.25$~\pm~$0.04    \\   
   {Pick-Place-Wall} & \textbf{0.28}$~\pm~$0.12 &  0.13$~\pm~$0.07 &  \textbf{0.21}$~\pm~$0.16  \\
Pick-Out-Of-Hole&
   0.26$~\pm~$0.25 &0.27$~\pm~$0.27 &\textbf{0.59}$~\pm~$0.08    \\ 
  
    Pick-Place&
   {0.20}$~\pm~$0.03 &0.29$~\pm~$0.11 &\textbf{0.72}$~\pm~$0.09    \\

Handle-Pull-Side&
   {0.14}$~\pm~$0.04  &0.09$~\pm~$0.05 &\textbf{0.94}$~\pm~$0.08   \\   
   \midrule
   Average & \textbf{0.73}$~\pm~$0.07&0.67$~\pm~$0.07 &0.68$~\pm~$0.07 \\
		\bottomrule
	\end{tabular}
	\label{tab:full2}
\end{table}

\begin{table}
	\centering
	\caption{Comparison between FOCAL, MACAW, and BOReL with online adaptation and IDAQ on MuJoCo tasks.}
	\begin{tabular}{l|c|c|c|c}
		\toprule
		Environment & IDAQ & FOCAL & MACAW &BOReL \\
		\midrule
		Cheetah-Vel &  \textbf{-171.52}$~\pm~$21.96 & -287.70$~\pm~$30.62 & -233.97$~\pm~$23.46 &-301.4$~\pm~$36.8 \\
		Point-Robot &  \textbf{-5.10}$~\pm~$0.26 & -15.38$~\pm~$0.95 & -14.61$~\pm~$0.98  &-17.28$~\pm~$1.16\\
		Point-Robot-Sparse &  \textbf{7.78} $~\pm~$0.64 & 0.83$~\pm~$0.37 & 0.00$~\pm~$0.00 &0.00$~\pm~$0.00 \\
		\bottomrule
	\end{tabular}
	\label{tab:full3}
\end{table}

\begin{table}
	\centering
	\caption{Comparison between baselines with offline adaptation and IDAQ on MuJoCo tasks.}
	\begin{tabular}{l|c|c|c}
		\toprule
		Environment & IDAQ & \makecell[c]{FOCAL with Expert Context} & \makecell[c]{MACAW with Expert Context} \\
		\midrule
		Cheetah-Vel &  \textbf{-171.52}$~\pm~$21.96 & \textbf{-156.07}$~\pm~$23.22 & -292.92$~\pm~$36.66  \\
		Point-Robot &  \textbf{-5.10}$~\pm~$0.26 & \textbf{-4.68}$~\pm~$0.18 & -19.60$~\pm~$1.15  \\
		Point-Robot-Sparse &  \textbf{7.78} $~\pm~$0.64 & \textbf{8.37}$~\pm~$0.67  & 0.00$~\pm~$0.00\\
		\bottomrule
	\end{tabular}
	\label{tab:full4}
\end{table}

\subsection{Dataset Returns}
Table \ref{tab:mazee} and Table \ref{tab:mazee2} show the average returns of the offline datasets used in meta-training.

\begin{table}
	\centering
	\caption{Dataset average returns on 50 Meta-World tasks.}
	\begin{tabular}{l|c}
		\toprule
		Environment & Dataset Return \\
		\midrule
   {Coffee-Push}& 1487.84\\
   	{Faucet-Close} &  4039.52 \\
   {Door-Unlock} &  3653.40 \\
	
   {Plate-Slide-Side}&3530.06\\  
   {Faucet-Open} &  4145.18 \\
		{Button-Press-Wall} &  2899.34  \\
   {Plate-Slide}&4395.86    \\ 
   {Door-Close} & 4519.77 \\
   {Drawer-Close}& 4233.04\\  
   {Plate-Slide-Back-Side}& 4735.74\\
   {Door-Lock} & 3352.35\\
   {Window-Open}&4382.87 \\  
   
   {Door-Open} & 4401.24 \\
   {Plate-Slide-Back} & 4732.90   \\
   {Window-Close}&3572.99 \\
   {Reach-Wall}& 4805.38  \\  
   Dial-Turn & 3824.44\\
   {Handle-Press-Side} & 4836.04  \\
   {Handle-Pull}& 3878.46  \\ 
   {Handle-Press} &  4851.51 \\
   {Reach}& 4820.84  \\ 
   Lever-Pull& 914.01  \\   
   {Hammer} & 4251.57  \\
   Drawer-Open & 4041.54  \\
   Sweep&  4354.57 \\
		Button & 3855.10 \\ 
		{Stick-Push} & 4230.52 \\
		{Coffee-Button} & 4049.44  \\
   Push-Wall& 3699.66 \\
   {Shelf-Place}&2813.55   \\   
   {Basketball} & 4071.31 \\
		{Hand-Insert} &  3963.11 \\
		{Sweep-Into} & 4019.97\\
    Button-Press-Topdown & 3570.22  \\
   {Peg-Unplug-Side} & 4235.23 \\
   Assembly & 4238.57\\
	Push&3147.80\\  
   {Bin-Picking} & 4257.54 \\
    {Push-Back} & 3841.90 \\
   {Box-Close} & 4012.39 \\
  Soccer& 2798.96 \\
  Button-Press-Topdown-Wall& 3781.99 \\
   {Disassemble} & 3865.37 \\

   Coffee-Pull & 4166.44  \\
   {Stick-Pull} &  4015.26 \\
    Peg-Insert-Side&3826.91   \\   
   {Pick-Place-Wall} & 2499.35  \\
Pick-Out-Of-Hole&3405.74  \\ 
    Pick-Place&3482.43   \\   
Handle-Pull-Side&2425.14  \\   
		\bottomrule
	\end{tabular}
	\label{tab:mazee}
\end{table}

\begin{table}
	\centering
	\caption{Dataset average returns on Meta-World-Meidum tasks and MuJoCo tasks.}
	\begin{tabular}{l|c}
		\toprule
		Environment & Dataset Return \\
		\midrule
		Sweep (Med) &  2874.00 \\
		Peg-Insert-Side (Med)  & 2342.34 \\
		Point-Robot & -5.72 \\
		Point-Robot-Sparse  & 7.28 \\
		Cheetah-Vel  & -138.29 \\
		\bottomrule
	\end{tabular}
	\label{tab:mazee2}
\end{table}

\newpage

\section{Additional Ablation Studies}
\label{abl}
In this section, we will further conduct various ablation studies to investigate the robustness of IDAQ in dataset quality and hyper-parameters.

\textbf{Reference stage length.} Table \ref{tab:len} shows IDAQ's performance with different reference stage lengths. The total number of adaptation episodes is 20. We find that IDAQ performs well during 10-15 episodes, which is 50\%-75\% of the total number of adaptation episodes. A small reference stage length (5) may lead to a possibly unreliable task belief and cause a degradation in performance. The 19-episode does not perform the iterative optimization process, and the task belief updates will not converge.


\textbf{Dataset Quality.}
We test IDAQ and baselines with several ``medium'' datasets, which are collected by periodically evaluating policies of SAC. As shown in Table \ref{tab:med}, IDAQ still significantly outperforms baseline algorithms on medium datasets, which implies IDAQ's ability to learn various datasets. 

\begin{table}[H]
	\centering
	\caption{IDAQ's performance with various initial stage lengths. }
	\begin{tabular}{l|c|c|c|c}
		\toprule
		Environment & 5 Episodes & 10 Episodes & 15 Episodes & 19 Episodes\\
		\midrule
		Point-Robot & -6.04 $~\pm~$ 0.31 & \textbf{-5.11}$~\pm~$0.21 & \textbf{-5.10} $~\pm~$ 0.26 & -5.37 $~\pm~$0.11\\
  Point-Robot-Sparse & 4.04 $~\pm~$0.58 & \textbf{7.78} $~\pm~$0.64 & \textbf{8.07} $~\pm~$0.62 & 7.29$~\pm~$ 0.50\\
		\bottomrule
	\end{tabular}
	\label{tab:len}
\end{table}


\begin{table}[H]
	\centering
	\caption{Algorithms' performance on datasets of various qualities.}
	\begin{tabular}{l|c|c|c}
		\toprule
		Environment & IDAQ & FOCAL & MACAW\\
		\midrule

   Sweep&\textbf{0.77}$~\pm~$0.04 &0.32$~\pm~$0.08 &0.20$~\pm~$0.20     \\
		Sweep (Med) &  \textbf{0.59}$~\pm~$0.13  &  0.38$~\pm~$0.13  & 0.04$~\pm~$0.03 \\
    
   \midrule
   
    Peg-Insert-Side&
   \textbf{0.30}$~\pm~$0.04 &0.08$~\pm~$0.03 &0.00$~\pm~$0.00    \\   
		Peg-Insert-Side (Med) &  \textbf{0.30}$~\pm~$0.14  &  0.10$~\pm~$0.07  & 0.00$~\pm~$ 0.00   \\
		\bottomrule
	\end{tabular}
	\label{tab:med}
\end{table}

\textbf{Hyper-parameters.} The uncertainty quantifications are robust to the hyperparameters in this paper. Prior work \citep{lu2021revisiting} introduces three sensitive hyperparameters: penalty scale, rollout length, and the number of models in offline model-based RL. However, in the in-distribution online adaptation framework of IDAQ, penalty scale and rollout length are not required. IDAQ utilizes uncertainty quantifications to rank online episodes for deriving in-distribution contexts instead of using it as a regularization. For IDAQ+Prediction Error and IDAQ+Prediction Variance, the hyperparameter of these uncertainty quantifications is the number of models $L$ in an ensemble.  We perform ablation studies on this hyperparameter. 

Results in Table \ref{tab:52} and Table \ref{tab:53} show that IDAQ+Prediction Error and IDAQ+Prediction Variance are robust to the choice of the number of models in an ensemble, and IDAQ+Return significantly outperforms IDAQ+Prediction Error and IDAQ+Prediction Variance with different hyperparameter values.

\begin{table*}[h]
	\centering
	\caption{Ablation on the number of ensembles $L$ in IDAQ+Prediction Error.}
	\begin{tabular}{l|c|c|c|c|c}
		\toprule
		Example Env & $L=2$ & $L=4$ & $L=8$ & $L=12$ & IDAQ+Return\\
		\midrule
		{Sweep (Med)} & 0.13$~\pm~$0.03 & 0.14$~\pm~$0.05 &  0.16$~\pm~$0.03 & 0.17$~\pm~$0.04& 0.59$~\pm~$0.13\\
		\midrule
		
		Point-Robot &  -6.33$~\pm~$0.22  & -5.70$~\pm~$0.05 & -6.14$~\pm~$0.24 & -5.98$~\pm~$0.24 & -5.10$~\pm~$0.26 \\ 
		\midrule
		{Reach}& 0.85$~\pm~$0.02 & 0.86$~\pm~$0.02 & 0.85$~\pm~$0.02 & 0.86$~\pm~$0.01& 0.85$~\pm~$0.03\\
		
		\bottomrule
	\end{tabular}
	\label{tab:52}
\end{table*}

\begin{table*}[h]
	\centering \caption{Ablation on the number of ensembles $L$ in IDAQ+Prediction Variance.}
	\begin{tabular}{l|c|c|c|c|c}
		\toprule
		Example Env & $L=2$ & $L=4$ & $L=8$ & $L=12$ & IDAQ+Return\\
		\midrule
		{Sweep (Med)} & 0.05$~\pm~$0.02 & 0.05$~\pm~$0.02 &  0.06$~\pm~$0.02 & 0.05$~\pm~$0.02& 0.59$~\pm~$0.13\\
		\midrule
		
		Point-Robot &  -19.30$~\pm~$0.92   & -21.29$~\pm~$0.85 & -20.51$~\pm~$1.06 & -19.48$~\pm~$0.93 & -5.10$~\pm~$0.26  \\
		\midrule
		{Reach}& 0.47$~\pm~$0.04 & 0.46$~\pm~$0.03 & 0.47$~\pm~$0.03 & 0.46$~\pm~$0.03& 0.85$~\pm~$0.03 \\
		
		\bottomrule
	\end{tabular}
	\label{tab:53}
	\vspace{-0.1in}
\end{table*}

\textbf{Other baselines.} Our approach IDAQ is compatible with existing context-based offline meta-RL methods. FOCAL is one of such state-of-the-art methods that can address the challenge of \textit{MDP Ambiguity} \citep{dorfman2021offline} in offline meta-training. Due to the task-dependent behavior policy, the sub-datasets of different tasks may be disjoint and a promising solution is to apply contrastive loss \citep{li2020focal,yuan2022robust} on the latent task embeddings to distinguish tasks. Offline PEARL+Contrastive Loss (abbreviated as ``OP+CL'') is another popular offline meta-training baseline in the literature \citep{li2020focal,yuan2022robust}. Combining IDAQ with OP+CL, we will investigate how much improvement can be achieved by in-distribution online adaptation in OP+CL. Results in Table \ref{tab:54} demonstrate that IDAQ can significantly improve over OP+CL (the ``OP+CL+IDAQ'' column), and achieves similar or slightly lower performance than FOCAL+IDAQ.

\begin{table*}[h]
	\centering
	\caption{Improvement of OP+CL+IDAQ over the OP+CL baseline.}
	\begin{tabular}{l|c|c|c|c}
		\toprule
		Example Env  & FOCAL+IDAQ & FOCAL& OP+CL+IDAQ& OP+CL \\
		\midrule
		{Push} &  \textbf{0.55}$~\pm~$0.10 &0.34 $~\pm~$0.14& 0.45$~\pm~$0.04 & 0.19$~\pm~$0.02  \\
		{Pick-Place}   &  \textbf{0.20}$~\pm~$0.03 & 0.07 $~\pm~$0.02& \textbf{0.19}$~\pm~$0.04 & 0.08$~\pm~$0.02 \\
		{Soccer}   &  \textbf{0.44}$~\pm~$0.04 & 0.11$~\pm~$0.03&\textbf{0.42}$~\pm~$0.05 &0.19$~\pm~$0.02 \\
		{Drawer-Close}&\textbf{0.99}$~\pm~$0.02 & \textbf{0.96} $~\pm~$ 0.04 & \textbf{1.00}$~\pm~$0.01 & 0.82$~\pm~$0.04 \\  
		{Reach} & {0.85}$~\pm~$0.03 & 0.62$~\pm~$0.05& \textbf{0.90}$~\pm~$0.01 & 0.59$~\pm~$0.03 \\
		\midrule
		{Sweep (Med)}  &  \textbf{0.59}$~\pm~$0.13 & 0.38 $~\pm~$ 0.13&0.48$~\pm~$0.09 & 0.19$~\pm~$0.06 \\
		{Peg-Insert-Side (Med)} &  \textbf{0.30}$~\pm~$0.14  & 0.10 $~\pm~$ 0.07 &0.19$~\pm~$0.04 &0.06$~\pm~$0.02 \\
		
		\bottomrule
	\end{tabular}
	\label{tab:54}
\end{table*}


\end{document}